\documentclass{article}
\usepackage{graphicx}
\usepackage{enumitem}
\usepackage{float}
\usepackage{graphicx,subfigure}
\usepackage[super]{nth}
\usepackage{amsmath}
\usepackage{bm}
\usepackage{hyperref}
\usepackage{bigints}
\usepackage[dvipsnames]{xcolor}  
\usepackage{lineno} 
\usepackage{array}
\usepackage{amssymb}
\usepackage{mathtools}
\usepackage{authblk}
\usepackage{xcolor}

\usepackage{amsthm}

\usepackage[utf8]{inputenc}
\usepackage{amsmath}
\usepackage{amsthm}
\usepackage{amssymb}
\usepackage{graphicx}
\usepackage{amsfonts}
\usepackage{latexsym}
\usepackage[title]{appendix}
\usepackage{color}
\usepackage{hyperref}
\usepackage{nomencl}
\makenomenclature

\newtheorem{theorem}{Theorem}[section]
\newtheorem{lemma}[theorem]{Lemma}
\newtheorem{proposition}[theorem]{Proposition}
\newtheorem{corollary}[theorem]{Corollary}
\newtheorem{definition}[theorem]{Definition}

\newtheorem{example}[theorem]{Example}
\newtheorem{remark}[theorem]{Remark}

\numberwithin{figure}{section}
\numberwithin{equation}{section}

\newcommand{\argmin}[1]{\underset{#1}{\operatorname{arg}\operatorname{min}}\;}

\newcommand{\vertiii}[1]{{\left\vert\kern-0.25ex\left\vert\kern-0.25ex\left\vert #1 
    \right\vert\kern-0.25ex\right\vert\kern-0.25ex\right\vert}}

\begin{document}

\title{{\bf A Structure-Preserving Kernel Method for Learning\\ 
Hamiltonian Systems}}

\author{Jianyu Hu$^{1}$, Juan-Pablo~Ortega$^{1}$, and Daiying Yin$^{1}$}


\maketitle

\begin{abstract}
A structure-preserving kernel ridge regression method is presented that allows the recovery of nonlinear Hamiltonian functions out of datasets made of noisy observations of Hamiltonian vector fields. The method proposes a closed-form solution that yields excellent numerical performances that surpass other techniques proposed in the literature in this setup. From the methodological point of view, the paper extends kernel regression methods to problems in which loss functions involving linear functions of gradients are required and, in particular, a differential reproducing property and a Representer Theorem are proved in this context. The relation between the structure-preserving kernel estimator and the Gaussian posterior mean estimator is analyzed. A full error analysis is conducted that provides convergence rates using fixed and adaptive regularization parameters. The good performance of the proposed estimator together with the convergence rate is illustrated with various numerical experiments.
\end{abstract}

\makeatletter
\addtocounter{footnote}{1} \footnotetext{Jianyu Hu, Juan-Pablo Ortega, and Daiying Yin are with the Division of Mathematical Sciences, School of Physical and Mathematical Sciences, Nanyang Technological University, Singapore. Their email addresses are {\texttt{Jianyu.Hu@ntu.edu.sg}}, {\texttt{Juan-Pablo.Ortega@ntu.edu.sg}}, and {\texttt{YIND0004@e.ntu.edu.sg}}, respectively.}
\makeatother

\tableofcontents

\section{Introduction}

Hamiltonian systems are essential tools to model physical systems \cite{Abraham1978, Marsden1994, arnol2013mathematical}. In the simplest case in which the phase space is Euclidean and is endowed with a constant symplectic form, Hamiltonian systems are determined by a scalar-valued Hamiltonian function $H:\mathbb{R}^{2d}\longrightarrow\mathbb{R}$, $ d\in \mathbb{N} $, and when using the so-called canonical Darboux coordinates, the corresponding dynamics is governed by the well-known {\bf Hamilton's equations}
\begin{equation}
\begin{aligned}\label{ham-sys}
\dot{\mathbf{z}}(t) = J \nabla H(\mathbf{z}(t)),
\end{aligned}
\end{equation}
where $\mathbf{z}=(\mathbf{q}^{\top},\mathbf{p}^{\top})^{\top}\in\mathbb{R}^{2d}$ is the phase space vector comprising the positions and the momenta of the system, and $J$ is the canonical symplectic matrix. 
Modern technology has made collecting trajectory data directly from physical systems increasingly feasible. This motivates us to address the fundamental inverse problem: {\it determining the underlying Hamiltonian function and the governing Hamilton's equations from trajectory data}.

Machine learning-based methods have become popular and effective approaches to tackling this problem. A straightforward strategy that has been proposed is to directly learn the Hamiltonian vector field as a map that assigns each point in the phase space to the Hamiltonian vector field at that point \cite{racca2021automatic,zhang2022gfinns}. However, this method is not structure-preserving because there is no guarantee that the learned vector field comes from a Hamiltonian function in the presence of estimation and approximation errors. An improved version of the above is the notion of Hamiltonian Neural Networks (HNN), for which the main idea is to model the Hamiltonian function instead of its induced vector field so that equation (\ref{ham-sys}) will produce a genuine Hamiltonian vector field by construction, see \cite{bertalan2019learning,david2023symplectic,greydanus2019hamiltonian,han2021adaptable,toth2019hamiltonian}. 
In addition to learning the Hamiltonian function, there are other emerging structure-preserving methods, including Lagrangian neural networks \cite{cranmer2020lagrangian}, symplectic neural networks \cite{jin2020sympnets}, symplectic recurrent neural networks \cite{chen2019symplectic}, symplectic reversible neural networks \cite{valperga2022learning}, a method based on learning the generating function \cite{chen2021data}, and symmetry-preserving method \cite{Toshev2023, vaquero2023symmetry}. In \cite{RCSP1}, a mathematical characterization of linear port-Hamiltonian systems is studied and applied to learning, which preserves the port-Hamiltonian structure by construction. Nevertheless, for a general nonlinear system, there is an apparent lack of analysis in the literature that has to do with the estimation error bound for such learning problems due to the difficulty of finding a proper framework to perform analysis. We would like to highlight the difference between the above-mentioned inverse problem and the forward problem of producing solutions to a known governing equation. Physics-informed neural networks (PINNs) \cite{RAISSI2019686} have brought numerical success to learning the solutions of physical systems described by partial differential equations, whereas the convergence of the learned solution has only been proved in special cases \cite{Shin2020OnTC}. In the Hamiltonian case, however, the governing equation is simply an ODE, for which a variational integrator theory has been well-developed \cite{marsden_west_2001} to obtain the solutions. The analog of PINNs in the Hamiltonian case has been studied in Chen et al. \cite{chen2023data}, where a convergence result of the empirical loss was shown for the forward problem of solving the Euler-Lagrange equations via a physics-informed neural network \cite{cuomo2022scientific,yang2018physics}. In general, theoretical sample estimation error bounds are rarely presented for the inverse learning problem, that is, the problem of learning the Hamiltonian function. 

The two methodologies at the core of this paper are Gaussian processes (GP) and kernel ridge regressions. These methods have been shown to be practically powerful and theoretically sound and have found many practical applications in dealing with nonlinear phenomena \cite{williams2006gaussian}. They are widely used in modeling physical systems and electrical, biological, and chemical engineering problems. The use of kernel methods for the data-based learning of generic vector fields has been pioneered in \cite{bouvrie2017kernel, Bouvrie2017} (see also \cite{hamzi2021learning, pmlr-v202-hou23c} and references therein). Kernel methods also naturally appear in the related field of reservoir computing \cite{hermans:rkhs, RC25}. 

In kernel ridge regression applications, a choice of kernel and regularization parameters is needed. It is a well-known fact that if the Tikhonov regularization parameter satisfies a certain relation with the noise level of the Gaussian process, then the posterior mean of the Gaussian process regression coincides with the estimator given by the kernel ridge regression (see \cite{huszar2012optimally,kimeldorf1970correspondence,o1991bayes} and  \cite{Kanagawa2018} for a review). This equivalence reveals that the probabilistic structure of the Gaussian posterior mean has a nontrivial overlap with the functional structure (that is, as an operator) of the kernel ridge regression, which makes possible the analysis of the estimation error bounds of both together. The rate of pointwise convergence of a Gaussian process regression without noise was studied in \cite{seleznev1999certain,stein1990uniform,yakowitz1985comparison}.
Under the $L^2$ norm, the convergence upper rate of a Gaussian process regression without noise was established in \cite{Kanagawa2018,van2011information,tuo2020kriging}, and in the presence of noise in \cite{lederer2019uniform,wang2022gaussian}.
It is worthwhile to point out that the above literature works for the learning of a generic function with noisy additive data, that is, learning the function $f$ in a model of the type $y = f (x)+\varepsilon$, where the problem of structure preservation is not present. 
In \cite{pfortner2022physics}, a physics-informed Gaussian process regression was proposed for the forward problem of learning solutions of linear PDEs. In contrast, for the inverse problem of learning \eqref{ham-sys}, the task is to recover the scalar Hamiltonian function from either trajectory sampling points in the phase space or the Hamiltonian vector fields at these points. In other words, the task is to recover the Hamiltonian function $H$ from noisy observations of the right-hand side of the equation \eqref{ham-sys}. 
The presence of a gradient in this equation poses technical difficulties and, simultaneously, is the key to the notion of structure preservation. 

{\it This paper proposes a structure-preserving kernel ridge regression method to recover a potentially high-dimensional and nonlinear Hamiltonian function from Hamiltonian vector fields observed in the presence of stochastic noise}.
Our method has the advantage of {\it retaining structure preservation} in the kernel ridge regression solution while simultaneously {\it producing excellent numerical performances}. This important feature is due, in part, to the availability in this context of a generalized Representer Theorem (we call it the {\it Differential Representer Theorem}) that convexifies the estimation problem and circumvents the need to solve convoluted optimization problems, as it is the case for, e.g., HNNs. Moreover, we shall see that the equivalence of the Gaussian posterior mean estimator and the structure-preserving kernel estimator hold in the structure-preserving setup under the same condition as in the general case. 


\paragraph{Conventions and summary of the main results.} The main purpose of this paper is to learn in a structure-preserving fashion the unknown Hamiltonian function $H:\mathbb{R}^{2d}\longrightarrow\mathbb{R}$ of the system \eqref{ham-sys} out of realizations of random samples containing $N$ noisy observations of the Hamiltonian vector field. More explicitly, the observed data consists of $N$ independent random samples of states in the phase space, as well as noisy observations of the Hamiltonian vector fields at the corresponding $N$ states. 

We shall write the associated random samples as: 
\begin{align}
\label{tra-dat}
\begin{split}
\mathbf{Z}_N&:= \mathrm{Vec}\left(\mathbf{Z}^{(1)}|\cdots |\mathbf{Z}^{(N)}\right)\in \mathbb{R}^{2dN}, \\
\mathbf{X}_{\sigma^2,N}&:=\mathrm{Vec}\left(\mathbf{X}_{\sigma^2}^{(1)}|\cdots|\mathbf{X}_{\sigma^2}^{(N)}\right)\in \mathbb{R}^{2dN},
\end{split}
\end{align}
where $\mathbf{Z}^{(n)} \in  \mathbb{R}^{2d}$ is the phase space vector containing the position and the conjugate momenta of the system, and $\left\{\mathbf{Z}^{(1)},\cdots ,\mathbf{Z}^{(N)}\right\} $ are  IID random variables with the same distribution $\mu_{\mathbf{Z}}$. The symbol `$\mathrm{Vec}$' stands for the vectorization of the corresponding matrices and $\mathbf{X}^{(n)}_{\sigma^2} \in  \mathbb{R}^{2d}$ denotes a noisy vector field value at $\mathbf{Z}^{(n)}$, that is, $\mathbf{X}^{(n)}_{\sigma^2}=J\nabla H(\mathbf{Z}^{(n)})+\bm{\varepsilon}^{(n)}$, where $\bm{\varepsilon}^{(n)}$ are IID $\mathbb{R}^{2d} $-valued random variables with mean zero and variance $\sigma^2$. We shall denote by $\mathbf{z}^{(n)}$ and $\mathbf{x}_{\sigma^2}^{(n)}$ the realizations of the random variables  $\mathbf{Z}^{(n)}$ and $\mathbf{X}_{\sigma^2
}^{(n)}$, respectively. We then denote the collection of realizations as
\begin{align}
\label{realization rvs}
\begin{split}
\mathbf{z}_N&:= \mathrm{Vec}\left(\mathbf{z}^{(1)}|\cdots |\mathbf{z}^{(N)}\right)\in \mathbb{R}^{2dN}, \\
\mathbf{x}_{\sigma^2,N}&:=\mathrm{Vec}\left(\mathbf{x}_{\sigma^2}^{(1)}|\cdots |\mathbf{x}_{\sigma^2}^{(N)}\right)\in \mathbb{R}^{2dN}.
\end{split}
\end{align}
In the sequel, if $f:\mathbb{R}^{2d}\to \mathbb{R}^s$ is a function, we then shall denote the value $\mathrm{Vec}\left(f(\mathbf{Z}^{(1)})|\cdots|f(\mathbf{Z}^{(N)})\right)\in \mathbb{R}^{sN}$ by $f(\mathbf{Z}_N)$.

Now, to address the above-mentioned learning problem, we propose a structure-preserving kernel ridge regression method. In contrast to traditional kernel ridge regressions, our approach guarantees that the learned vector field is indeed Hamiltonian. 
Structure-preservation is achieved by searching for vector fields $\mathbf{f}:\mathbb{R}^{2d}\to\mathbb{R}^{2d}$ with Hamiltonian form, that is, $\mathbf{f}_h:=X_h=J\nabla h$, where $h:\mathbb{R}^{2d} \longrightarrow \mathbb{R}$ is an element of the reproducing kernel Hilbert space (RKHS) $\mathcal{H}_K$ associated to a Mercer kernel map $K:\mathbb{R}^{2d} \times  \mathbb{R}^{2d}\longrightarrow \mathbb{R}$ (all these concepts are carefully defined later on). 
More precisely, we will be studying the following optimization problem 
\begin{align}\label{str-min}
 \widehat{h}_{\lambda,N}:=\mathop{\arg\min}\limits_{h\in\mathcal{H}_K} \frac{1}{N}\sum_{n=1}^{N} \left\|X_h(\mathbf{Z}^{(n)})-\mathbf{X}^{(n)}_{\sigma^2}\right\|^2 + \lambda\|h\|_{\mathcal{H}_K}^2,
\end{align}
where $X_h=J\nabla h$ and $\lambda\geq0$ is a Tikhonov regularization parameter.
We call the solution $\widehat{h}_{\lambda,N}$ of the optimization problem  \eqref{str-min} the structure-preserving kernel estimator of the Hamiltonian function. We now summarize the main achievements in connection with the structure-preserving kernel ridge regression approach introduced in the paper, and we put them in relation to its structure.

\begin{enumerate}[leftmargin=*]
\item   In Section \ref{RKHS and Gaussian process regression}, we prove a reproducing property of differentiable Mercer-like kernels on
unbounded sets of the Euclidean space, which is a generalization of similar results proved for compact \cite{novak2018reproducing} and bounded \cite{ferreira2012reproducing} underlying spaces. This property is contained in  Theorem \ref{Par-Rep}, and we call it the {\bf differential reproducing property}; this result, in particular, enables us to embed the RKHS $\mathcal{H}_K $ into the space $C_b^s(\mathbb{R}^{2d})$ ($s\geq 1$)  when the corresponding kernel $K\in C_b^{2s+1}(\mathbb{R}^{2d}\times\mathbb{R}^{2d})$. The symbol $C_b^{s}(\mathcal{X})$ denotes, roughly speaking, the set of bounded functions that exhibit $s$-order partial derivatives which are also bounded on $\mathcal{X}$.
Therefore, if the kernel function $K $ of the RKHS $\mathcal{H}_K $ is regular enough, then the functions in $\mathcal{H}_K$ will be differentiable. Thus the learning inverse problem \eqref{str-min} is well-defined. The condition $K\in C_b^{2s+1}(\mathbb{R}^{2d}\times\mathbb{R}^{2d})$, as it will be seen, is very mild and is satisfied, for example, by the standard Gaussian kernel for any $s\geq 1$. Several interesting byproducts of the embedding theorem are also discussed. 

\item  In Section \ref{Structure-preserving kernel ridge regression}, we point out that if choosing the noise $\bm\varepsilon^{(n)}$ as IID Gaussian random variables, the standard condition existing in the literature \cite{feng2021learning,Kanagawa2018} under which the Gaussian process posterior mean coincides with the kernel ridge regression estimator, namely,
\begin{equation}\label{gau-ord}
    \lambda=\frac{\sigma^2}{N},
\end{equation}
leads to the same conclusion even in the presence of the gradient in our setup. As such, when one incorporates the gradient into the structure-preserving learning scheme, the link between the operator and the probabilistic representation of the estimator can be established just as if the gradient were not involved. Furthermore, we study a Gram-like matrix that involves derivatives of the kernel function (we call it the {\it differential Gram matrix})  which appears in the expression of the Gaussian posterior and in the solution of the kernel ridge regression. An important consequence of our work is that the differential Gram matrix is still positive semidefinite, which ensures that the structure-preserving kernel estimator admits a unique solution that can be written using a closed-form expression. By inspecting the estimator formula, we show that the estimator (at least in the case in which it coincides with the Gaussian posterior mean) can be easily adapted to an online learning scenario to avoid computing the inverse of a possibly large matrix iteratively.

\item   In Section \ref{Convergence analysis and error bounds}, we derive convergence results for the structure-preserving kernel estimator to the ground-truth Hamiltonian under the RKHS norm $\left\|\cdot \right\|_{\mathcal{H}_K} $, which is, on compact sets, strictly stronger than the $L^2$ and $L^{\infty}$ norms. The assumption that the true Hamiltonian $H \in \mathcal{H}_K$ may not be very restrictive if the kernel $K$ is universal (see later at the end of Section \ref{RKHS: Preliminaries}).
The convergence results are derived for both a fixed and a dynamical (that is, adaptive with respect to the sample size $N$) Tikhonov regularization parameter $\lambda$. 

\begin{description}[leftmargin=*]
\item [(i)] For a fixed $\lambda$, we apply the so-called $\Gamma$-convergence technique \cite{dal2012introduction} to show  convergence (see Theorem \ref{Sam-Con}) of the sampling error. In this case, it is not possible to explicitly obtain a convergence rate. However, by fixing $\lambda$ and letting the sample size $N$ to be sufficiently large, it can be shown the reconstruction error $\left\| \widehat{h}_{\lambda,N}-H\right\|_{\mathcal{H}_K}$ can be made arbitrarily small.

\item [(ii)] For a dynamical (adaptive with $N$) regularization parameter $\lambda$ that satisfies $\lambda\propto N^{-\alpha}$, we obtain the upper convergence rate (see Theorems \ref{Sam-Err} and \ref{Tot-Err}) 
\begin{align}\label{tot-rat}
\left\|\widehat{h}_{\lambda,N}-H\right\|_{\mathcal{H}_K} \lesssim N^{-\min\{\alpha\gamma, \frac{1}{2}(1-3\alpha)\}},   
\end{align}
for $\alpha\in(0,\frac{1}{3})$ and $\gamma\in(0,1)$, where $\gamma$ is a regularity parameter associated to the so-called {\it source condition} defined in Section \ref{Convergence analysis and error bounds} which is slightly different from the literature \cite{bauer2007regularization,feng2021learning}. Furthermore, in the presence of the so-called {\it coercivity condition} (see  \eqref{coercivity} later on in the text; this is a modification of the condition with the same name in \cite{feng2021learning}), we can improve this upper convergence rate (see Corollary \ref{Sam-Err2} and Theorem \ref{Tot-Err}) to
\begin{align} \label{tot-rat2}
\left\|\widehat{h}_{\lambda,N}-H\right\|_{\mathcal{H}_K} \lesssim N^{-\min\{\alpha\gamma, \frac{1}{2}(1-2\alpha)\}},    
\end{align}
with $\alpha\in(0,\frac{1}{2})$ and $\gamma\in(0,1)$. Note that in \eqref{gau-ord}, Gaussian process regression corresponds to an order of $\alpha=1$, and hence is not contained in these convergence results. We investigate the convergence analysis via numerical experiments in Section \ref{Numerical experiments}.
We stress that the embedding theorem (Theorem \ref{Par-Rep}) indicates that $\mathcal{H}_K \hookrightarrow C_b^s(\mathbb{R}^{d})$, and therefore the convergence above can also be stated as
\begin{align*}
\left\|\widehat{h}_{\lambda,N}-H\right\|_{C_b^s} \lesssim N^{-\min\{\alpha\gamma, \frac{1}{2}(1-3\alpha)\}}\quad \mathrm{or}\quad N^{-\min\{\alpha\gamma, \frac{1}{2}(1-2\alpha)\}}.  
\end{align*}
\end{description}

\item  In Section \ref{Numerical experiments}, we illustrate the methods presented in the paper by applying the structure-preserving kernel estimator to learn the Hamiltonian functions of various Hamiltonian dynamical systems, including some common Hamiltonian systems, a Hamiltonian system with a highly non-convex potential, as well as systems whose potential functions exhibit singularities. We perform a numerical convergence analysis to illustrate the theoretical bounds derived in Section \ref{Convergence analysis and error bounds}. We also compare the performance and training cost between our algorithm and that of the HNN approach, where the Hamiltonian function is modeled as a neural network and trained with gradient descent.
\end{enumerate}

\paragraph{Comparison with some related results in the literature}
In our structure-preserving scheme, given a scalar-valued function $H: \mathbb{R}^{2d} \rightarrow  \mathbb{R}$ defined on phase space, we consider vector fields of the form $X_H=J\nabla H$, and it is our goal to preserve this structure in the learning process. Our approach is related to the inspiring work in \cite{feng2021learning}, where $n$ interacting particle systems are considered and for which the Hamiltonian function takes the  form 
\begin{align}\label{inter-par-Ham}
H(\mathbf{q},\mathbf{p}) = \sum_{i=1}^n  \frac{1}{2}\|\mathbf{p}_i\|^2 -\frac{1}{2n}\sum_{1\leq i<j\leq n}\Phi(\|\mathbf{q}_i-\mathbf{q}_j\|),
\end{align}
where  $\Phi:\mathbb{R}_+\rightarrow \mathbb{R}$ is the interacting potential satisfying $\Phi'(r)=\phi(r)r$ for $r\in \mathbb{R}_+$. Unlike in \cite{feng2021learning}, where the function $\phi:\mathbb{R}_+\rightarrow \mathbb{R}$ is modeled as a Gaussian process, it is the full Hamiltonian function $H:\mathbb{R}^{2d}\rightarrow \mathbb{R}$ that we model as a Gaussian process in our approach. The increase in dimension of the input space that this strategy implies carries significant technical difficulties in its wake, which we will spell out now. 

First, using the specific form in \eqref{inter-par-Ham}, the work  \cite{feng2021learning} has at its disposal an analytical expression for the Hamiltonian vector field, whereas, for a general unknown Hamiltonian, the gradient operator has to be carried through the entire machine learning framework for the purpose of structure-preservation. Second, due to the presence of gradients in our framework, we have to ensure that the partial derivatives of functions in the RKHS exist and remain in the RKHS (see Section \ref{Reproducing properties of differentiable kernels
on unbounded sets}) so that the learning problem formulated in Section \ref{Structure-preserving kernel ridge regression} is well-posed. Third, the kernel ridge regression in \cite{feng2021learning} deals with a standard Gram matrix that involves exclusively kernel evaluations at phase space data points. This matrix which we denote by $K(\mathbf{Z}_N,\mathbf{Z}_N)$, is automatically positive semidefinite, and hence the regression estimator has a well-defined formula. In our case, however, the analog of the Gram matrix is given by $\nabla_{1,2}K(\mathbf{Z}_N,\mathbf{Z}_N)$, which we call the {\it differential Gram matrix} and we prove that is also positive semidefinite. Finally, we stress that our scheme is not limited to learning Hamiltonian systems but applies to any dynamical system whose vector fields are linear transformations of the gradient operator. Similar observations can be found in the literature in different contexts (see \cite[Section 4]{pfortner2022physics} and references therein).

Second, we want to clarify an important issue in relation to the convergence results that will be shown later on in Section \ref{Convergence analysis and error bounds}. In \cite{feng2021learning}, it is proved that in the presence of the so-called coercivity condition, the regularization parameter $\lambda$ being of order $N^{-\alpha}$, for $0<\alpha<\frac{1}{3}$, is a sufficient condition for the structure-preserving kernel estimator to converge to the ground truth. However, the convergence of the Gaussian process posterior mean as claimed in \cite{feng2021learning} does not hold true since the Gaussian process posterior mean coincides with the structure-preserving kernel estimator if and only if $\alpha=1$. In our framework, we obtain the same range of $\alpha$, that is $0<\alpha<\frac{1}{3}$, which leads to the convergence of the structure-preserving kernel estimator to the ground truth, but {\it without the need to assume the coercivity condition}, and hence allow a much wider choice of kernels than those that can be considered in \cite{feng2021learning,lu2019nonparametric}. Having said that, we shall see in Section \ref{Convergence rates using adaptive Tikhonov regularization parameters} that a modification of the {\it coercivity condition} in \cite{feng2021learning} can be used to enlarge the available scaling indices from $\alpha\in(0,\frac{1}{3})$ to $\alpha\in(0,\frac{1}{2})$ by using a more precise analysis. 

Finally, unlike Hamiltonian neural network approaches \cite{bertalan2019learning,david2023symplectic,greydanus2019hamiltonian,han2021adaptable,toth2019hamiltonian}, our scheme provides an explicit expression for the estimator of the Hamiltonian as the solution of a least-squares problem, which is automatically convex. The training algorithm has a closed-form solution, does not involve gradient descent, and hence, does not suffer from related problems and requires much less training time. Consequently, we expect our proposed method to perform better on small training datasets or when learning highly non-convex objective functions (see Section \ref{Numerical experiments} for illustrations). For large training datasets, the need to compute inverses of large matrices to obtain the structure-preserving kernel estimator can be handled by treating the dataset in an online fashion (see Section \ref{online learning}) as numerous algorithms of this type are available for the online solution of least squares problems. We also point out that even though numerous results exist on the universal approximation properties of artificial neural networks, the analysis of the estimation errors in related algorithms is, in practice, more elusive. In this paper, we can control the approximation and the estimation error simultaneously with the RKHS norm in the presence of noise (see Section \ref{Convergence analysis and error bounds}). To conclude, as we show in Proposition \ref{Dis-Con}, the RKHS framework also yields bounds on the flow prediction error using Gr\"onwall-like inequalities.

The numerical implementation can be found in this GitHub repository:
\url{https://github.com/Learning-of-Dynamic-Processes/LearningHamiltonianEuclideanSpace}.

\if
The remainder of this paper is structured as follows. 
In section \ref{RKHS and Gaussian process regression}, we recall some basics of the kernel method, present the reproducing properties of differentiable kernels and apply the Gaussian process regression to learning Hamiltonian vector fields \eqref{ham-sys}. 
In Section \ref{Structure-preserving kernel ridge regression}, we propose the structure-preserving kernel ridge regression method and show the equivalence between the Gaussian process posterior mean and the structure-preserving kernel regression estimator. 
In section \ref{Convergence analysis and error bounds}, we show the convergence result of the structure-preserving kernel estimator with both fixed and dynamical Tikhonov regularization parameters. Some numerical experiments are presented in Section \ref{Numerical experiments}, followed by a conclusion in Section \ref{Conclusion}. \fi

\section{RKHS and Gaussian process regression}
\label{RKHS and Gaussian process regression} 

This partially introductory section contains most of the results on reproducing kernel Hilbert spaces (RKHS), Gaussian process (GP) regression, and their relation with differentiability, which are needed in the rest of the paper. In Section \ref{RKHS: Preliminaries}, we recall basic definitions and examples of RKHS. Section \ref{Gaussian process regression} reviews Gaussian process regression and the corresponding posterior mean estimator and spells out the case of interest in this paper, that is, the estimation of the Hamiltonian function in \eqref{ham-sys} out of a finite number of observations of the system in phase space and noisy observations of the vector field. 
Finally, in Section \ref{Reproducing properties of differentiable kernels on unbounded sets}, we prove an important sufficient condition (Theorem \ref{Par-Rep}) that guarantees that the functions of an RKHS are differentiable and their derivatives still possess reproducing properties. We call this result the {\bf differential reproducing property}. This theorem generalizes to non-compact setups similar results that already exist for compact \cite{novak2018reproducing, zhou2008derivative} and bounded \cite{ferreira2012reproducing} underlying spaces, guarantees the existence of the gradient of functions in the RKHS, and provides a convenient and powerful way to perform mathematical manipulations on the derivatives that will be much used in the subsequent sections.

\subsection{RKHS: preliminaries} \label{RKHS: Preliminaries}
Let $\mathcal{X}$ be a nonempty set. A  {\bf Mercer kernel} on $\mathcal{X}$ is a positive semidefinite symmetric function $K:\mathcal{X}\times\mathcal{X}\to \mathbb{R}$. Positive semidefinite means that it satisfies 
\begin{align}\label{pos-def}
\sum_{i=1}^n\sum_{j=1}^n c_ic_jK(x_i,x_j)\geq 0,    
\end{align}
for any $x_1,\cdots,x_n\in\mathcal{X}$, $c_1,\cdots,c_n\in\mathbb{R}$, and any $n$. Property \eqref{pos-def} is equivalent to requiring that the Gram matrix $G:=[K(x_i,x_j)]_{i,j=1}^n$ is positive semidefinite for any  $x_1,\cdots,x_n\in\mathcal{X}$ and any given $n$. We list some kernels that are much used in practice. 

\begin{example}[{\bf Gaussian kernel}]
\label{gaussian kernel example}
\normalfont
Let $\mathcal{X} \subset \mathbb{R}^d$. For $\eta>0$, a Gaussian  kernel $K_{\eta}: \mathcal{X} \times \mathcal{X} \rightarrow \mathbb{R}$ is defined by
\begin{equation}
\label{gaussian kernel eta}
{K_{\eta}}\left(\mathbf{x}, \mathbf{y}\right)=\exp \left(-\frac{\left\|\mathbf{x}-\mathbf{y}\right\|^2}{\eta^2}\right), \quad \mathbf{x}, \mathbf{y} \in \mathcal{X}.
\end{equation}
\end{example}

\begin{example}[{\bf Sobolev kernel}]\label{Sobolev_kernell}
\normalfont
Again, let $\mathcal{X} \subset \mathbb{R}^d$. For $d,n\in \mathbb{N}_{+}$ and $n > d/2$, the Sobolev kernel $K_{d,n}$ has the following specific form \cite{novak2018reproducing}
\begin{equation}
\label{sobolev kernel expression}
K_{d, n}(\mathbf{x}, \mathbf{y})=\int_{\mathbb{R}^{d}} \frac{\prod_{j=1}^{d} \cos \left(2 \pi\left(x_j-y_j\right) u_j\right)}{1+\sum_{0<|\alpha| \leq n} \prod_{j=1}^{d}\left(2 \pi u_j\right)^{2 \alpha_j}} \mathrm{~d} u,
\end{equation}
for all $\mathbf{x}, \mathbf{y} \in \mathbb{R}^{d}$, where $x_j, y_j, u _j$, $j \in  \left\{1, \ldots , d\right\} $, are the components of $\mathbf{x}, \mathbf{y}, \mathbf{u}\in \mathbb{R}^{d}$, respectively, $\alpha \in {\mathbb R}^d $ is a vector with positive components and $|\alpha|= \sum_{j =1} ^d \alpha_j $. For $n=\infty$, the kernel is defined as
\begin{equation}\label{sob-inf}
K_{d, \infty}(\mathbf{x}, \mathbf{y})=\prod_{j=1}^{d} \frac{2}{\pi\left(x_j-y_j\right)^3}\left(\sin \left(x_j-y_j\right)-\left(x_j-y_j\right) \cos \left(x_j-y_j\right)\right) \quad \text { for all } \mathbf{x}, \mathbf{y} \in \mathbb{R}^{d}.
\end{equation} 
\end{example}

\noindent A Mercer kernel is the key element to define a {\bf reproducing kernel Hilbert space (RKHS)} as follows.
\begin{definition}[{\bf RKHS}]\label{Rkh-Def} Let $K:\mathcal{X}\times\mathcal{X}\to \mathbb{R}$ be a Mercer kernel on a nonempty set $\mathcal{X}\subseteq \mathbb{R}^d$. A Hilbert space $\mathcal{H}_K$ of real-valued functions on ${\cal X} $ endowed with the pointwise sum and pointwise scalar multiplication, and with inner product $\langle\cdot , \cdot \rangle_{\mathcal{H}_K}$ is called a reproducing kernel Hilbert space (RKHS) associated to $K$ if the following properties hold:
\begin{description}
\item [(i)]  For all $x\in\mathcal{X}$, we have that the function $K(x,\cdot)\in\mathcal{H}_K$.
\item [(ii)]  For all $x\in\mathcal{X}$ and for all $f\in\mathcal{H}_K$, the following reproducing property holds 
$$
f(x)=\langle f,K(x,\cdot)\rangle_{\mathcal{H}_K}.
$$
\end{description}
\end{definition}

The Moore-Aronszajn Theorem \cite{aronszajn1950theory} establishes that given a Mercer kernel $K$ on a set ${\cal X} $, there is a unique Hilbert space of real-valued functions $\mathcal{H}_K$ on ${\cal X} $ for which $K$ is a reproducing kernel, which allows us to talk about the RKHS $\mathcal{H}_K$ associated to $K$. The RKHS $\mathcal{H}_K$ is constructed as follows. Denote 
\begin{align*}
\mathcal{H}_{0}:=\mathrm{span}\left\{K(x,\cdot),x\in \mathcal{X}\right\}=\left\{f=\sum_{i=1}^nc_iK(x_i,\cdot)\mid n\in\mathbb{N},c_1,\cdots,c_n\in\mathbb{R},x_1,\cdots,x_n\in\mathcal{X}\right\}.
\end{align*}
The functions of the form $K _x:=K(x,\cdot): {\cal X} \longrightarrow \mathbb{R}$, $x\in \mathcal{X}$, are called the kernel sections of $K$. The kernel function can be used to define an inner product $\langle\cdot , \cdot \rangle_{{\mathcal H} _0}$ on $\mathcal{H}_{0} $ as the bilinear extension of the assignment $\langle K _x, K _y\rangle_{{\mathcal H} _0}:= K(x,y)  $, $x,y \in {\cal X} $.
The RKHS $\mathcal{H}_K$ is then defined as the completion of the space $\mathcal{H}_0$ with respect to the  norm $\langle\cdot , \cdot \rangle_{{\mathcal H} _0}$. It can be shown (see \cite[Theorem 4.21]{Steinwart2008} for more details) that:
\begin{align*}
\mathcal{H}_{K}=\left\{f=\sum_{i=1}^{\infty}c_iK(x_i,\cdot)\mid c_i\in \mathbb{R},x_i~{\in}~\mathcal{X}, \|f\|^2_{\mathcal{H}_K}=\sum_{i,j=1}^{\infty}c_iK(x_i,x_j)c_j<\infty\right\}.    
\end{align*}
Moreover, it can be shown that functions $f $ in the RKHS $\mathcal{H}_{K}$ inherit many analytical properties of the kernel $K$. For instance, if the $K$ is $s$-times differentiable (in the sense of \cite[Definition 4.35]{Steinwart2008}), then so are the functions in $\mathcal{H}_{K}$ (see \cite[Corollary 4.36]{Steinwart2008}).

\begin{remark}
\normalfont
{\bf (i)} In Definition \eqref{Rkh-Def}, the map $\Phi: {\cal X} \longrightarrow {\mathcal H} _K $ defined by $x\mapsto K(x,\cdot)=:K _x$ is usually called the {\bf canonical feature map}. The kernel $K$ can be written as an inner product in the RKHS of kernel sections, that is,
\begin{align*}
K(x,y)=\langle K(x,\cdot),K(y,\cdot)\rangle_{\mathcal{H}_K}\quad x,y\in\mathcal{X},    
\end{align*}
which follows from the reproducing property. 

\noindent {\bf (ii)} Definition \ref{Rkh-Def} admits the following equivalent formulation. Given a set ${\cal X} $, a Hilbert space $\mathcal{H}$ of real-valued functions on $\mathcal{X}$ is said to be a RKHS if for all $x\in\mathcal{X}$, the evaluation functional $L_x: {\mathcal H} \longrightarrow \mathbb{R} $ defined by  $ f\mapsto f(x)$ from $\mathcal{H}$ to $\mathbb{R}$ is continuous, that is, there exists a finite constant $0<M _x < \infty $, such that 
\begin{align*}
|L_x(f)|:=|f(x)|\leq M_x \|f\|_{\mathcal{H}},\quad \forall\ f\in\mathcal{H}.    
\end{align*}
If that condition is satisfied, the Riesz representation theorem implies that for all $x\in\mathcal{X}$, there exists a unique element $K_x$ of $\mathcal{H}$ with the reproducing property,
\begin{align*}
f(x)=L_x(f)=\langle f,K_x\rangle_{\mathcal{H}} \quad \forall\ f\in\mathcal{H}.    
\end{align*}
Using this observation, we hence define the kernel map $K: {\cal X}\times {\cal X} \longrightarrow \mathbb{R} $ as the function
\begin{equation*}
K(x,y):= \langle K  _x, K _y \rangle_{{\mathcal H}}, \quad \mbox{for all $x,y \in {\cal X} $.}
\end{equation*}
This shows that a RKHS defines a reproducing kernel function that is a Mercer kernel. We already mentioned that given a Mercer kernel $K$ on a set $\mathcal{X}$, by the Moore-Aronszajn Theorem \cite{aronszajn1950theory}, there is a unique Hilbert space of real-valued functions on ${\cal X} $ for which $K$ is a reproducing kernel. Consequently, there is a bijection between RKHSs and Mercer kernels.
\end{remark}

\noindent {\bf Universal kernels.} Consider the RKHS ${\mathcal H} _K  $ defined on a Hausdorff topological space ${\cal X}  $. Let  ${\cal Z}\subset {\cal X}$ be an arbitrary but fixed compact subset of ${\cal X} $ and denote by ${\mathcal H} _K  ({\cal Z})\subset {\mathcal H} _K$ the completion in the RKHS norm of the span of kernel sections determined by the elements of ${\cal Z} $. We write this as:
\begin{equation}
\label{universal kernel}
{\mathcal H} _K  ({\cal Z})= \overline{{\rm span} \left\{K _z\mid z \in {\cal Z}\right\}}.
\end{equation} 
Denote now by $\overline{\mathcal{H}_{K}({\cal Z})}$ the uniform closure of ${\mathcal H} _K  ({\cal Z}) $.
A kernel $K$ is called universal if  for any compact subset $\mathcal{Z}\subset {\cal X}$, we have that $\overline{\mathcal{H}_{K}({\cal Z})}=C({\cal Z})$, with $C({\cal Z})$ the set of real-valued continuous functions on ${\cal Z} $. Equivalently, this implies that for any $\varepsilon>0$, and any function $f\in C(\mathcal{Z})$, there exits a function $g\in\mathcal{H}_{K}({\cal Z})$, such that $\|f-g\|_{\infty}<\varepsilon$.  Many kernels that are used in practice are indeed universal  \cite{micchelli2006universal,steinwart2001influence}, e.g., the Gaussian kernel on Euclidean space. When the kernel $K$ is continuous, the space ${\mathcal H} _K  ({\cal Z}) $ is sometimes simply defined as the uniform closure of ${\rm span} \left\{K _z\mid z \in {\cal Z}\right\} $ and then $K$ is declared to be universal whenever $\mathcal{H}_{K}({\cal Z})=C({\cal Z})$ (see \cite{micchelli2006universal}). This is so because under the continuity hypothesis, the reproducing property and the compactness of ${\mathcal Z} $ imply that the uniform closure that defines $K \left({\mathcal K}\right) $ contains the completion of the vector space ${\rm span} \left\{K _z\mid z \in {\cal Z}\right\} $.

\subsection{Gaussian process regression}
\label{Gaussian process regression}
Gaussian process (GP) regression is a Bayesian nonparametric method for regression widely used in machine learning. Roughly speaking, Gaussian process regression produces a posterior distribution and a likelihood function of the unknown function based on the training data and a given prior distribution. It highly relies on the linearity of the operation acting on the Gaussian prior. In the case of the Hamiltonian system \eqref{ham-sys}, if we model the Hamiltonian function $H$ as a Gaussian prior, then the computation of the Gaussian posterior for $X_H:=J\nabla H$ is feasible since both the gradient $\nabla$ and the left multiplication by the (constant) canonical symplectic matrix $J$ are linear operations.

\begin{definition}[{\bf Gaussian process}] 
Let $(\Omega, \mathcal{F}, \mathbb{P}) $ be a probability space and let  $\mathcal{X}\subseteq\mathbb{R}^{d}$ be a nonempty set. Let $K^\theta:\mathcal{X}\times \mathcal{X} \to \mathbb{R}$ be a symmetric and positive-definite kernel (the symbol $\theta$ denotes a set of kernel parameters traditionally) and let $m:\mathcal{X}\to\mathbb{R}$ be a real function. A function $f:\mathcal{X} \times  \Omega \to \mathbb{R}$ is said to be a Gaussian process (GP) with mean function $m$ and covariance function $K^\theta$ denoted as $f\sim\mathcal{GP}(m,K^\theta)$ if for any $n \in \mathbb{N}$ and any $\mathbf{x}_1,\cdots,\mathbf{x}_n\in \mathcal{X}$, the random variable $(f(\mathbf{x}_1, \cdot ),\cdots,f(\mathbf{x}_n, \cdot))^T$ follows a multivariate normal distribution in $\mathbb{R}^n$ of the form
\begin{align*}
(f(\mathbf{x}_1, \cdot ),\cdots,f(\mathbf{x}_n, \cdot ))^T\sim \mathcal{N}(m_n,K_n^\theta) ,   
\end{align*}
with mean $m_n:=(m(\mathbf{x}_1),\cdots,m(\mathbf{x}_n))$ and covariance $K_n^\theta:=(K^\theta(\mathbf{x}_i,\mathbf{x}_j))_{i,j=1}^n$.
\end{definition}

We now formulate the learning of the Hamiltonian function $H: \mathbb{R}^{2d}\longrightarrow \mathbb{R}$ from observations of the Hamiltonian vector field as a Gaussian process regression. In the sequel, we shall use the following compact notation for the Hamiltonian vector field associated with the Hamiltonian $H: \mathbb{R}^{2d}\longrightarrow \mathbb{R} $:
\begin{align}\label{com-ham}
\mathbf{x}=X_H(\mathbf{z}):=  J\nabla H(\mathbf{z}) \in \mathbb{R}^{2d}. 
\end{align}

\paragraph{Observation data regime.} 
As we explained in \eqref{tra-dat}, the data consists of noise-free observations of the phase space and noisy observations of the vector field that we collectively denote as:
\begin{align}
\label{random variables in dataset 2}
\begin{split}
\mathbf{Z}_N&:= \mathrm{Vec}\left(\mathbf{Z}^{(1)}|\cdots |\mathbf{Z}^{(N)}\right)\in \mathbb{R}^{2dN}, \\
\mathbf{X}_{\sigma^2,N}&:=\mathrm{Vec}\left(\mathbf{X}_{\sigma^2}^{(1)}|\cdots |\mathbf{X}_{\sigma^2}^{(N)}\right)\in \mathbb{R}^{2dN},
\end{split}
\end{align}
where $\mathbf{Z}^{(n)} \in  \mathbb{R}^{2d}$ is the phase space vector containing the position and the conjugate momenta of the system, and $\left\{\mathbf{Z}^{(1)},\cdots ,\mathbf{Z}^{(N)}\right\} $ are  IID random variables with the same distribution $\mu_{\mathbf{Z}}$. The letter $M $ will sometimes denote $M=2dN $. The symbol $\mathrm{Vec}$ stands for the vectorization of the corresponding matrices and $\mathbf{X}^{(n)}_{\sigma^2} \in  \mathbb{R}^{2d}$ denotes a noisy vector field value at $\mathbf{Z}^{(n)}$, that is, 
\begin{equation}
\label{observation regime}
\mathbf{X}^{(n)}_{\sigma^2}=J\nabla H(\mathbf{Z}^{(n)})+\bm{\varepsilon}^{(n)}, 
\end{equation}
where $\bm{\varepsilon}^{(n)}$ are IID $\mathbb{R}^{2d} $-valued random variables with mean zero and variance $\sigma^2$. As we explained in \eqref{realization rvs}, the realizations of the random variables in \eqref{random variables in dataset 2} will be denoted using lowercase as in \eqref{realization rvs}.

\paragraph{The GP regression for the Hamiltonian function.}
We model the Hamiltonian $H: \mathbb{R}^{2d}\longrightarrow \mathbb{R}$ as a GP prior $\mathcal{GP}(0,K^{\theta})$ with zero mean function and covariance function $K^{\theta}:\mathbb{R}^{2d}\times\mathbb{R}^{2d}\to\mathbb{R}$, where $\theta$ represents undetermined kernel parameters, that is, $H \sim\mathcal{GP}(0,K^{\theta})$. 
We now note that by the boundedness of a certain operator that we shall establish later on in Proposition \ref{Wel-Ope} and by \cite[Theorem 1]{pfortner2022physics}, we can state that for any $\mathbf{z},\mathbf{z}'\in\mathbb{R}^{2d}$,
\begin{equation*}
\begin{bmatrix}
\mathbf{x}\\ \mathbf{x}'
\end{bmatrix}
\sim \mathcal{N} \left(
\mathbf{0}
, \begin{bmatrix}
K_{X_H}^{\theta}(\mathbf{z},\mathbf{z}) & K_{X_H}^{\theta}(\mathbf{z},\mathbf{z}')\\ (K^{\theta}_{X_H})^{\top}(\mathbf{z}',\mathbf{z}) &K_{X_H}^{\theta}(\mathbf{z}',\mathbf{z}')
\end{bmatrix}\right),
\end{equation*}
where $K_{X_H}^{\theta}(\mathbf{z},\mathbf{z}')$ is the covariance matrix 
\begin{eqnarray*}
\operatorname{Cov}(X_H(\mathbf{z}),X_H(\mathbf{z}')) =J\nabla_{1,2}K^{\theta}(\mathbf{z},\mathbf{z}')J^{\top},
\end{eqnarray*}
and $\nabla_{1,2}K^{\theta}$ represents the matrix of partial derivatives of $K^{\theta}$ with respect to the first and second arguments. If we assume that the observation noise term $\bm{\varepsilon}^{(n)} $ in \eqref{observation regime} is Gaussian and it is independent of $\mathbf{Z}_N$, we have that
\begin{align*}
\mathbf{X}_{\sigma^2,N} \sim \mathcal{N} \left(\mathbf{0} , K_{X_H}^{\theta}(\mathbf{z}_N,\mathbf{z}_N)+\sigma^2 I_{2dN}\right),
\end{align*}
where $K_{X_H}^{\theta}(\mathbf{z}_N,\mathbf{z}_N)$ denotes the covariance matrix between $X_H(\mathbf{z}_N)$ and $X_H(\mathbf{z}_N)$ whose $(n,n')$-matrix component is given by $J\nabla_{1,2}K^{\theta}(\mathbf{z}^{(n)},\mathbf{z}^{(n')})J^{\top}$.
Thus the negative log marginal likelihood of  $\mathbf{X}_{\sigma^2,N}$ given the data $(\mathbf{z}_N,\mathbf{x}_{\sigma^2,N})$ and parameters $\theta,\sigma$ is
\begin{equation}\label{lik-fun}
\begin{aligned}
-\log p(\mathbf{X}_{\sigma^2,N}|\mathbf{z}_N,\mathbf{x}_{\sigma^2,N},\theta,\sigma^2) =~& \frac{1}{2} (\mathbf{x}_{\sigma^2,N})^{\top}(K_{X_H}^{\theta}(\mathbf{z}_N,\mathbf{z}_N) + \sigma^2I_{2dN})^{-1}\mathbf{x}_{\sigma^2,N}\\ 
+&\ \frac{1}{2}\log\Big|K_{X_H}^{\theta}(\mathbf{z}_N,\mathbf{z}_N)+\sigma^2I_{2dN}\Big| + {dN} \log 2\pi.
\end{aligned}
\end{equation}
By maximizing the likelihood function, that is, minimizing the equation \eqref{lik-fun}, we obtain that the maximum likelihood parameter estimators $\widehat{\theta},\widehat{\sigma}$ are determined by the relations 
\begin{equation*}
\begin{cases}
\mathrm{Tr}\left( \left(D D^{\top} - (K_{X_H}^{\theta}(\mathbf{z}_N,\mathbf{z}_N) + \sigma^2I_{2dN})^{-1}\right) \frac{\partial K_{X_H}^{\theta}(\mathbf{z}_N,\mathbf{z}_N)}{\partial \theta_j}\right)=0,\\
\mathrm{Tr}\left( (D D^{\top} - (K_{X_H}^{\theta}(\mathbf{z}_N,\mathbf{z}_N) + \sigma^2I_{2dN})^{-1}) \right)\sigma=0,
\end{cases}
\end{equation*}
where $D = (K_{X_H}^{\theta}(\mathbf{z}_N,\mathbf{z}_N)+ \sigma^2I_{2dN})^{-1}\mathbf{x}_{\sigma^2,N}$ and $\mathrm{Tr}(\cdot )$ denotes the matrix trace operator.
 
Next, we derive the expression of the Gaussian posterior estimator for predicting $H(\mathbf{z}^\ast)$ at $\mathbf{z}^\ast \in \mathbb{R}^{2d}$, given $\theta$ and $ \sigma$.  

\begin{theorem}\label{Pos-Est}
Suppose that the observation noise term $\bm{\varepsilon}^{(n)} $ in \eqref{observation regime} is Gaussian and it is independent of $\mathbf{Z}_N$. Assume also that the the parameters $\theta, \sigma$ are known and that we are given the training dataset $(\mathbf{z}_N,\mathbf{x}_{\sigma^2,N})$ defined in \eqref{tra-dat},  Then for each $\mathbf{z}^\ast \in \mathbb{R}^{2d}$, $H(\mathbf{z}^{\ast})$ satisfies
\begin{equation*}
H(\mathbf{z}^\ast)|\mathbf{z}_N,\mathbf{x}_{\sigma^2,N} \sim \mathcal{N}(\overline{\phi}_N(\mathbf{z}^{\ast}),\overline{\Sigma}_N(\mathbf{z}^{\ast})),
\end{equation*}
where
\begin{align}
{\overline{\phi}_N}(\mathbf{z}^{\ast})&=K_{H,X_H}^{\theta}(\mathbf{z}^\ast,\mathbf{z}_N)(K_{X_H}^{\theta}(\mathbf{z}_N,\mathbf{z}_N) + \sigma^2I_{2dN})^{-1}\mathbf{x}_{\sigma^2,N},\label{pos-min}\\
\overline{\Sigma}_N(\mathbf{z}^{\ast})&=K^\theta(\mathbf{z}^\ast,\mathbf{z}^\ast) - K_{H,X_H}^{\theta}(\mathbf{z}^\ast,\mathbf{z}_N)(K_{X_H}^{\theta}(\mathbf{z}_N,\mathbf{z}_N) + \sigma^2I_{2dN})^{-1}K_{X_H,H}^{\theta}(\mathbf{z}_N,\mathbf{z}^\ast)\label{pos-var}.
\end{align}
The symbols $K_{X_H,H}^{\theta}(\mathbf{z}_N, \mathbf{z}^*) = (K_{H,X_H}^{\theta})^{\top}(\mathbf{z}^*,\mathbf{z}_N)$ denote the covariance matrix between $X_H(\mathbf{z}_N)$ and $H(\mathbf{z}^*)$ whose $n$-component is given by $\mathrm{Cov}(X_H(\mathbf{z}^{(n)}),H(\mathbf{z}^*))$. 
\end{theorem}

\begin{proof}
Since $X_H(\mathbf{z}_N)$ is defined componentwise by \eqref{com-ham}, for any $\mathbf{z}^\ast \in \mathbb{R}^{d}$, we have that
  \begin{equation*}
    \begin{bmatrix}
    X_H(\mathbf{z}_N)\\
    H(\mathbf{z}^\ast)
    \end{bmatrix}
    \sim \mathcal{N} \left( 0,
    \begin{bmatrix}
    K_{X_H}^{\theta}(\mathbf{z}_N, \mathbf{z}_N) & K_{X_H,H}^{\theta}(\mathbf{z}_N, \mathbf{z}^\ast)\\
    K_{H,X_H}^{\theta}(\mathbf{z}^\ast, \mathbf{z}_N) & K^\theta(\mathbf{z}^\ast,\mathbf{z}^\ast)
    \end{bmatrix}
    \right),
\end{equation*} 
where $K_{X_H}^{\theta}(\mathbf{z}_N, \mathbf{z}_N)$ is the covariance matrix between $X_H(\mathbf{z}_N)$ and $X_H(\mathbf{z}_N)$, and $K_{X_H,H}^{\theta}(\mathbf{z}_N, \mathbf{z}^*) = (K_{H,X_H}^{\theta})^{\top}(\mathbf{z}^*,\mathbf{z}_N)$ is the covariance matrix between $X_H(\mathbf{z}_N)$ and $H(\mathbf{z}^*)$, that is, 
\begin{align*}
K_{X_H,H}^{\theta}(\mathbf{z}_N, \mathbf{z}^*) = \mathrm{Vec}\left(\operatorname{Cov}\left(X_H(\mathbf{z}^{(1)}),H(\mathbf{z}^\ast)\right)|\cdots |\operatorname{Cov}\left(X_H(\mathbf{z}^{(N)}),H(\mathbf{z}^\ast)\right)\right).    
\end{align*}
Note now that by hypothesis  $\mathbf{X}^{(n)}_{\sigma^2}  = X_H(\mathbf{Z}^{(n)}) + \bm\varepsilon^{(n)}$ with IID noise $\bm{\varepsilon}^{(n)} \sim \mathcal{N}(\mathbf{0}, \sigma^2 I_{2d})$ for all $n=1, \cdots, N$. Then by a property shown in Lemma \ref{Con-Gau}, the results follow.
\end{proof}

\subsection{Reproducing properties of differentiable kernels
on unbounded sets}\label{Reproducing properties of differentiable kernels
on unbounded sets}

In applying kernel methods to machine learning, the so-called kernel trick is among the most important features that lead to its success. Mathematically speaking, this trick refers to the fact that one can recover the inner product of two features in the RKHS without knowing what the feature map is. In the setting of structure-preserving machine learning for Hamiltonian systems, the gradient operation has to be carried throughout the algorithm. Therefore, it is natural to hope that the gradient or the partial derivatives of functions in the RKHS remain within the RKHS and enjoy the same or similar reproducing properties. In this section, we give a detailed and rigorous proof of the above statement on unbounded subsets of Euclidean spaces that generalizes similar results in the literature for either compact \cite{novak2018reproducing} or bounded \cite{ferreira2012reproducing} underlying spaces. We also discuss interesting corollaries of our theorem regarding embeddings of function spaces.

We first spell out the mathematical framework. 
We recall some notations introduced, for instance, in \cite{zhou2008derivative}. First, given $s\in \mathbb{N}$, we define the index set $I_s:=\{\alpha\in\mathbb{N}^d:|\alpha|\leq s\}$ where $|\alpha|=\sum_{j=1}^d\alpha_j$ for $\alpha=(\alpha_1,\dots,\alpha_d)\in\mathbb{N}^d$. For a function $f:\mathbb{R}^d\to\mathbb{R}$, we denote its partial derivative $D^{\alpha}f$ (if it exists) as
$$
D^\alpha f(\mathbf{x})=\frac{\partial^{|\alpha|}}{\partial x_1^{\alpha_1} \partial x_2^{\alpha_2} \cdots \partial x_{d}^{\alpha_{d}}} f(\mathbf x) \quad \text { for all } \mathbf x=\left(x_1, x_2, \ldots, x_{d}\right) \in \mathbb{R}^{d}.
$$
Let $C_b^s(\mathbb{R}^d)$ be the set of bounded $s$-continuously differentiable functions with bounded derivatives given by
\begin{equation*}
C_b^s(\mathbb{R}^d)=\left\{ f\in C^s(\mathbb{R}^d)\mid\|f\|_{C_b^s}:=\sup_{\alpha\in I_s}\|D^{\alpha}f\|_{\infty}<\infty  \right\},
\end{equation*}
where the uniform norm $\left\|\cdot \right\|_{\infty} $ is defined by $\|f\|_{\infty}=\sup_{\mathbf x\in\mathbb{R}^d}|f(\mathbf x)|$.
Now, for a kernel function $K:\mathbb{R}^d\times\mathbb{R}^d\to\mathbb{R}$ and any $\alpha,\beta\in \mathbb{N}^d$, we denote 
\begin{equation*}
D^\alpha K(\mathbf{x},{\bf y})=\frac{\partial^{|\alpha|}}{\partial x_1^{\alpha_1} \partial x_2^{\alpha_2} \cdots \partial x_{d}^{\alpha_{d}}} K(x_1,\dots,x_d,y_1,\dots,y_d), \quad \mathbf{x},{\bf y}\in \mathbb{R}^{d},
\end{equation*}
\begin{equation*}
D^{(\alpha,\beta)} K(\mathbf{x},{\bf y})=\frac{\partial^{|\alpha+\beta|}}{\partial x_1^{\alpha_1} \partial x_2^{\alpha_2} \cdots \partial x_{d}^{\alpha_{d}}\partial y_1^{\beta_1} \partial y_2^{\beta_2} \cdots \partial y_{d}^{\beta_{d}}} K(x_1,\dots,x_d,y_1,\dots,y_d), \quad \mathbf{x},{\bf y}\in \mathbb{R}^{d}.
\end{equation*}
Additionally, given $\mathbf{x}\in\mathbb{R}^d$, we denote by $(D^{\alpha}K)_{\mathbf{x}}$ the function on $\mathbb{R}^d$ given by $(D^{\alpha}K)_{\mathbf{x}}({\bf y})=D^{\alpha}K(\mathbf{x},{\bf y})$. By the symmetry of $K$, it holds that
\begin{equation*}
    D^{\alpha}(K_{{\bf y}})(\mathbf{x})=(D^{\alpha}K)_{\mathbf{x}}({\bf y})=D^{\alpha}K(\mathbf{x},{\bf y})\quad \forall~ \mathbf{x},{\bf y}\in\mathbb{R}^{d}.
\end{equation*}

We now state the following reproducing properties of differentiable kernels on $\mathbb{R}^d\times\mathbb{R}^d$, which are a generalization of similar results proved for compact \cite{novak2018reproducing, zhou2008derivative} and bounded \cite{ferreira2012reproducing} underlying spaces. Statements about the boundedness of the partial derivatives of elements in an RKHS can be found in Lemma 4.34 and Corollary 4.36 of \cite{Steinwart2008}.

\begin{theorem}[{\bf Differential reproducing property}]
\label{Par-Rep}
Let $s\in\mathbb{N}$, and $K:\mathbb{R}^{d}\times \mathbb{R}^{d}\rightarrow \mathbb{R}$ be a Mercer kernel such that $K\in C_b^{2s+1}(\mathbb{R}^{d}\times \mathbb{R}^{d})$. 
Then the following statements hold:
\begin{description}
\item [(i)] For any $\mathbf{x}\in \mathbb{R}^{d}$ and $\alpha \in I_s$, $(D^{\alpha}K)_{\mathbf{x}} \in \mathcal{H}_K$.
\item [(ii)]  The following partial derivative reproducing property holds true for any $\alpha \in I_s$:
\begin{equation}\label{dif-rep}
D^{\alpha}f(\mathbf{x})=\langle (D^{\alpha}K)_{\mathbf{x}},f \rangle_{\mathcal{H}_K},\mbox{ for all }  \mathbf{x}\in \mathbb{R}^{d}   \mbox{ and all }   f\in \mathcal{H}_K.
\end{equation}
\item [(iii)] Let $\kappa^2=\|K\|_{C_b^{2s}(\mathbb{R}^{d} \times \mathbb{R}^d)}$. The inclusion $J: \mathcal{H}_K \hookrightarrow C_b^s(\mathbb{R}^{d})$ is well-defined and bounded:
\begin{equation}
\label{embedding ineq}
\|f\|_{C_b^s} \leqslant \kappa\|f\|_{\mathcal{H}_K}  \mbox{ for all }   f \in \mathcal{H}_K .    
\end{equation}
\end{description}
\end{theorem}

The detailed proof of Theorem \ref{Par-Rep} is contained in Appendix \ref{The proof of some theorems}. This theorem has a surprising corollary that will help us later on in Section \ref{Structure-preserving kernel ridge regression} to explain why kernel ridge estimators are unique after the addition of regularization terms. Define:
\begin{align}
\label{halphaspaces}
\mathcal{H}_{\alpha} := \{f\in\mathcal{H}_K\mid D^\alpha f=0\}.  
\end{align}
It can be checked that for each multi-index $\alpha$, $\mathcal{H}_\alpha$ is a closed vector subspace of $\mathcal{H}_K$. Indeed, let $ \left\{f _n\right\}_{n \in \mathbb{N}} $ be a convergent sequence of elements in ${\mathcal H} _\alpha$ such that $\left\|f _n-f\right\|_{{\mathcal H} _K}\rightarrow 0 $ as $n \rightarrow \infty $, for some $f \in {\mathcal H}_K$. Then, by \eqref{embedding ineq} we can conclude that for $s=| \alpha |$,
\begin{equation*}
\left\|D ^\alpha \left(f _n-f\right) \right\| _{\infty}\leq \left\|f _n-f\right\|_{C_b^s} \leqslant \kappa\|f _n-f\|_{\mathcal{H}_K}\rightarrow 0, \quad \mbox{as $n \rightarrow \infty$,}
\end{equation*}
which implies that $D ^\alpha f=0$ and hence that ${\mathcal H}_\alpha $ is a closed subspace of ${\mathcal H}_K$. Hence, the RKHS can be decomposed as 
\begin{align*}
\mathcal{H}_K=\mathcal{H}_{\alpha}\oplus \mathcal{H}_{\alpha}^\bot,
\end{align*}
where $\mathcal{H}_{\alpha}^\bot $ denotes the orthogonal complement of $\mathcal{H}_{\alpha}$ with respect to the RKHS inner product $\left\langle \cdot , \cdot \right\rangle_{\mathcal{H}_K}$.
\begin{corollary}
\label{ker-reproducing-property}
Let $s\in\mathbb{N}$, and $K:\mathbb{R}^{d}\times \mathbb{R}^{d}\rightarrow \mathbb{R}$ be a Mercer kernel such that $K\in C_b^{2s+1}(\mathbb{R}^{d}\times \mathbb{R}^{d})$. Then, for each $\alpha\in I_s$, we have
\begin{align*}
\left\langle (D^\alpha K)_{\mathbf{x}},f\right\rangle_{\mathcal{H}_K}=0, \quad \text{ for all } \mathbf{x}\in \mathbb{R}^d \text{ and } f\in \mathcal{H}_{\alpha}.    
\end{align*}
Equivalently, $(D^\alpha K)_{\mathbf{x}} \in \mathcal{H}_{\alpha}^\bot $,  for all $\mathbf{x}\in \mathbb{R}^d $.
\end{corollary}
\begin{proof}
It is straightforward by letting $f\in\mathcal{H}_{\alpha}$ in the differential reproducing property \eqref{dif-rep}.    
\end{proof}

\begin{remark}
\label{two observations}
\normalfont
{\bf (i)} The condition $K\in C_b^{2s+1}(\mathbb{R}^{d}\times \mathbb{R}^{d})$ in Theorem \ref{Par-Rep} can be replaced by $K\in C_b^{2s}(\mathbb{R}^{d}\times \mathbb{R}^{d})$ and $D^{(\alpha,\beta)}K$ being uniformly continuous on $\mathbb{R}^{2d}$ for all $|\alpha+\beta|=2s$.

\noindent {\bf (ii)} As we see below in expression \eqref{gau-rep}, the RKHS associated to the Gaussian kernel in Example \ref{gaussian kernel example} does not contain any polynomial or non-zero constant functions, which implies that for each $\alpha\in I_s$, $\mathcal{H}_{\alpha} = \{\mathbf{0}\}$. In this case, Corollary \ref{ker-reproducing-property} trivially holds.

\end{remark}

\begin{example}{\bf (The RKHS of the Gaussian kernel is contained in $C_b^{s}(\mathbb{R}^d) $ for all $s \in \mathbb{N}$)}\label{Gau-Ker}
\normalfont
Let $\mathcal{X}\subseteq \mathbb{R}^d$ be a set with a nonempty interior. Consider the Gaussian kernel $K_{\eta}(\mathbf{x},\mathbf{y})=\exp\left(-\frac{\|\mathbf{x}-\mathbf{y}\|^2}{\eta^2}\right)$ with a constant $\eta>0$ for $\mathbf{x},\mathbf{y}\in\mathcal{X}$. One of the main results in \cite{minh2010some} shows that the corresponding Gaussian RKHS $\mathcal{H}_{K_{\eta}}$ is infinite-dimensional and can be written as
\begin{equation}\label{gau-rep}
\mathcal{H}_{K_{\eta}}=\left\{f(\mathbf{x})=e^{-\frac{\|\mathbf{x}\|^2}{\eta^2}}\sum_{|\alpha|=0}^{\infty}w_{\alpha}\mathbf{x}^{\alpha}\mid \|f\|^2_{\mathcal{H}_K}=\sum_{k=0}^{\infty}\frac{k!}{(2/\eta^2)^k}\sum_{|\alpha|=k}\frac{w^2_{\alpha}}{C_{\alpha}^k}<\infty\right\}.
\end{equation}
In this expression $ \alpha=(\alpha_1,\dots,\alpha_d)\in\mathbb{N}^d$, $\mathbf{x} ^\alpha= x _1^{\alpha _1} \cdots x _d^{\alpha _d}$, $w _\alpha \in \mathbb{R} $, and $C _\alpha^k  $ are the multinomial coefficients given by $C _\alpha^k= k !/ \left(\alpha _1 ! \cdots \alpha_d!\right)$. The inner product $\langle \cdot , \cdot \rangle_{\mathcal{H}_{K_{\eta}}}$ is given by
\begin{equation*}
\langle f,g\rangle_{\mathcal{H}_{K_{\eta}}}=\sum_{k=0}^{\infty}\frac{k!}{(2/\eta^2)^k}\sum_{|\alpha|=k}\frac{w_{\alpha}v_{\alpha}}{C_{\alpha}^k},
\end{equation*}
for $f(\mathbf{x})=e^{-\frac{\|\mathbf{x}\|^2}{\eta^2}}\sum_{|\alpha|=0}^{\infty}w_{\alpha}\mathbf{x}^{\alpha}, g(\mathbf{x})=e^{-\frac{\|\mathbf{x}\|^2}{\eta^2}}\sum_{|\alpha|=0}^{\infty}v_{\alpha}\mathbf{x}^{\alpha}\in\mathcal{H}_{K_{\eta}}$. An orthonormal basis for $\mathcal{H}_{K_{\eta}}$ is 
\begin{equation*}
\left\{\phi_{\alpha}(\mathbf{x})=\sqrt{\frac{(2/\eta^2)^kC_{\alpha}^k}{k!}}e^{-\frac{\|\mathbf{x}\|^2}{\eta^2}}\mathbf{x}^{\alpha},\ \ |\alpha|=k,\ \ k=0,\cdots,\infty \right\}.
\end{equation*}
The characterization in \eqref{gau-rep} shows that the functions in $\mathcal{H}_{K_{\eta}}$ are real analytic, and hence $\mathcal{H}_{K_{\eta}}\subseteq C_b^s(\mathcal{X})$ for any $s\in\mathbb{N}$ on any compact set $\mathcal{X}$. When  $\mathcal{X}$ is noncompact, e.g. $\mathcal{X}=\mathbb{R}^d$, it is not easy to check this inclusion property directly using \eqref{gau-rep}. However,  Theorem \ref{Par-Rep} guarantees that the inclusion still holds since we only need to verify that the Gaussian kernel $K\in C_b^{2s+1}(\mathbb{R}^d\times \mathbb{R}^d)$, which can be easily checked.
\end{example}

\begin{example}{\bf (Sobolev Spaces are RKHSs and $H^\infty(\mathbb{R}^d)\subset C_b^s(\mathbb{R}^d)$ for all $s\in\mathbb{N}$)}
\normalfont
The standard Sobolev space $H^n(\mathbb{R}^d)$ consists of $L^2$ functions whose weak derivatives up to order $n$ are square integrable. 
Then $H^n(\mathbb{R}^d)$ is a separable Hilbert space equipped with the inner product
$$
\langle f, g\rangle_{H^n(\mathbb{R}^d)}=\sum_{|\alpha| \leq n}\left\langle D^\alpha f, D^\alpha g\right\rangle_{L^2\left(\mathbb{R}^d\right)}, \quad \text { for all } f, g \in H^n(\mathbb{R}^d) .
$$
We emphasize that here $D^{\alpha}$ denotes the weak derivative, and $L^2\left(\mathbb{R}^d\right)$ is the standard space of square-integrable functions with the inner product
$$
\langle f, g\rangle_{L^2(\mathbb{R}^d)}=\int_{\mathbb{R}^d} f(\mathbf{x}) {g(\mathbf{x})} \mathrm{d} \mathbf{x} .
$$
For $n > d/2$, it can be proved \cite{novak2018reproducing} that the Sobolev space $H^n(\mathbb{R}^{d})$ is the reproducing kernel Hilbert space associated to the kernel $K_{d,n}$ introduced in \eqref{sobolev kernel expression}, that is,
\begin{equation*}
H^n(\mathbb{R}^{d})=\mathcal{H}_{K_{d,n}},
\end{equation*}
and, additionally,  $\|f\|_{\mathcal{H}_{K_{d,n}}}=\|f\|_{H^n(\mathbb{R}^d)}$, for any $f\in H^n(\mathbb{R}^d)$. 

Notice now that the kernel $K_{d, \infty}$ in \eqref{sob-inf} is continuous, symmetric, and positive semidefinite. Hence, it is a Mercer kernel. Furthermore, using Taylor expansion, it can be verified that the kernel $K_{d, \infty}\in C^{2s+1}_b(\mathbb{R}^{d}\times \mathbb{R}^{d})$ for any $s\in\mathbb{N}$. Thus, by Theorem \ref{Par-Rep}, we easily obtain an interesting embedding result, that is, $H^\infty(\mathbb{R}^d)\subseteq C_b^s(\mathbb{R}^d)$ for any $s\in\mathbb{N}$. 
\end{example}

\section{Structure-preserving kernel ridge regression}\label{Structure-preserving kernel ridge regression}

This section presents a structure-preserving kernel ridge regression method to estimate the unknown Hamiltonian function.
We formulate the learning problem as a statistical inverse problem and provide an operator-theoretic framework to represent the kernel regression estimators. Using a generalization of the standard Representer Theorem to our structure-preserving framework that we shall present later on in Theorem \ref{Rep-Ker}, we shall prove that these estimators can be written down as the linear combination of the gradient of the kernel sections evaluated at the dataset.

The kernel ridge regression results we just introduced imply that even though we are in a structure-preserving setup, we can cast the learning problem as the solution of a convex Gramian regression. The convexity feature we just mentioned is a clear comparative advantage with the (potentially non-convex) maximum likelihood problem introduced in \eqref{lik-fun} in relation to the Gaussian process regression approach. This is also why later in Theorem \ref{Brd-Gau}, we establish conditions under which these two estimators, that is, GP and kernel ridge regressions, coincide.

As we already explained in the introduction, for the statistical learning problem, we are given the noisy Hamiltonian vector field random samples
\begin{align}
\label{noisy observations of X}
\mathbf{X}_{\sigma^2}^{(n)}:= X_H(\mathbf{Z}^{(n)})+\bm{\varepsilon}^{(n)}, \quad n=1,\cdots,N,
\end{align}
where $X_H$ is the Hamiltonian vector field of $H: \mathbb{R}^{2d} \longrightarrow\mathbb{R}$ (the function that we have to estimate), $\mathbf{Z}^{(n)}$ are IID random variables with the distribution $\mu_{\mathbf{Z}}$, and $\bm{\varepsilon}^{(n)}$ are $\mathbb{R}^{2d} $-valued IID random variables with mean zero and variance $\sigma^2$ which are independent of $\mathbf{Z}^{(n)}$.

In a standard kernel ridge regression setting, one constructs an empirical quadratic risk functional 
\begin{align}
\label{empirical quadratic risk}
\frac{1}{N}\sum_{n=1}^{N} \|\mathbf{f}(\mathbf{Z}^{(n)})-\mathbf{X}^{(n)}_{\sigma^2}\|^2,
\end{align} 
(to which eventually a quadratic ridge/Tikhonov regularization term is added) and finds the least square (or ridge) estimator of the vector field ${\bf f}$ over a hypothesis function space, which is, in this case, the RKHS associated with a prescribed kernel defined on the phase space. The standard Representer Theorem \cite[Theorem 6.11]{Mohri:learning:2012} guarantees that this convex optimization problem has a unique solution that can be expressed as a linear combination of kernel sections spanned by the dataset.

We stress that using the standard kernel ridge regression to learn the Hamiltonian vector field is not structure-preserving, as there is no guarantee that the vector field that has been learned this way comes from a Hamiltonian system. That is why, later on, in Theorem \ref{Rep-Ker}, we shall adapt the standard Representer Theorem to our structure-preserving framework. We shall see that an analogous result can be formulated in the sense that the solution of the natural empirical risk-minimization problem in our context (see \eqref{emp-pro} and \eqref{emp-fun} below) is not in the span of the kernel sections evaluated at the data (like in the standard Representer Theorem) but of something similar computed with partial gradients of the kernel function. 

\noindent {\bf Structure-preserving kernel ridge regression}\quad
The idea behind structure-preservation in the context of kernel regression is that we search the vector field $\mathbf{f}:\mathbb{R}^{2d}\to\mathbb{R}^{2d}$ that minimizes the risk functional \eqref{empirical quadratic risk} among those that have the form  $\mathbf{f}_{h}:=X_h=J\nabla h$, where $h: \mathbb{R}^{2d} \rightarrow  \mathbb{R}$ belongs to the RKHS $\mathcal{H}_K$ associated with a Mercer  kernel $K$ defined on the phase space $\mathbb{R}^{2d} $. 
This approach obviously guarantees that the learned vector field is Hamiltonian with Hamiltonian function $h$.

In order to make the method explicit, we shall be solving the following minimization problem 
\begin{align}
\widehat{h}_{\lambda,N}&:=\mathop{\arg\min}\limits_{h\in \mathcal{H}_{K}} \ \widehat{R}_{\lambda,N}(h), \label{emp-pro}\\
\widehat{R}_{\lambda,N}(h)&:=\frac{1}{N}\sum_{n=1}^{N} \|X_h(\mathbf{Z}^{(n)})-\mathbf{X}^{(n)}_{\sigma^2}\|^2 + \lambda\|h\|_{\mathcal{H}_K}^2, \label{emp-fun}
\end{align}
where $X_h=J\nabla h$ and $\lambda\geq0$ is the Tikhonov regularization parameter.
We shall refer to the minimizer $ \widehat{h}_{\lambda,N}$ as the {\bf structure-preserving kernel estimator} of the Hamiltonian function $H$. The functional $\widehat{R}_{\lambda,N}$ is referred to as the {\bf regularized empirical risk}. 

The measure-theoretic analogue, referred to as {\bf regularized statistical risk}, is denoted as $R_{\lambda}$ and is defined by
\begin{equation}
\label{exp-fun}
R_{\lambda}(h):=\|X_h- X_H\|_{L^2(\mu_{\mathbf{Z}})}^2+\lambda \|h\|_{\mathcal{H}_K}^2+\sigma^2, 
\end{equation}
where $X_h,X_H$ are the Hamiltonian vector fields of $h$ and $H$, respectively. We denote by $h^{*}_{\lambda}\in\mathcal{H}_K$ the {\bf best-in-class function} with the minimal associated in-class regularized statistical risk, that is,
\begin{align}
\label{exp-pro}
h^{*}_{\lambda}:= \mathop{\arg\min}\limits_{h\in \mathcal{H}_{K}}  R_{\lambda}(h) .  
\end{align}

We say that regularized empirical and statistical risks are {\bf consistent} within the RKHS if for every $h\in\mathcal{H}_K$, we have that
\begin{align}
\label{expectation with respect to epsilon}
\lim\limits_{N\to \infty} \mathbb{E}_{\bm\varepsilon}\left[\widehat{R}_{\lambda, N} (h)\right]=R_{\lambda}(h),\quad a.s., 
\end{align}
where $\mathbb{E}_{\bm\varepsilon}$ means taking the expectation for all the random variables $\bm\varepsilon^{(n)}$.
 
We can show that in our setting, the regularized empirical risk in \eqref{emp-fun} and the regularized statistical risk in \eqref{exp-fun} are consistent. Indeed, denote the empirical measure as $\mu^N_{\mathbf{Z}}:=\frac{1}{N}\sum_{n=1}^N\delta_{\mathbf{Z}^{(n)}}$.
The strong law of large numbers shows that for each $h\in\mathcal{H}_K$, we have
\begin{align*}
\lim\limits_{N\to \infty} \mathbb{E}_{\bm\varepsilon}\left[\widehat{R}_{\lambda, N} (h)\right]&=\lim\limits_{N\to \infty} \mathbb{E}_{\bm\varepsilon}\left[\frac{1}{N}\sum_{n=1}^{N} \|X_h(\mathbf{Z}^{(n)})-\mathbf{X}^{(n)}_{\sigma^2}\|^2\right] + \lambda\|h\|_{\mathcal{H}_K}^2\\
&=\lim\limits_{N\to \infty} \mathbb{E}_{\bm\varepsilon}\left[\frac{1}{N}\sum_{n=1}^{N} \|X_h(\mathbf{Z}^{(n)})-X_H(\mathbf{Z}^{(n)})-\bm\varepsilon^{(n)}\|^2\right] + \lambda\|h\|_{\mathcal{H}_K}^2\\
&=\lim\limits_{N\to \infty} \int_{\mathbb{R}^{2d}}\|X_h(\mathbf{y})- X_H(\mathbf{y})\|^2\mathrm{d}\mu^N_{\mathbf{Z}}(\mathbf{y})+\sigma^2+\lambda\|h\|_{\mathcal{H}_K}^2\\
&=
\int_{\mathbb{R}^{2d}}\|X_h(\mathbf{y})- X_H(\mathbf{y})\|^2\mathrm{d}\mu_{\mathbf{Z}}(\mathbf{y})+\sigma^2+\lambda\|h\|_{\mathcal{H}_K}^2=R_{\lambda}(h),
\end{align*}
almost surely.

\subsection{Operator representations of the kernel ridge learning problem and its solution} 
\label{Operator representations of the kernel ridge learning problem and its solution} 

One important problem that needs to be solved concerning the structure-preserving kernel estimator we just introduced is its convergence to the ground-truth Hamiltonian under the RKHS norm. This shall be carried out
later in Section \ref{Convergence analysis and error bounds}. The first step towards this result is representing the estimator in an operator framework. To represent the minimizers of the inverse learning problems \eqref{emp-pro}-\eqref{emp-fun} and \eqref{exp-fun}-\eqref{exp-pro}, we introduce the operators $A: {\mathcal H}_K \longrightarrow {\mathcal H}_K^{2d}$ and $A_N: {\mathcal H}_K \longrightarrow \mathbb{R}^{2dN}$ defined as
\begin{align}\label{ope-def}
A h=J\nabla h, \text{  and  } A_Nh=\frac{1}{\sqrt{N}}\mathbb{J}\nabla h(\mathbf{Z}_N), \quad \text{ for all } h\in\mathcal{H}_K,
\end{align}
respectively, where $\mathbb{J}=\operatorname{diag}\{J,\ldots,J\}$ (with $N$ diagonal blocks) is a $2dN$-dimensional square matrix and ${\mathcal H}_K^{2d}$ is the Cartesian product made of $2d $ copies of the RKHS ${\mathcal H} _K  $.

By Theorem \ref{Par-Rep}, if a kernel $K\in C_b^{2s+1}(\mathbb{R}^{2d}\times\mathbb{R}^{2d})$ for some $s\in\mathbb{N}^+$, then we can embed its corresponding RKHS into $C^s_b(\mathbb{R}^{2d})$. 
Moreover, the gradients of functions in $\mathcal{H}_K$ exist and their components belong to the same RKHS, which implies that the operators $A: {\mathcal H}_K \longrightarrow {\mathcal H}_K^{2d}$ and $A_N: {\mathcal H}_K \longrightarrow \mathbb{R}^{2dN}$ in \eqref{ope-def} are well-defined. 
We start by studying the properties of the linear operator $A$ in the following proposition.

\begin{proposition}
\label{Wel-Ope}
Let $K\in C_b^3(\mathbb{R}^{2d}\times\mathbb{R}^{2d})$ be a Mercer kernel. Then, the operator $A: {\mathcal H}_K \longrightarrow {\mathcal H}_K^{2d}$ defined in \eqref{ope-def} is a bounded linear operator that maps $\mathcal{H}_K$ into $L^2(\mathbb{R}^{2d};\mu_{\mathbf{Z}};\mathbb{R}^{2d})\subset {\mathcal H}_K^{2d}$ with an operator norm $\|A\|$ that satisfies $\|A\|\leq \sqrt{2d}\kappa$, with $\kappa^2=\|K\|_{C_b^{2}(\mathbb{R}^{2d} \times \mathbb{R}^{2d})}$. The adjoint operator 
$A^{*}: L^2(\mathbb{R}^{2d};\mu_{\mathbf{Z}};\mathbb{R}^{2d}) \longrightarrow {\mathcal H} _K$ of $A: {\mathcal H}_K \longrightarrow L^2(\mathbb{R}^{2d};\mu_{\mathbf{Z}};\mathbb{R}^{2d})$ is given by
\begin{align}\label{adjoint}
A^{*}g=\int_{\mathbb{R}^{2d}} g^{\top}(\mathbf{x})J\nabla_1 K(\mathbf{x},\cdot) \, \mathrm{d} \mu_{\mathbf{Z}}(\mathbf{x}), \quad \mbox{for all $g\in L^2(\mathbb{R}^{2d};\mu_{\mathbf{Z}};\mathbb{R}^{2d})$.}
\end{align} 
As a consequence, {the bounded linear operator $B: \mathcal{H}_K\longrightarrow \mathcal{H}_K$, defined by}
\begin{align}\label{positive}
Bh:=A^{*}A h=\int_{\mathbb{R}^{2d}} \nabla^{\top}h(\mathbf{x}) \nabla_1 K(\mathbf{x},\cdot) \mathrm{~d}\mu_{\mathbf{Z}}(\mathbf{x}),
\end{align}
is a positive semidefinite trace class operator that satisfies $\operatorname{Tr}(B)\leq 2d\kappa^2$.
\end{proposition}

In the case of finite data, that is, $N< \infty$, we adopt the empirical version of $A$, denoted by $A_N$, that we defined in \eqref{ope-def} and that can be used to represent the learning problem \eqref{emp-pro}-\eqref{emp-fun}. The following proposition, which can be considered as an empirical version of Proposition \ref{Wel-Ope}, can be proved for that operator.

\begin{proposition}\label{Wel-Emp} 
Given the phase space and noisy vector field data $(\mathbf{Z}_N,\mathbf{X}_{\sigma^2,N})$, the operator
$A_N: \mathcal{H}_K \rightarrow \mathbb{R}^{2dN}$ defined by
\begin{align*}
A_Nh=\frac{1}{\sqrt{N}}X_h(\mathbf{Z}_N):=\frac{1}{\sqrt{N}}\mathrm{Vec}\left(X_h(\mathbf{Z}^{(1)})| \cdots | X_h(\mathbf{Z}^{(N)})\right),
\end{align*}
is a bounded linear operator.
The adjoint operator $A_N^*:\mathbb{R}^{2dN}  \rightarrow \mathcal{H}_K$ of $A$ is a finite rank operator given by
$$A^*_{N}W=\frac{1}{\sqrt{N}}W^{T}\mathbb{J}\nabla_1 K(\mathbf{Z}_N,\cdot),$$
with $W\in \mathbb{R}^{2dN}$ and where $\mathbb{J}=\operatorname{diag}\{J,\cdots,J\}$ (with $N$ diagonal copies). Moreover, the operator $B_N$ defined by
\begin{align}\label{emp-ope}
B_Nh:=A^*_{N}A_Nh=\frac{1}{N}\nabla^{\top}h(\mathbf{Z}_N)\nabla_1 K(\mathbf{Z}_N,\cdot),    
\end{align}
is a positive semidefinite compact operator. 
\end{proposition}

Having defined the operators $A$ and $A_N$, we are ready to derive an operator representation of the minimizers that solve the inverse learning problems \eqref{emp-pro}-\eqref{emp-fun} and \eqref{exp-fun}-\eqref{exp-pro}.

\begin{proposition}
\label{Rep-Ope} 
Let $\widehat{h}_{\lambda,N}$ and $h_{\lambda}^*$ be the minimizers of \eqref{emp-fun} and \eqref{exp-fun} respectively. Then, for all $\lambda>0$, the minimizers $h_{\lambda}^*$ and $ \widehat{h}_{\lambda,N}$ are unique and they are given by 
\begin{align}
h_{\lambda}^*:&=(B+\lambda I)^{-1}A^{*}X_H. \nonumber \\
 \widehat{h}_{\lambda,N}:&=\frac{1}{\sqrt{N}}(B_N+\lambda I)^{-1}A_N^{*}\mathbf{X}_{\sigma^2,N}, \mbox{ with } B_N=A_N^*A_N.\label{em10} 
\end{align}
\end{proposition}

\begin{proof}
By definition, $ \widehat{R}_{\lambda,N}$ is a functional on $\mathcal{H}_K$. Moreover,  its G\^ateaux derivative is given by
\begin{align*}
\mathrm{d}  \widehat{R}_{\lambda,N}(h) \cdot (\psi)& =\lim_{t\to 0}\frac{ \widehat{R}_{\lambda,N}(h+t\psi)- \widehat{R}_{\lambda,N}(h)}{t} =2 \Big\langle A_N\psi,A_Nh-\frac{1}{\sqrt{N}}\mathbf{X}_{\sigma^2,N}\Big\rangle_{\mathbb{R}^{2dN}}+2\lambda \langle \psi,h\rangle_{\mathcal{H}_K}\\
&=2 \Big\langle \psi,(B_N+\lambda I)h-\frac{1}{\sqrt{N}}A_N^*\mathbf{X}_{\sigma^2,N}\Big\rangle_{\mathcal{H}_K}, \quad \mbox{for any $h, \psi \in \mathcal{H}_K$.} 
\end{align*}
Thus, the critical points of the functional $ \widehat{R}_{\lambda,N}$ are determined by the equality
\begin{align}\label{Cri-eqn}
\Big\langle \psi,(B_N+\lambda )h-\frac{1}{\sqrt{N}}A_N^*\mathbf{X}_{\sigma^2,N}\Big\rangle_{\mathcal{H}_K} = 0 , \quad \mbox{for all} \quad \psi\in \mathcal{H}_K.
\end{align}
Since $B$ and $B_N$ are compact operators, for any $\lambda>0$, the operators $B+\lambda I$ and $B_N+\lambda I$ are bounded and their inverse exists. We note that this is so because by Propositions \ref{Wel-Ope} and \ref{Wel-Emp} the operators $B$  and $B _N  $ are positive. Therefore, the critical equation \eqref{Cri-eqn} has a unique solution $ \widehat{h}_{\lambda,N}:=\frac{1}{\sqrt{N}}(B_N+\lambda I)^{-1}A_N^{*}\mathbf{X}_{\sigma^2,N}$ because of the arbitrariness of $\psi$. Similarly, $h_{\lambda}^*:=(B+\lambda I)^{-1}A^{*}X_H$ is the unique minimizer of the regularized statistical risk \eqref{exp-fun}.
\end{proof}

We now prove a fact that will be needed later on in the paper in connection to what we call the {\bf differential Gram matrix} that we define as $\nabla_{1,2}K(\mathbf{Z}_N,\mathbf{Z}_N) $ (the symbol $\nabla_{1,2}$ denotes partial derivatives with respect to all the entries in $K$). More specifically, we now show that the differential Gram matrix is positive semidefinite and that hence $\nabla_{1,2}K(\mathbf{Z}_N,\mathbf{Z}_N)+\lambda NI$ is invertible for any $\lambda>0$, where $I$ is the identity matrix. 

\begin{proposition}\label{Pos-Gra}
Given a Mercer kernel $K$  such that $K\in C_b^3(\mathbb{R}^{2d}\times \mathbb{R}^{2d})$, the Gram matrix $\nabla_{1,2}K(\mathbf{Z}_N,\mathbf{Z}_N)$ is positive semidefinite. 
\end{proposition}

\begin{proof}
Denote $M=2dN$. Since the differential Gram matrix $\nabla_{1,2}K(\mathbf{Z}_N,\mathbf{Z}_N)$ is real symmetric, then so is $\mathbb{J}\nabla_{1,2}K(\mathbf{Z}_N,\mathbf{Z}_N)\mathbb{J}^{\top}$ and hence there exists an orthonormal matrix $P\in\mathbb{R}^{M\times M}$ that diagonalizes $\mathbb{J}\nabla_{1,2}K(\mathbf{Z}_N,\mathbf{Z}_N)\mathbb{J}^{\top}$. This means that
\begin{align*}
\mathbb{J}\nabla_{1,2}K(\mathbf{Z}_N,\mathbf{Z}_N)\mathbb{J}^{\top}&=PDP^{\top}\\
&=\begin{bmatrix}
|&|&\dots&|\\
f_1&f_2&\dots&f_M\\
|&|&\dots&|
\end{bmatrix}
\begin{bmatrix}
d_1&0&\dots&0 \\
0&d_2&\dots&0 \\
\vdots&\vdots&\ddots&\vdots\\
0&0&\dots&d_M
\end{bmatrix}\begin{bmatrix}
|&|&\dots&|\\
f_1&f_2&\dots&f_M\\
|&|&\dots&|
\end{bmatrix}^{\top},
\end{align*} 
where  $\left \{d_i \right\}_{i=1}^{M}$ and $\left \{f_i \right\}_{i=1}^{M}$ are the real eigenvalues and the corresponding eigenvectors of $\mathbb{J}\nabla_{1,2}K(\mathbf{Z}_N,\mathbf{Z}_N)\mathbb{J}^{\top}$. Note that $\left \{d_i \right\}_{i=1}^{M}$ are also eigenvalues of $\nabla_{1,2}K(\mathbf{Z}_N,\mathbf{Z}_N)$.
We now define $\widetilde{e}_i=\langle f_i, \mathbb{J}\nabla_1K(\mathbf{Z}_N,\cdot ) \rangle_{\mathbb{R}^{M}}$. By part {\bf (i)} of Theorem \ref{Par-Rep}, we obtain that $\widetilde{e}_i\in\mathcal{H}_K$ and that
\begin{equation}
\begin{aligned}\label{non-bas}
\|\widetilde{e}_i\|_{\mathcal{H}_K}^2 &= \left\langle\langle f_i, \mathbb{J}\nabla_1K(\mathbf{Z}_N,\cdot ) \rangle_{\mathbb{R}^{M}},\langle f_i, \mathbb{J}\nabla_1K(\mathbf{Z}_N,\cdot ) \rangle_{\mathbb{R}^{M}}\right\rangle_{\mathcal{H}_K} \\
&=f_i^{\top}\mathbb{J}\nabla_{1,2}K(\mathbf{Z}_N,\mathbf{Z}_N)\mathbb{J}^{\top}f_i=f_i^{\top}(d_i f_i)
=d_i,
\end{aligned}
\end{equation}
where in the second equality we used part {\bf (ii)} of Theorem \ref{Par-Rep}. Hence, $d_i\geq0$ for all $i=1,\cdots,M$ and hence we can conclude that the differential Gram matrix $\nabla_{1,2}K(\mathbf{Z}_N,\mathbf{Z}_N)$ is positive semidefinite.
\end{proof}

\begin{remark}
\normalfont
It follows from the definition of positive semidefinite kernels that the usual Gram matrix $K(\mathbf{Z}_N, \mathbf{Z}_N)$ is positive semidefinite. Proposition \ref{Pos-Gra} shows that just the additional hypothesis $K\in C_b^3(\mathbb{R}^{2d}\times \mathbb{R}^{2d})$ (necessary to invoke the differential reproducing property in Theorem \ref{Par-Rep}) suffices for the differential Gram matrix $\nabla_{1,2}K(\mathbf{Z}_N,\mathbf{Z}_N)$ to satisfy the same property.
\end{remark}

\noindent {\bf The Differential Representer Theorem and the solution of the structure-preserving kernel ridge regression}\quad
Below, we derive what we call a Differential Representer Theorem to make an explicit distinction from the usual Representer theorem. This result shows that the estimator $\widehat{h}_{\lambda, N}$ introduced in \eqref{emp-pro}, that is, the minimizer of the regularized empirical risk functional, can be written as a linear combination of the partial derivatives of the kernel function $K(\mathbf{z},\cdot)$ with respect to the components of the $\mathbf{z}$ variable, and then evaluated in the dataset $\mathbf{Z}_N$.

\begin{theorem}[{\bf Differential Representer Theorem for Symplectic Vector Spaces}]
\label{Rep-Ker} 
For every $\lambda>0$, the optimization problem \eqref{emp-pro} has a unique solution $\widehat{h}_{\lambda, N}$ that can be written as
\begin{equation}
\label{rep-ker}
 \widehat{h}_{\lambda,N}= \sum_{i=1}^N \langle \widehat {\bf c}_{i},\nabla_1K({\bf Z}^{(i)},\cdot)\rangle,
\end{equation}
with $\widehat {\bf c}_{1}, \ldots, \widehat {\bf c}_{N} \in \mathbb{R}^{2d}$, $\left\langle \cdot , \cdot \right\rangle $ the Euclidean inner product in $\mathbb{R}^{2d} $, and where $\nabla_1K(\mathbf{z},\cdot) \in \mathbb{R}^{2d}$ denotes the gradient of $K$ with respect to the $\mathbf{z}$ variable. Moreover, if we denote by $\widehat{\mathbf{c}} \in \mathbb{R}^{2dN}$ the vectorization of $\left(\widehat {\bf c}_{1}| \cdots | \widehat {\bf c}_{N}\right)$, then we have
\begin{align*}
\widehat{\mathbf{c}}=(\nabla_{1,2}K(\mathbf{Z}_N,\mathbf{Z}_N)+\lambda NI)^{-1}\mathbb{J}^{\top}\mathbf{X}_{\sigma^2,N}.
\end{align*}
The matrix $\nabla_{1,2}K(\mathbf{Z}_N,\mathbf{Z}_N) $ is the differential Gram matrix defined above.
\end{theorem}

\begin{proof} The proof is based on the operator representations of the minimizers that we introduced in Proposition \ref{Wel-Emp}, which allows us to use tools from the spectral theory. 
Let $\mathcal{H}_K^{N}$ be the space given by
\begin{equation}\label{sub-spa}
\mathcal{H}_K^{N}:=\mathrm{span}\left\{(\nabla_1K)_i(\mathbf{Z}^{(j)},\cdot)\mid i=1,\cdots,2d, \ j=1, \ldots , N\right\},
\end{equation}
where $(\nabla_1K)_i(\mathbf{z},\cdot)$ denotes the $i$-th component of the gradient $\nabla_1K(\mathbf{z},\cdot)$ of $K$ with respect to $\mathbf{z}$. obviously $\mathcal{H}_K^N$ is a subspace of $\mathcal{H}_K$ since $(\nabla_1K)_i(\mathbf{z},\cdot)\in\mathcal{H}_K$ for all $\mathbf{z}\in \mathbb{R}^{2d}$ and $i=1,\cdots,2d$ by Theorem \ref{Par-Rep}. Then by the representation of the operator $B_N$ in Proposition \ref{Wel-Emp}, we know that $B_N(\mathcal{H}_{K}^N) \subseteq \mathcal{H}_{K}^N$ (see the expression \eqref{emp-ope}), that is, $\mathcal{H}_{K}^N$ is an invariant space for the operator $B_N$. This implies that, for any $\lambda>0 $, $(B_N+ \lambda I)(\mathcal{H}_{K}^N) \subseteq \mathcal{H}_{K}^N$. Now, since by Proposition \ref{Wel-Emp} the operator $B _N $ is positive semidefinite, we can conclude that the restriction $(B_N+ \lambda I)|_{\mathcal{H}_{K}^N} $ is invertible and since the space $\mathcal{H}_{K}^N  $ is finite-dimensional then it is also an invariant subspace of $(B_N+ \lambda I)|_{\mathcal{H}_{K}^N}^{-1} $, that is 
\begin{equation*}
(B_N+ \lambda I)|_{\mathcal{H}_{K}^N}^{-1} \left(\mathcal{H}_{K}^N\right) \subset \mathcal{H}_{K}^N.
\end{equation*}
Thus, there exist vectors $\widehat {\bf c}_{1}, \ldots, \widehat {\bf c}_{N} \in \mathbb{R}^{2d}$ such that 
\begin{equation}\label{rep-ker1}
\widehat{h}_{\lambda,N}= \sum_{i=1}^N \langle \widehat {\bf c}_{i},\nabla_1K(\mathbf{{\bf Z}}^{(i)},\cdot)\rangle.
\end{equation}
Then, applying $(B_N + \lambda I)$ on both sides of \eqref{em10}, plugging \eqref{rep-ker1} into the identity, and denoting by $\widehat{\mathbf{c}} \in \mathbb{R}^{2dN}$ the vectorization of $\left(\widehat {\bf c}_{1}| \cdots | \widehat {\bf c}_{N}\right)$, we obtain
\begin{align}
\label{euler equ in this case}
\widehat{\mathbf{c}}^{\top} \left (\frac{1}{N}\nabla_{1,2}K(\mathbf{Z}_N,\mathbf{Z}_N)+\lambda I\right) \nabla_1 K(\mathbf{Z}_N,\cdot)=\frac{1}{\sqrt{N}}A_N^{*}\mathbf{X}_{\sigma^2,N} = \frac{1}{N}\mathbf{X}_{\sigma^2,N}^{\top}\mathbb{J}\nabla_1 K(\mathbf{Z}_N,\cdot).
\end{align}
Since the matrix $\nabla_{1,2}K(\mathbf{Z}_N,\mathbf{Z}_N)+\lambda N I$ is invertible due to the positive semidefiniteness of the differential Gram matrix $\nabla_{1,2}K(\mathbf{Z}_N,\mathbf{Z}_N)$ that we proved in Proposition \ref{Pos-Gra}, we can write the expression 
\begin{align}
\label{expression for chat}
\widehat{\mathbf{c}}= (\nabla_{1,2}K(\mathbf{Z}_N,\mathbf{Z}_N)+\lambda NI)^{-1}\mathbb{J}^{\top} \mathbf{X}_{\sigma^2,N},
\end{align}
that a straightforward verification shows that plugged into \eqref{euler equ in this case} satisfies \eqref{euler equ in this case}. This shows that the function $\widehat{h}_{\lambda,N} $ in \eqref{rep-ker1} with $\widehat{\mathbf{c}} $ determined by \eqref{expression for chat} is a minimizer of the regularized empirical risk functional $ \widehat{R}_{\lambda,N}$ in \eqref{emp-fun}. Since by Proposition \ref{Rep-Ope}, this minimizer is unique, the result follows.
\end{proof}

\begin{remark}\normalfont
In this paper, we employ the classical Tikhonov regularization technique for our learning problems, which is a typical choice among modern regularization methods for inverse problems \cite{engl1996regularization,benning2018modern}.
This regularization term plays a crucial role in ensuring the well-posedness of the minimization problems (3.3)-(3.4) and (3.5)-(3.6), in combination with the differential reproducing property in Theorem \ref{Par-Rep}.
Define the kernel of the operator $A$ as follows:
\begin{align*}
\mathcal{H}_{\mathrm{null}}:=\{h\in\mathcal{H}_K\mid Ah=X_h=0\}=\{f\in\mathcal{H}_K\mid\nabla h=0\}.
\end{align*}
In general, the space $\mathcal{H}_{\mathrm{null}}$ contains non-zero constant functions (unlike the case of the Gaussian kernel which, as we saw in Remark \ref{two observations} {\bf (ii)}, is trivial). Moreover,
using the notation introduced in \eqref{halphaspaces} we can write
\begin{equation}
\label{expressionnull}
\mathcal{H}_{\operatorname{null}}= \bigcap\limits_{i=1}^{2d} {\mathcal H}_{e_i},
\end{equation}
where the vectors $e _i\in \mathbb{R}^{2d} $ are the elements of the canonical basis. Moreover, by Corollary \ref{ker-reproducing-property} and the expression \eqref{rep-ker} it is clear that
\begin{equation}
\label{expressionnullperp}
\widehat{h}_{\lambda,N} \in \mathcal{H}_{\operatorname{null}}^\bot=\bigoplus\limits_{i=1}^{2d} {\mathcal H}_{e_i}^\bot
\end{equation}
It is clear that adding elements in $\mathcal{H}_{\mathrm{null}} $ to $\widehat{h}_{\lambda,N} $ does not change the corresponding Hamiltonian vector field and hence we may wonder why, according to Theorem \ref{Rep-Ker}, the optimizer $\widehat{h}_{\lambda,N} $ is unique. The explanation for this fact is in the use of the regularization term. Indeed, let $\widehat{h}_{\lambda,N}$ be the minimizer in \eqref{rep-ker}  and let $h\in\mathcal{H}_{\mathrm{null}}$. Even though $\widehat{h}_{\lambda,N}$ and $\widehat{h}_{\lambda,N}+h$ have the same Hamiltonian vector field associated, it is easy to show that $\widehat{h}_{\lambda,N}+h$ is a minimizer of \eqref{emp-pro} if and only if $h\equiv0$. This is because
\begin{align*}
\widehat{R}_{\lambda,N}(\widehat{h}_{\lambda,N}+h)&:=\frac{1}{N}\sum_{n=1}^{N} \|X_{\widehat{h}_{\lambda,N}+h}(\mathbf{Z}^{(n)})-\mathbf{X}^{(n)}_{\sigma^2}\|^2 + \lambda\|\widehat{h}_{\lambda,N}+h\|_{\mathcal{H}_K}^2  \\
&=\frac{1}{N}\sum_{n=1}^{N} \|X_{\widehat{h}_{\lambda,N}}(\mathbf{Z}^{(n)})-\mathbf{X}^{(n)}_{\sigma^2}\|^2 + \lambda\left(\|\widehat{h}_{\lambda,N}\|_{\mathcal{H}_K}^2+\|h\|^2_{\mathcal{H}_K} +2\langle \widehat{h}_{\lambda,N},h\rangle_{\mathcal{H}_K}\right) \\
&=\frac{1}{N}\sum_{n=1}^{N} \|X_{\widehat{h}_{\lambda,N}}(\mathbf{Z}^{(n)})-\mathbf{X}^{(n)}_{\sigma^2}\|^2 + \lambda\left(\|\widehat{h}_{\lambda,N}\|_{\mathcal{H}_K}^2+\|h\|^2_{\mathcal{H}_K} \right),
\end{align*}
where the last equality is due to \eqref{expressionnull} and \eqref{expressionnullperp}.  
\end{remark}

\begin{remark}
\normalfont
{\bf (i)} The main difference between the Differential Representer Theorem \ref{Rep-Ker} and the usual Representer Theorem is that the gradient of the kernel function is involved in the statement. The usual Representer Theorem asserts that the minimizer of certain regularized losses can be represented as a linear combination of kernel sections evaluated at the data points, whereas in our case, the minimizer is given by a linear combination of partial derivatives of the kernel, also evaluated at the data points. 

\noindent {\bf (ii)}  In this paper, we have set to learn Hamiltonian functions from observed Hamiltonian vector fields. Nevertheless, this framework can be applied to learn gradient systems or, more generally, vector fields generated by an arbitrary linear transformation of the gradient of differentiable functions. Besides, our approach has a strong connection to solving ill-posed linear PDEs~\cite{schaback2006kernel}, which can be formulated as follows:
\begin{align*}
B h = f,
\end{align*}
where $B$ is defined in (3.10), and $f := A_{N}^* \mathbf{X}_{\sigma^2,N}$ is derived from the data observations. This type of ill-posed linear PDEs is usually approximated by adding a regularization term, such as the quadratic Tikhonov regularization \cite{plato2018optimal} or non-quadratic regularization \cite{burger2004convergence}.
However, compared to traditional pseudo-differential operators~\cite{hormander2007analysis}, the operator $B$ depends in our case on the measure used for data sampling, which even though it introduces new challenges in analyzing the error bounds, it also leads to important applications in the machine learning setting.
\end{remark}

\subsection{Equivalence of the Gaussian posterior mean estimator and the structure-preserving kernel estimator}

In this subsection, we establish operator representations of the Gaussian process posterior mean estimator and the marginal variance given in Theorem \ref{Pos-Est}. Moreover, we shall prove that if the regularization constant $\lambda$ in the kernel ridge regression problem that we solved in Theorem \ref{Rep-Ker} is set to the value $\lambda=\frac{\sigma^2}{N}$, then the posterior mean estimator obtained out of the GP approach and structure-preserving kernel estimator coincide. Although a relation of this type has already been pointed out in the literature \cite{Kanagawa2018}, we emphasize that our result is proved in the presence of a gradient and in the structure-preserving machine learning setup in which the developments of this paper take place.
 
In Theorem \ref{Pos-Est}, we showed that the posterior mean and marginal variance estimators can be written using gradients of the kernel function evaluated at the data set. Thus, by leveraging Proposition \ref{Rep-Ope} and Theorem \ref{Rep-Ker}, we immediately obtain the following operator representations for the posterior mean and marginal variance estimators.
This shows, in particular, under what condition the posterior mean estimator and the structure-preserving kernel estimator coincide. 

\begin{theorem}
\label{Brd-Gau}  
Suppose $H\sim \mathcal{GP} (0, K)$ with covariance function $K\in C_b^3(\mathbb{R}^{2d}\times\mathbb{R}^{2d})$ and that the observation noise term $\bm{\varepsilon}^{(n)} $ in \eqref{observation regime} is Gaussian and it is independent of $\mathbf{Z}_N$. Choose now the covariance function $K$ as the kernel of the RKHS $\mathcal{H}_K$ in the inverse learning problem \eqref{emp-pro}-\eqref{emp-fun}. If $\lambda=\frac{\sigma^2}{N}$, then it holds that
\begin{description}
\item [(i)]  The posterior mean estimator $\overline\phi_N$ in \eqref{pos-min} has the operator representation
\begin{align*}
\overline{\phi}_N = \widehat{h}_{\lambda,N} =\frac{1}{\sqrt{N}}(A_N^*A_N+\lambda)^{-1}A_N^{*}\mathbf{x}_{\sigma^2,N}.
\end{align*}   
\item [(ii)] The marginal posterior variance $\overline{\Sigma}_N$ in  \eqref{pos-var} can be written as 
\begin{align*}
\overline{\Sigma}_N(\mathbf{z}^{\ast})=K(\mathbf{z}^{\ast},\mathbf{z}^{\ast})-K^{\lambda,N}(\mathbf{z}^{\ast},\mathbf{z}^{\ast}),
\end{align*} 
where $
K^{\lambda,N}(\mathbf{z}^{\ast},\mathbf{z}^{\ast}):=\frac{1}{\sqrt{N}}[(A_N^*A_N+\lambda I)^{-1}A_N^{*}K_{X_H,H}(\mathbf{z}_N,\mathbf{z}^\ast)](\mathbf{z}^{\ast})$.
\end{description}
\end{theorem}

\begin{proof}
Note that $K_{X_H}(\mathbf{z}_N,\mathbf{z}_N)=\mathbb{J}\nabla_{1,2}K(\mathbf{z}_N,\mathbf{z}_N)\mathbb{J}^{\top}$. Thus if  $\lambda=\frac{\sigma^2}{N}$, we obtain $\overline{\phi}_N = \widehat{h}_{\lambda,N}$ by combining equations \eqref{pos-min} and \eqref{rep-ker}.  Then by the operator representation of the structure-preserving kernel estimator $ \widehat{h}_{\lambda,N}$ in Proposition \ref{Rep-Ope}, we have
\begin{align*}
\overline{\phi}_N  =\frac{1}{\sqrt{N}}(A_N^*A_N+\lambda)^{-1}A_N^{*}\mathbf{x}_{\sigma^2,N}.
\end{align*}   

Moreover, we get the operator representation of the marginal posterior variance $\overline{\Sigma}_N$ in equation \eqref{pos-var} as follows
\begin{align*}
\overline{\Sigma}_N(\mathbf{z}^{\ast})&=K(\mathbf{z}^\ast,\mathbf{z}^\ast) - K_{H,X_H}(\mathbf{z}^\ast,\mathbf{z}_N)(K_{X_H}(\mathbf{z}_N,\mathbf{z}_N) + \sigma^2I_{2dN})^{-1}K_{X_H,H}(\mathbf{z}_N,\mathbf{z}^\ast)\\
&=K(\mathbf{z}^\ast,\mathbf{z}^\ast) - \frac{1}{\sqrt{N}}[(A_N^*A_N+\lambda I)^{-1}A_N^{*}K_{X_H,H}(\mathbf{z}_N,\mathbf{z}^\ast)](\mathbf{z}^{\ast}).
\end{align*}
The results follow.
\end{proof}

\begin{remark}
\normalfont
We emphasize that, because of the previous result, the equivalence of the posterior mean estimator and the structure-preserving kernel estimator holds true {\it if and only if} $\lambda=\frac{\sigma^2}{N}$.
\end{remark}

\begin{remark}
\normalfont
A non-trivial consequence of this result is that the posterior mean estimator $\overline\phi_N$ belongs to the RKHS ${\mathcal H}_K $. This does not imply or is related to the paths of $\mathcal{GP} (0, K)$ being in ${\mathcal H}_K $ which, as it has been pointed out in \cite{Kanagawa2018}, is not the case almost surely. A larger RKHS induced by ${\mathcal H}_K $ can be nevertheless constructed for which this is true. We encourage the reader to check with \cite{Kanagawa2018} for illuminating discussions on this point and for a comprehensive account of the classical results on this topic.
\end{remark}

\begin{remark}\normalfont
Online and lifelong learning are important topics in machine learning dealing with the issue of updating a model when observed data comes in a streamed fashion without the need to use the entire past dataset. With deep neural networks, the problem of online learning can be challenging and is typically dealt with in a case-by-case manner. Fortunately, for regression problems, there has been a natural solution to online learning that can be found in \cite{Nadungodage2021}, and the particular case of the kernel ridge regression in \cite{van2014online}. These strategies aim at updating the estimator cheaply when new data arrives. Since our learning framework has a regression nature, the above-mentioned schemes can be readily modified to achieve online learning.    
\end{remark}

\section{Estimation and approximation error analysis }
\label{Convergence analysis and error bounds}
  
In this section, we establish a rigorous quantitative framework that analyzes the ability of the structure-preserving kernel estimator $ \widehat{h}_{\lambda,N}$ to recover the unknown Hamiltonian $H$. 
A standard approach in this setup is to decompose the {\bf reconstruction error} $\widehat{h}_{\lambda,N}-H $ as the sum of what we shall be calling the {\bf estimation} and {\bf approximation errors}. 
\begin{align} 
\label{error first decomposition}
\widehat{h}_{\lambda,N}-H =\underbrace{\widehat{h}_{\lambda,N}-h_{\lambda}^*}_{\text{Estimation error}}\quad+\underbrace{h_{\lambda}^*-H }_{\text{Approximation error}},
\end{align}
where we recall that  $h^{*}_{\lambda}\in\mathcal{H}_K$ the best-in-class function introduced in \eqref{exp-pro} that minimizes the regularized statistical risk. The estimation error comes from two sources: the randomness in the sampling and the randomness in the noise term that we added in \eqref{noisy observations of X} to the observations of the Hamiltonian vector field. Thus, using the operator representations that we introduced in Propositions \ref{Wel-Emp} and \ref{Rep-Ope}, we further decompose the estimation error into what we shall call the {\bf sampling error} and its {\bf noisy sampling error} parts: 
\begin{equation} 
\begin{aligned}\label{min-dec}
\widehat{h}_{\lambda,N}-h_{\lambda}^*
&= \frac{1}{\sqrt{N}}{ (B_N+\lambda)^{-1}A_N^{*}\mathbf{X}_{\sigma^2,N}-{h_{\lambda}^*}}= { (B_N+\lambda)^{-1}B_NH+\frac{1}{\sqrt{N}}(B_N+\lambda)^{-1}A_N^{*}\mathbf{E}_N-{h_{\lambda}^*}}\\
:&= \widetilde h_{\lambda,N}-h_{\lambda}^*+\frac{1}{\sqrt{N}}(B_N+\lambda)^{-1}A_N^{*}\mathbf{E}_N,
\end{aligned}
\end{equation}
where the noise vector $\mathbf{E} _N  $ is defined as 
\begin{equation*}
\mathbf{E}_N=\mathrm{Vec}\left(\bm{\varepsilon}^{(1)}|\cdots |\bm{\varepsilon}^{(N)}\right)\in \mathbb{R}^{2dN},
\end{equation*}
and by hypothesis follows a multivariate distribution with zero mean and variance $\sigma^2I_{2dN}$. 
The noise-free term $\widetilde{h}_{\lambda,N}:=(B_N+\lambda)^{-1}B_NH$ is actually the unique minimizer of the following noise-free minimization problem
\begin{align}
\widetilde{h}_{\lambda,N}: & =\argmin{h\in \mathcal{H}_K}\widetilde{R}_{\lambda,N}(h),\label{fre-pro}
\\
\widetilde{R}_{\lambda,N}(h):&=\|A_Nh-A_NH\|^2+\lambda \|h\|_{\mathcal{H}_K}^2. \nonumber
\end{align}
The functional $\widetilde{R}_{\lambda,N}$ is indeed the functional $ \widehat{R}_{\lambda,N}$ in \eqref{emp-fun} without the noise term.
We call $\widetilde{h}_{\lambda,N}$ the {\bf noise-free structure-preserving kernel estimator}. Furthermore, we have $\|\mathbb{E}_{\bm\varepsilon}[ \widehat{h}_{\lambda,N}]-H\|_{\mathcal{H}_K}=\|\widetilde{h}_{\lambda,N}-H\|_{\mathcal{H}_K}$, where $\mathbb{E}_{\bm\varepsilon}$ is the expectation with respect to the noise vector as in \eqref{expectation with respect to epsilon}. To sum up, we have refined the decomposition \eqref{error first decomposition} by splitting the reconstruction error into three parts, namely, the {approximation error}, the {sampling error}, and the {noisy sampling error}. More specifically, we have
\begin{align} 
\label{error second decomposition}
\widehat{h}_{\lambda,N}-H=\underbrace{\widetilde{h}_{\lambda,N}-h_{\lambda}^*}_{\text{Sampling error}}\quad+ \quad\underbrace{\widehat{h}_{\lambda,N}-\widetilde{h}_{\lambda,N}}_{\text{Noisy sampling error}}\quad+\quad\underbrace{h_{\lambda}^*-H}_{\text{Approximation error}},
\end{align}
where, according to \eqref{min-dec}, the noisy sampling error satisfies that 
\begin{equation}
\label{noisy part of sampling}
\widehat{h}_{\lambda,N}-\widetilde{h}_{\lambda,N}=\frac{1}{\sqrt{N}}(B_N+\lambda)^{-1}A_N^{*}\mathbf{E}_N.
\end{equation}
Following the approach introduced in \cite{feng2021learning}, we shall separately analyze these three errors. The analysis of the approximation error and the noisy sampling error is relatively standard and, for the sake of completeness, it is conducted in Appendix \ref{Analysis of the approximation error and noise part of the sampling error}. We shall then proceed in Section \ref{PAC bounds with fixed Tikhonov regularization parameter} to formulate probably approximately correct (PAC) bounds for the sampling and the reconstruction errors in which, unlike other results in the literature, the regularization constant $\lambda$ remains fixed, and the size of the estimation sample is allowed to vary independently from it. To obtain convergence rates, we shall have to adopt in the following Section \ref{Convergence rates using adaptive Tikhonov regularization parameters} a more conventional approach in which the Tikhonov regularization parameter is adapted as the sample size is modified.

\paragraph{The approximation error.} To bound the approximation error, it is customary to impose restrictions on the target Hamiltonian function. We consider the following hypothesis that is used in \cite{feng2021learning} under the denomination {\bf source condition}. Let $\gamma\in(0,1)$,  $S>0$, and $B=A^*A$ as in \eqref{positive}. We assume that
\begin{align}\label{sou-con}
H\in \Omega_S^\gamma:=\{h\in\mathcal{H}_K\mid h=B^\gamma \psi, \psi\in\mathcal{H}_K, \|\psi\|_{\mathcal{H}_K} < S\}.    
\end{align}

\begin{proposition}\label{source-character}
    If a Hamiltonian function $H$ satisfies the source condition \eqref{sou-con}, then $H\in \mathcal{H}_{\mathrm{null}}^{\bot}$. In other words, $\Omega_S^\gamma\subset \mathcal{H}_{\mathrm{null}}^{\bot}$.
\end{proposition}
\begin{proof} 
Recall first that by Proposition \ref{Wel-Ope}, the operator $B=A^*A$ is a positive compact operator. Let $B=\sum_{n=1}^{L}\lambda_n\langle \cdot, e_n\rangle e_n$ (possibly $L=\infty$) be the spectral decomposition of $B$ with $0<\lambda_{n+1}\leq \lambda_{n}$ and $\{e_n\}_{n=1}^{L}$ be an orthonormal basis of $\mathcal{H}_K$. Hence for any $\gamma\in (0,1)$, we have $B^\gamma=\sum_{n=1}^L\lambda_n^\gamma \langle \cdot,e_n\rangle_{\mathcal{H}_K}e_n$.
Notice that by the representation of the operator $B$ given in \eqref{positive}, for an arbitrary function $h\in\mathcal{H}_{\mathrm{null}}$ (if it is not an empty set), we have that $Bh=0$, which implies that $\lambda_n\langle h,e_n\rangle_{\mathcal{H}_K}=0$ for all $n=1,\cdots,L$. Hence for any $\gamma\in (0,1)$, we have $\lambda_n^\gamma\langle h,e_n\rangle_{\mathcal{H}_K}=0$ for all $n=1,\cdots,L$. Then we obtain that $B^\gamma h=0$. Finally for arbitrary $\psi\in\mathcal{H}_K$, we compute the inner product 
$\langle B^\gamma\psi,h\rangle_{\mathcal{H}_K}=\langle \psi,B^\gamma h\rangle_{\mathcal{H}_K}=0$, which yields that $B^\gamma\psi\in \mathcal{H}_{\mathrm{null}}^{\bot}$ for any $\psi\in \mathcal{H}_K$. Therefore, we conclude that $\Omega_S^\gamma\subset \mathcal{H}_{\mathrm{null}}^{\bot}$.
\end{proof}

\begin{proposition}
Let $B=\sum_{n\in\mathbb{N}_+}\lambda_n\langle \cdot, e_n\rangle e_n$, be a spectral decomposition of $B$, with an orthonormal basis $\{e_n\}_{n=1}^L$. Then, $\Omega_S^\gamma$ has the following characterization 
\begin{align*}
    \Omega_S^\gamma :=\mathcal{H}_{S}^{\gamma} = \left\{h\in\mathcal{H}_{\mathrm{null}}^\bot\mid \sum_{n\in\mathbb{N}^+}\frac{1}{\lambda_n^{2\gamma}}|\langle h,e_n\rangle_{\mathcal{H}_K}|^2<S^2\right\}.
\end{align*}
\end{proposition}
\begin{proof}
First, we show that $\mathcal{H}_{S}^{\gamma}\subset\Omega_{S}^\gamma$. For an arbitrary $h\in \mathcal{H}_{S}^{\gamma}$, we can define a function $\psi\in \operatorname{span}\{e_n, n\in \mathbb{N}_+\}$ such that $\langle\psi,e_n\rangle_{\mathcal{H}_K}=\frac{1}{\lambda_n^{\gamma}}\langle h,e_n\rangle_{\mathcal{H}_K}$ for all $n\in\mathbb{N}^+$. One computes
\begin{align*}
\|\psi\|_{\mathcal{H}_K}^2=\sum_{n\in\mathbb{N}^+}\frac{1}{\lambda_n^{2\gamma}}|\langle h,e_n\rangle_{\mathcal{H}_K}|^2<S^2.
\end{align*}
Furthermore, since $h\in \mathcal{H}_{\mathrm{null}}^{\bot}=\operatorname{span}\{e_n, n\in \mathbb{N}_+\}$, we have that
\begin{align*}
 B^\gamma\psi = \sum_{n\in\mathbb{N}^+} \lambda_n^\gamma \langle \psi,e_n\rangle_{\mathcal{H}_K}e_n = \sum_{n\in\mathbb{N}^+} \lambda_n^\gamma \frac{1}{\lambda_n^{\gamma}}\langle h,e_n\rangle_{\mathcal{H}_K}e_n =h,
\end{align*}
which yields that $h\in\Omega_S^\gamma$. Now, we show that $\Omega_{S}^\gamma\subset\mathcal{H}_{S}^{\gamma}$. For an arbitrary $h\in \Omega_{S}^{\gamma}$, there exists a function $\psi\in\mathcal{H}_K$ with norm $\|\psi\|_{\mathcal{H}_K}<S$, such that $B^\gamma \psi=h$. Let $\psi^{\prime}$ be the projection of $\psi$ onto $\mathcal{H}_{\mathrm{null}}^{\bot}$, then $h =B^{\gamma}\psi^{\prime}$. Hence, we obtain that
\begin{align*}
  \sum_{n\in\mathbb{N}^+} \lambda_n^\gamma \langle \psi^{\prime},e_n\rangle_{\mathcal{H}_K}e_n=\sum_{n\in\mathbb{N}^+}  \langle h,e_n\rangle_{\mathcal{H}_K}e_n,   
\end{align*}
which implies that $\langle\psi^{\prime},e_n\rangle_{\mathcal{H}_K}=\frac{1}{\lambda_n^{\gamma}}\langle h,e_n\rangle_{\mathcal{H}_K}$ for all $n\in\mathbb{N}^+$ since $\lambda_n>0$. Therefore, 
\begin{align*}
    \|\psi^{\prime}\|_{\mathcal{H}_K}^2=\sum_{n\in\mathbb{N}^+}\frac{1}{\lambda_n^{2\gamma}}|\langle h,e_n\rangle_{\mathcal{H}_K}|^2<S^2.
\end{align*}
Combing Proposition \ref{source-character}, we have that $h\in\mathcal{H}_{S}^{\gamma}$. Therefore, we can conclude that $\Omega_{S}^\gamma=\mathcal{H}_{S}^{\gamma}$.
\end{proof}

The source condition \eqref{sou-con} provides a standard way to analyze the approximation error that is spelled out in Appendix \ref{Analysis of the approximation error and noise part of the sampling error}. The conclusion of that study is that in the presence of the source condition \eqref{sou-con}, the approximation error can be bound using the RKHS norm as 
\begin{align}\label{app-err}
\|h_{\lambda}^*-H\|_{\mathcal{H}_K} \leq \lambda^{\gamma}\|B^{-\gamma}H\|_{\mathcal{H}_K},
\end{align}
where $B^{-\gamma}H$ represents the pre-image of $H$, via an operator spectral decomposition. 

An alternative way to handle the approximation error is by using kernels that have the universality property spelled out around \eqref{universal kernel}. This approach is preferable in the presence of compactness hypotheses that are not present in our context, so we shall use the source condition in what follows.

\paragraph{The noisy sampling error.} The norm $\|\widetilde h_{\lambda,N}- \widehat{h}_{\lambda,N}\|_{\mathcal{H}_K} $ can be controlled by performing an analysis similar to the one carried out in \cite{feng2021learning} in which the treatment involving the usual Gram matrix $K(\mathbf{Z}_N,\mathbf{Z}_N)$ has to be extended to accommodate the differential Gram matrix $\nabla_{1,2}K(\mathbf{Z}_N,\mathbf{Z}_N)$. This work is carried out in detail in Appendix \ref{Analysis of the approximation error and noise part of the sampling error}, and we obtain that, for any $\delta>0$, with a probability at least $1-\delta/2$, it holds that
\begin{equation}
\begin{aligned}\label{noi-bou}
\left\|\widetilde h_{\lambda,N}- \widehat{h}_{\lambda,N}\right\|_{\mathcal{H}_K} \leq \frac{\sigma \kappa}{\lambda }\sqrt{\frac{2d}{N}}\left(1+\sqrt{\frac{1}{c}\log(4/\delta)}\right),
\end{aligned} 
\end{equation}
where $c$ is a positive constant appearing in the Hanson-Wright inequality (see Appendix \ref{Han-Wri}).

\subsection{PAC bounds with fixed Tikhonov regularization parameter}
\label{PAC bounds with fixed Tikhonov regularization parameter}

In this section, we provide probably approximately correct (PAC) bounds for the sampling and the total reconstruction errors using as our main tool the $\Gamma$-convergence Theorem (see, for instance, \cite[P81, Corollary 7.20]{dal2012introduction}). 
These theorems show convergence with high probability for a fixed value $\lambda$ of the regularization parameter and when the estimation sample size is sufficiently high. All along this section we assume that $K\in C_b^3(\mathbb{R}^{2d}\times\mathbb{R}^{2d})$ is a Mercer kernel.

\begin{theorem}[{\bf PAC bounds of the sampling error}]
\label{Sam-Con}
Let $h_{\lambda}^*$ and $\widetilde{h}_{\lambda,N}$ be the unique minimizers of the minimization problems \eqref{exp-pro} and \eqref{fre-pro}, respectively. Then for every $\lambda>0$ and any $\varepsilon, \delta>0$, there exits $n\in\mathbb{N}_{+}$ such that for all $N>n$, it holds that 
\begin{align*}
\mathbb{P}\left(\left\|\widetilde{h}_{\lambda,N}-h_{\lambda}^*\right\|_{\mathcal{H}_K}>\varepsilon\right) < \delta.  
\end{align*}
\end{theorem}

Combining the PAC bounds of the sampling error with the analysis of the approximation and the noisy sampling errors in \eqref{app-err} and \eqref{noi-bou}, respectively, we obtain the following PAC bounds of the total reconstruction error.

\begin{theorem}[{\bf PAC bounds of the total reconstruction error}]
\label{Pac-Rec}
Let $ \widehat{h}_{\lambda,N}$ be the unique minimizer of the minimization problem \eqref{emp-pro}.
Suppose that $H\in\Omega_S^{\gamma}$ as defined in \eqref{sou-con}. 
Then, for any $\varepsilon, \delta>0$, there exist $\lambda>0$ and $n\in\mathbb{N}_{+}$ such that for all $N>n$, it holds that 
\begin{align*}
\mathbb{P}\left(\left\| \widehat{h}_{\lambda,N}-H\right\|_{\mathcal{H}_K}>\varepsilon\right) < \delta.  
\end{align*}    
\end{theorem}
\begin{proof}
By Theorem \ref{Sam-Con}, for every $\lambda>0$ and any $\varepsilon, \delta>0$, there exits $n_1\in\mathbb{N}_{+}$ such that for all $N>n_1$, it follows that 
\begin{align}\label{fre-bou}
\mathbb{P}\left(\left\|\widetilde{h}_{\lambda,N}-h_{\lambda}^*\right\|_{\mathcal{H}_K}>\frac{\varepsilon}{3}\right) < \delta.  
\end{align}
Moreover, using the bounds \eqref{app-err} and \eqref{noi-bou} of the  approximation  and the noisy sampling errors, respectively, we can conclude that for any $\varepsilon>0$, there exists $\lambda>0$ and $n_2\in\mathbb{N}^+$, such that for all $N>n_2$, we have
\begin{align}\label{oth-bou}
\|h_{\lambda}^*-H\|_{\mathcal{H}_K}<\frac{\varepsilon}{3}, \quad \|\widetilde{h}_{\lambda,N}- \widehat{h}_{\lambda,N}\|_{\mathcal{H}_K}<\frac{\varepsilon}{3}.
\end{align}
Letting $n=\max\{n_1, n_2\}$ and combining the inequalities \eqref{fre-bou} and \eqref{oth-bou}, the result follows.
\end{proof}

\subsection{Convergence rates using adaptive Tikhonov regularization parameters}
\label{Convergence rates using adaptive Tikhonov regularization parameters}

The PAC bound of the total reconstruction error in Theorems \ref{Sam-Con} and \ref{Pac-Rec} are not informative in relation to convergence rates. 
In order to get a convergence upper rate of $\|{\widehat{h}_{\lambda,N}-H}\|_{\mathcal{H}_K}$ as $N \rightarrow \infty $, in this section we shall work not with a fixed, but with a dynamical $\lambda$ that is adapted with respect to the sample size $N$. More specifically, we shall assume that
\begin{align}\label{dyn-sca}
\lambda\propto N^{-\alpha},\quad \alpha >0,    
\end{align}
where the symbol "$\propto$" means that $\lambda$ has the order $N^{-\alpha}$ as $N\to\infty$. It is easy to see by combining the hypothesis \eqref{dyn-sca} with the bounds \eqref{app-err} and \eqref{noi-bou}, that the approximation error $\|{h_{\lambda}^*-H}\|_{\mathcal{H}_K}$ tends to zero with a convergence rate $N^{-\alpha\gamma}$ and so does the noisy sampling error $\|\widetilde{h}_{\lambda,N}- \widehat{h}_{\lambda,N}\|_{\mathcal{H}_K}$ with a convergence rate $N^{-(\frac{1}{2}-\alpha)}$, if $\alpha\in(0,\frac{1}{2})$.
Thus, to obtain a convergence rate for the total reconstruction error, it suffices to show that the sampling error for a dynamical $\lambda$ converges for some $\alpha\in(0,\frac{1}{2})$. This will be carried out in Theorem \ref{Sam-Err} where it will be shown that the sampling error $\|\widetilde{h}_{\lambda,N}-h_{\lambda}^*\|_{\mathcal{H}_K}$ converges to zero with a convergence upper rate $N^{-\frac{1}{2}(1-3\alpha)}$, with $\alpha\in(0,\frac{1}{3})$.

Recall that by Theorem \ref{Brd-Gau}, the Gaussian process posterior mean and the kernel ridge regression estimators coincide when $\lambda=\frac{\sigma^2}{N}$. This relation is of the type in the hypothesis \eqref{dyn-sca} with $\alpha=1$. Unfortunately, given what we just stated, our convergence framework does not apply to the case $\alpha=1$ and we shall hence have to restrict to the analysis of the convergence of the kernel ridge regression estimator.

As in previous sections we shall assume here that $K\in C_b^3(\mathbb{R}^{2d}\times\mathbb{R}^{2d})$ is a Mercer kernel. We start by obtaining a bound for the sampling error using an adaptive $\lambda$ as in \eqref{dyn-sca}. We start by decomposing 
\begin{align*}
&\widetilde{h}_{\lambda,N}-h_{\lambda}^*=(B_N+\lambda)^{-1}B_N H-(B+\lambda)^{-1}BH \\
=&(B_N+\lambda)^{-1}B_NH-(B_N+\lambda)^{-1}BH+(B_N+\lambda)^{-1}BH- (B+\lambda)^{-1}BH.
\end{align*}
The following lemma requires applying lemma \ref{Hil-Bou} and is inspired by a similar result in \cite{feng2021learning}. The main difference between our result and \cite[Lemma 22]{feng2021learning} is that different choices of norms are involved, namely, $\|H\|_{\infty}$ and what is denoted in that paper as  $\|H\|_{L^2(\tilde{\rho}_T^L)}$. Moreover, the authors assumed the so-called coercivity condition, which is a restriction on the choice of kernel. In our setting, we merely involve the RKHS norm $\|H\|_{\mathcal{H}_K}$, and hence shall not require that assumption.

We now present the following sampling error bound. 

\begin{theorem}{\bf (Sampling Error Bounds)}\label{Sam-Err} For a function $H \in \mathcal{H}_K$ and $0< \delta <1$, with probability at least $1-\delta$, it holds that
$$\left\|\widetilde{h}_{\lambda,N}-h_{\lambda}^*\right\|_{\mathcal{H}_K} \leq \left(\sqrt{ \frac{8\log(4/\delta)}{N}}+ 
1\right)\sqrt{ \frac{2\log(4/\delta)}{N\lambda^2}}\|H\|_{\mathcal{H}_K}2d\kappa^2\left(1+\kappa\sqrt{\frac{2d}{\lambda}}\right).$$
Moreover, if $\lambda$ satisfies \eqref{dyn-sca} with $\alpha\in(0,\frac{1}{3})$, then the  sampling error $\|\widetilde{h}_{\lambda,N}-h_{\lambda}^*\|_{\mathcal{H}_K}$ converges to zero with a convergence upper rate $N^{-\frac{1}{2}(1-3\alpha)}$.
\end{theorem}

\paragraph{Improved sampling error bounds using the coercivity condition.} The error analysis in \cite{feng2021learning,lu2019nonparametric} is conducted under the so-called {\bf coercivity condition}. Even though the theorem that we just proved does not require this coercivity condition, in the next result, we shall see that by assuming a similar condition, we can improve our sampling error convergence result for a wider range of $\alpha$. 

\begin{definition}[{\bf Coercivity condition}]\label{Coe-Con}
We say that the Hamiltonian system \eqref{ham-sys} satisfies the \textbf{coercivity condition} on $\mathcal{H}_K$,  if  for all $h \in \mathcal{H}_K$, there exists $ c_{\mathcal{H}_K}>0$ such that
\begin{align}
\label{coercivity}
\|A h\|^2_{L^2(\mu_{\mathbf{Z}})}=\|X_h\|^2_{L^2(\mu_{\mathbf{Z}})}\geq c_{\mathcal{H}_K}\|h\|^2_{\mathcal{H}_K}. \end{align}
We choose the supremum $c_{\mathcal{H}_K}$ that satisfies \eqref{coercivity} and refer to it as the \textit{coercivity constant}. 
\end{definition}

The coercivity condition that we just defined is slightly different from the one in \cite{feng2021learning,lu2019nonparametric} designed for the analysis of particle swarming models. On the right side of \eqref{coercivity}, they have the norm that they denote as $L^2(\tilde{\rho}_{\mathbf{X}})$  rather than the RKHS norm. 
There are two reasons for this difference with respect to the results in \cite{feng2021learning,lu2019nonparametric}, which require the coercivity condition.
On the one hand, we added a Tikhonov regularization term in the loss functions so that the operators $B+\lambda$ and $B_N+\lambda$ are always invertible and hence the existence of the solution of the inverse problems \eqref{emp-pro}-\eqref{emp-fun} and \eqref{exp-fun}-\eqref{exp-pro} is guaranteed by Proposition \ref{Rep-Ope}. On the other hand, as shown in Lemma \ref{Dec-Omp}, the bound is obtained directly with the RKHS norm $\|\cdot\|_{\mathcal{H}_K}$ rather than with $\|\cdot\|_{\infty}$ or the $L^2$ norm $\|\cdot\|_{L^2(\mu_{\mathbf{Z}})}$ and there is no need to transform the infinity norm and the $L^2$ norm to the RKHS norm using the coercivity condition.

We now show that using the coercivity condition \eqref{coercivity}, we can enlarge the available scaling indices from $\alpha\in(0,\frac{1}{3})$ to $\alpha\in(0,\frac{1}{2})$ by using a more precise analysis. That is the content of the following corollary.

\begin{corollary}[{\bf Sampling error bounds under the coercivity condition}]\label{Sam-Err2}
Suppose that the coercivity condition \eqref{coercivity} holds. Then, for any function $H \in \mathcal{H}_K$ and $0< \delta <1$, with probability at least $1-\delta$, it holds that
$$\left\|\widetilde{h}_{\lambda,N}-h_{\lambda}^*\right\|_{\mathcal{H}_K} \leq \left(\sqrt{ \frac{8\log(4/\delta)}{N}}+ 
1\right)\sqrt{ \frac{2\log(4/\delta)}{N\lambda^2}}2d\kappa^2\left(2+\kappa\sqrt{\frac{2d}{c_{\mathcal{H}_K}}}\right)\|H\|_{\mathcal{H}_K}.$$
Moreover, if $\lambda$ has the scaling \eqref{dyn-sca}, then the  sampling error $\|\widetilde{h}_{\lambda,N}-h_{\lambda}^*\|_{\mathcal{H}_K}$ converges to zero with convergence upper rate $N^{-\frac{1}{2}(1-2\alpha)}$ with $\alpha\in(0,\frac{1}{2})$.  
\end{corollary}

Having constructed sampling error bounds in Theorem \ref{Sam-Err} and Corollary \ref{Sam-Err2}, we wrap everything up and present the following convergence upper rate of the total reconstruction error.

\begin{theorem}[{\bf Convergence upper rate of the total reconstruction error}]\label{Tot-Err} 
Let $ \widehat{h}_{\lambda,N}$ be the unique minimizer of the minimization problem \eqref{emp-pro}. Suppose that $H$ satisfies the source condition \eqref{sou-con}, that is, $H\in\Omega_S^{\gamma}$. Then for all $\alpha\in(0,\frac{1}{3})$, and for any $0<\delta<1$, with probability as least $1-\delta$, it holds that
\begin{align*}
\left\| \widehat{h}_{\lambda,N}-H\right\|_{\mathcal{H}_K} \leq C(\gamma,\delta,\kappa) ~ N^{-\min\{\alpha\gamma, \frac{1}{2}(1-3\alpha)\}},   
\end{align*}  
where 
$$C(\gamma,\delta,\kappa)=\max\left\{\|B^{-\gamma}H\|_{\mathcal{H}_K}, 8\sqrt{ 4\log(8/\delta)}d^{\frac{3}{2}}\kappa^3\|H\|_{\mathcal{H}_K}\right\}.$$ 

Moreover, if the coercivity condition \eqref{coercivity} holds, then for all $\alpha\in(0,\frac{1}{2})$, and for any $0<\delta<1$, and with probability as least $1-\delta$, it holds that
\begin{align*}
\| \widehat{h}_{\lambda,N}-H\|_{\mathcal{H}_K} \leq C(\gamma,\delta,\sigma,\kappa,c,c_{\mathcal{H}_K}) ~ N^{-\min\{\alpha\gamma, \frac{1}{2}(1-2\alpha)\}},   
\end{align*}
where 
\begin{multline*}
C(\gamma,\delta,\sigma,\kappa,c,c_{\mathcal{H}_K})\\
=\max\left\{\|B^{-\gamma}H\|_{\mathcal{H}_K}, \sigma \kappa\sqrt{2d} \left(1+\sqrt{\frac{1}{c}\log(4/\delta)}\right), 4\sqrt{ 2\log(8/\delta)}d\kappa^2\left(2+\kappa\sqrt{\frac{2d}{c_{\mathcal{H}_K}}}\right)\|H\|_{\mathcal{H}_K}\right\}.
\end{multline*}
\end{theorem}
\begin{proof}
It is a direct combination of the inequalities \eqref{app-err}, \eqref{noi-bou}, Theorem \ref{Sam-Err}, and Corollary \ref{Sam-Err2}.  
\end{proof}

\begin{remark} 
\normalfont
Note that by the Differential Reproducing Property Theorem \ref{Par-Rep}, we immediately obtain that for all kernels $K\in C_b^{2s+1}(\mathbb{R}^{2d}\times \mathbb{R}^{2d})$ with $s\geq1$,
\begin{align*}
\| \widehat{h}_{\lambda,N}-H\|_{C_b^s} \leq \kappa C(\gamma,\delta,\kappa) ~ N^{-\min\{\alpha\gamma, \frac{1}{2}(1-3\alpha)\}}\quad \text{or}\quad \kappa C(\gamma,\delta,\sigma,\kappa,c,c_{\mathcal{H}_K}) ~ N^{-\min\{\alpha\gamma, \frac{1}{2}(1-2\alpha)\}}.   
\end{align*}  
\end{remark}

Theorem \ref{Tot-Err} guarantees that the structure-preserving kernel estimator provides a function that is close to the data-generating Hamiltonian function with respect to the RKHS norm. As a consequence, we now prove that the flow of the learned Hamiltonian system will uniformly approximate that of the underlying Hamiltonian, which justifies the use of the RKHS norm.

\begin{proposition}[{\bf From discrete data to continuous-time flows}]\label{Dis-Con} 
Let $\widehat{H}= \widehat{h}_{\lambda,N}$ be the structure-preserving kernel estimator of $H$ using a kernel $K\in C_b^5(\mathbb{R}^{2d}\times\mathbb{R}^{2d})$. Let $F:[0,T] \times  \mathbb{R} ^{2d} \longrightarrow \mathbb{R} ^{2d} $ and $ \widehat{F}:[0,T] \times  \mathbb{R} ^{2d} \longrightarrow \mathbb{R} ^{2d} $ be the flows over the time interval $[0,T]$ of the Hamilton equations associated to the Hamiltonian functions $H$ and $\widehat{H}$, respectively. Then, for any initial condition ${\bf z}\in \mathbb{R}^{2d} $, we have that 
\begin{align*}
\|F({\bf z})-\widehat{F}({\bf z})\|_{\infty}:=\max_{t\in[0,T]}\left\|F_t({\bf z})-\widehat{F}_t({\bf z})\right\|\leq C \|H-\widehat{H}\|_{\mathcal{H}_K}, 
\end{align*}
with the constant $C:=\sqrt{2d}\kappa T\exp\{2d\kappa\|H\|_{\mathcal{H}_K}
T\}$. 
\end{proposition}

\section{Numerical experiments}\label{Numerical experiments}

In this subsection, we apply our structure-preserving ridge regularized kernel estimator to learn Hamiltonian functions of various dynamical systems, where the dimension of the configuration space is $d=2$, that is, we shall be learning functions $H:\mathbb{R}^4\rightarrow \mathbb{R}$. Moreover, all our examples, except for the ones in Sections \ref{5.1.1} and \ref{numerical_convergence} are {\it simple} mechanical systems in the sense that the Hamiltonian function can be written as the sum of the kinetic energy plus a potential that depends only on the configuration variables, that is, 
\begin{equation}
\label{simple mechanical}
H(q_1,q_2,p_1,p_2)=T(p _1, p _2)+V(q _1, q _2),
\end{equation}
where $(q_1,q_2) $ are the position variables in the configuration space and $(p_1,p_2)$ are the conjugate momenta. An advantage of these systems is that the potential $V$ can be well visualized as a 3D plot, which can be used for the sake of comparison. In Section \ref{5.1}, we test our algorithm on some common Hamiltonian systems that are used as examples in the literature. In Section \ref{5.2}, we consider the more challenging task of learning a Hamiltonian system with a highly non-convex potential well. In Section \ref{numerical_convergence}, we perform a numerical investigation about the convergence rate derived in Theorem \ref{Tot-Err}. In Section \ref{5.3}, we experiment on the effectiveness of our algorithm in case the potential well exhibits singularities. Lastly, in Section \ref{5.4}, we compare the performance and training cost between our algorithm and that of the HNN approach, where the Hamiltonian function is modeled as a neural network and trained with gradient descent.

In the kernel ridge regressions that we will conduct to learn various Hamiltonian functions, we consistently use a Gaussian kernel. As explained in the introduction, our dataset contains the sampling points in the phase space $\mathbf{Z}_N$, and the corresponding noisy versions of the Hamiltonian vector fields $\mathbf{X}_{\sigma^2,N}$ at those sampling points. To generate these data, we randomly draw $N$ phase space points (with $N$ varying for each example) to construct $\mathbf{Z}_N$.  We subsequently evaluate the Hamiltonian data generating vector fields at $\mathbf{Z}_N$ to construct $\mathbf{X}_{\sigma^2,N}$. In the first examples, we set the noise to zero ($\sigma=0 $), and then in Section \ref{5.2}, we shall illustrate how the performance evolves when $\sigma$ varies. During the training phase, we perform a grid search combined with a 5-fold cross-validation scheme to determine the optimal parameter $\eta$ in the Gaussian kernel \eqref{gaussian kernel eta} and the constant coefficient $c$ in the adaptive relation \eqref{dyn-sca}, that is, $\lambda = c\cdot N^{-\alpha}$, where $\alpha=0.4$ is fixed. We will be searching for the optimal $c$ in the same grid of 
\begin{equation}\label{c grid}
(5e^{-6},1e^{-5},5e^{-5},1e^{-4},5e^{-4},1e^{-3},5e^{-3},1e^{-2},5e^{-2},1e^{-1},5e^{-1},1). 
\end{equation}
in all of the numerical examples, while the grid for searching $\eta$ could be different depending on the specific example. 


Finally, to assess the learning performance, we shall plot, for each example below, the potential $V(q_1,q_2)$ of the ground truth Hamiltonian function and that of the reconstructed Hamiltonian function by simply setting $p_1=p_2=0$ (this is what we shall call ``potential of the learned Hamiltonian" as well as for the system in Section \ref{5.1.1}). We stress that since the observed data are Hamiltonian vector fields, the Hamiltonian function can be reconstructed up to a scalar constant at best. Hence, we vertically shift the surface of the potential well of the reconstructed Hamiltonian towards the ground truth Hamiltonian, with the shifted distance equal to the average of the distances on the $(q_1,q_2)$ grid. The differences are then visualized with heatmaps.

\subsection{Some common Hamiltonian systems}\label{5.1}

\subsubsection{Double pendulum}\label{5.1.1}

The Hamiltonian of the double mathematical pendulum (two point masses of mass $m$ concatenated by two ideal massless strings of length $l$ and moving in a plane under the influence of gravity) using polar coordinates is
\begin{align*}
H(q_1,q_2,p_1,p_2) = \frac{1}{2ml^2}\cdot \frac{p_1^2+2p_2^2-2p_1p_2\cos(q_1-q_2)}{1+\sin^2(q_1-q_2)}+mgl\left[4-2\cos(q_1)-\cos(q_2)\right].
\end{align*}
Since the variables $(q_1,q_2) $ are angles, it is only a local version of the theorems in the paper that apply to this case. For the numerical experiment, we set $N=200$.  We sample $N$ states $(q_1,q_2,p_1,p_2)$ over a uniform distribution on $[-3,3]^4\subset \mathbb{R}^4$, and obtain the corresponding Hamiltonian vector fields. We then perform a grid search of parameters over $\eta$ in {\it numpy.arange(0.5, 4, 0.5)} and $c$ as in \eqref{c grid}. The optimal parameters are $\eta=1.5$ and $c=1e^{-5}$. We plot the potential function of the ground truth (Figure \ref{double_pendulum} (a)) and the reconstructed Hamiltonian (Figure \ref{double_pendulum} (b)) on the $(q_1,q_2)$ plane restricted to $[-3,3]^2$, with the optimal parameters $\eta$ and $c$. We also visualize the error in a heatmap (Figure \ref{double_pendulum} (c)) as elaborated in the introduction of Section \ref{Numerical experiments}.

\begin{figure}[htp]
    \centering
    \subfigure[]{\includegraphics[width=0.32\textwidth]{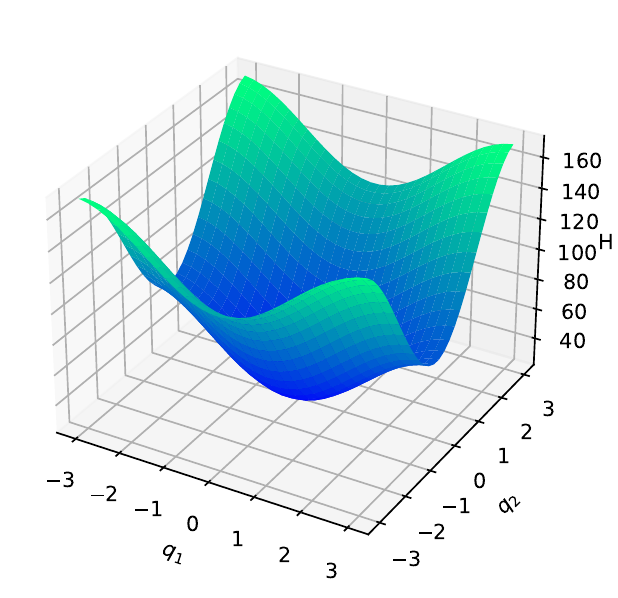}} 
    \subfigure[]{\includegraphics[width=0.32\textwidth]{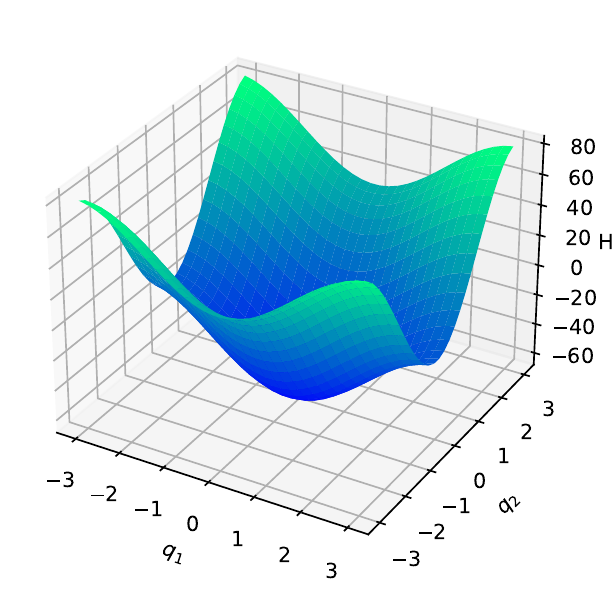}} 
    \subfigure[]{\includegraphics[width=0.32\textwidth]{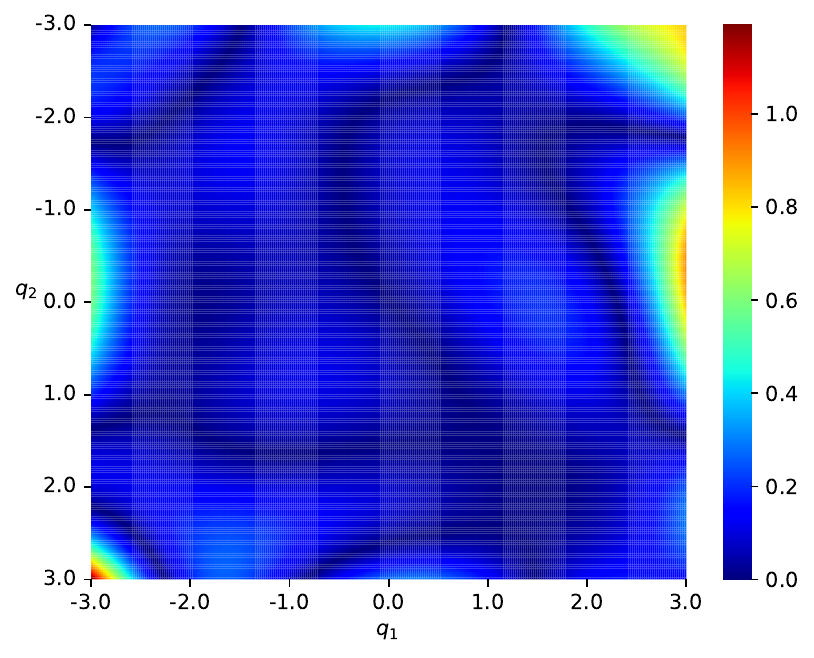}}
    \caption{Double pendulum: (a) Ground truth potential (b) Potential of the learned Hamiltonian (c) Mismatch error after vertical shift}
    \label{double_pendulum}
\end{figure}

\subsubsection{H\'{e}non-Heiles systems}
The H\'{e}non-Heiles system is a simplified model for the planar motion of a star around a galactic center restricted. The dynamical system has a governing Hamiltonian of the simple type in \eqref{simple mechanical}, namely,
\begin{align*}
    H(q_1,q_2,p_1,p_2) = \frac{p_1^2+p_2^2}{2}+\left(\frac{q_1^2+q_2^2}{2}+q_1^2q_2+\frac{q_2^3}{3}\right).
\end{align*}

For the numerical experiment, we adopt $N=100$. We sample $N$ initial conditions $(q_1,q_2,p_1,p_2)$ over a uniform distribution on $[-1,1]^4\subset \mathbb{R}^4$ and obtain the corresponding Hamiltonian vector fields. We then perform a grid search of parameters over $\eta$ in {\it numpy.arange(0.5, 4, 0.5)} and $c$ as in \eqref{c grid}. The optimal parameters are $\eta=3.5$ and $c=5e^{-6}$. We plot the potential function of the groundtruth (Figure \ref{Heiles} (a)) and the reconstructed Hamiltonians (Figure \ref{Heiles} (b)) on the $(q_1,q_2)$ plane restricted to $[-1,1]^2$, with the optimal parameters $\eta$ and $c$. We also visualize the error in a heatmap (Figure \ref{Heiles} (c)) as explained in the introduction of Section \ref{Numerical experiments}.

\begin{figure}[htp]
    \centering
    \subfigure[]{\includegraphics[width=0.32\textwidth]{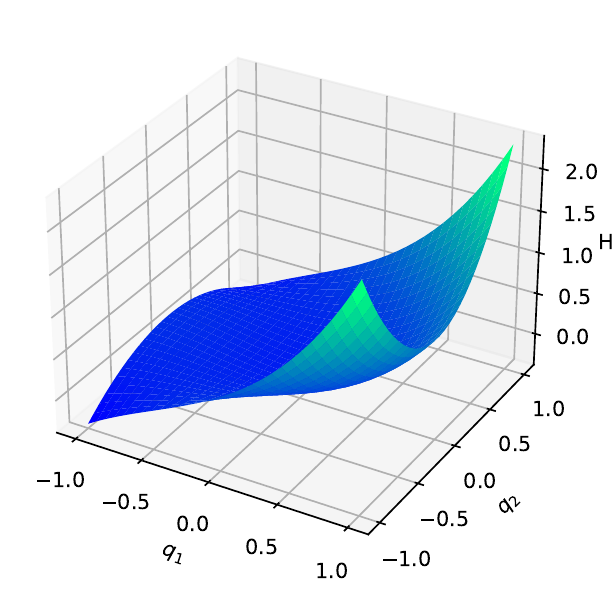}} 
    \subfigure[]{\includegraphics[width=0.32\textwidth]{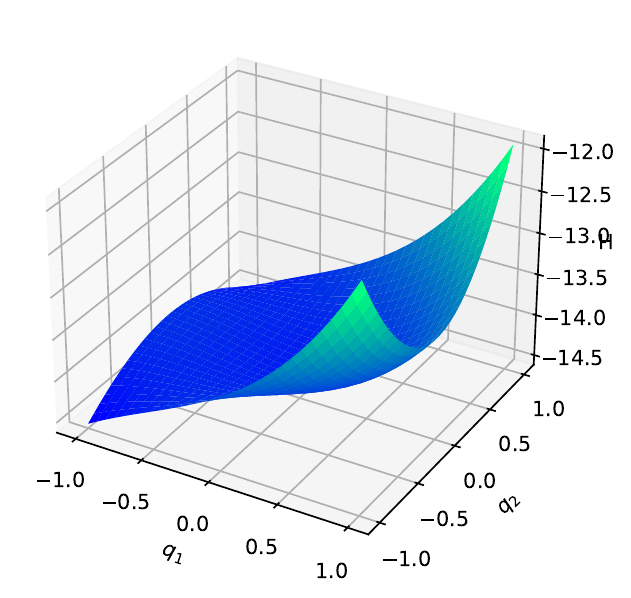}} 
    \subfigure[]{\includegraphics[width=0.32\textwidth]{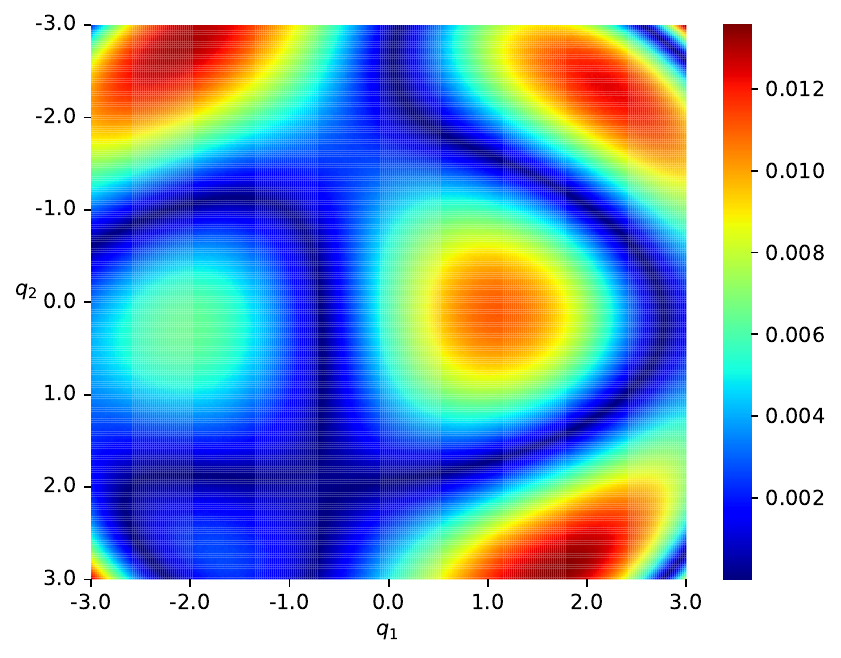}}
    \caption{H\'{e}non-Heiles system (a) Ground truth potential (b) Potential of the learned Hamiltonian (c) Mismatch error after vertical shift}
    \label{Heiles}
\end{figure}

\subsubsection{Frenkel-Kontorova model}
The Frenkel-Kontorova model describes the dynamics of a chain of particles with nearest-neighbor interaction subject to a periodic potential. The dynamical system has a governing Hamiltonian 
\begin{align*}
    H(q_1,q_2,p_1,p_2) = \frac{p_1^2+p_2^2}{2}-\cos(q_1)-\cos(q_2)+\frac{1}{2}g(q_2-q_1)^2.
\end{align*}

For the numerical experiment, we adopt $N=100$.  We sample $N$ initial conditions $(q_1,q_2,p_1,p_2)$ over a uniform distribution on $[-1,1]^4\subset \mathbb{R}^4$ and obtain the corresponding Hamiltonian vector fields. We then perform a grid search of parameters over $\eta$ in np.arange(0.5, 4, 0.5) and $c$ as in  \eqref{c grid}. The optimal parameters are $\eta=2.5$ and $c=5e^{-6}$. We plot the potential function of the ground truth (Figure \ref{Frenkel-Kontorova} (a)) and the reconstructed Hamiltonian (Figure \ref{Frenkel-Kontorova} (b)) on the $(q_1,q_2)$ plane restricted to $[-1,1]^2$, with the optimal parameters $\eta$ and $c$. We also visualize the error in a heatmap (Figure \ref{Frenkel-Kontorova} (c)) as elaborated in the introduction of Section \ref{Numerical experiments}.

\begin{figure}[htp]
    \centering
    \subfigure[]{\includegraphics[width=0.32\textwidth]{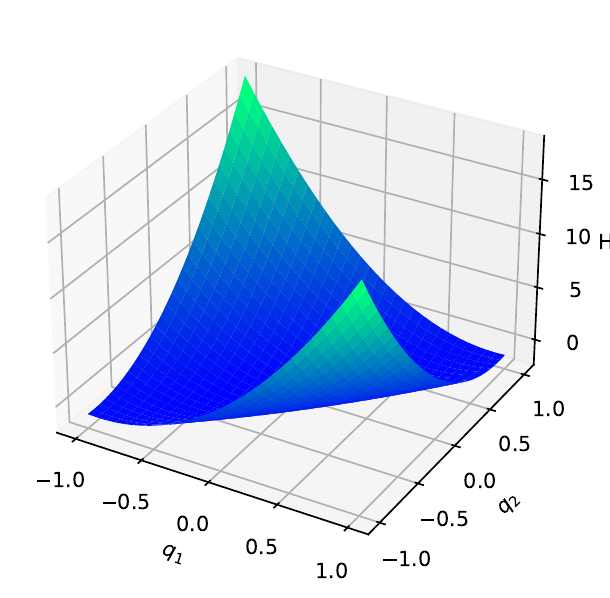}} 
    \subfigure[]{\includegraphics[width=0.32\textwidth]{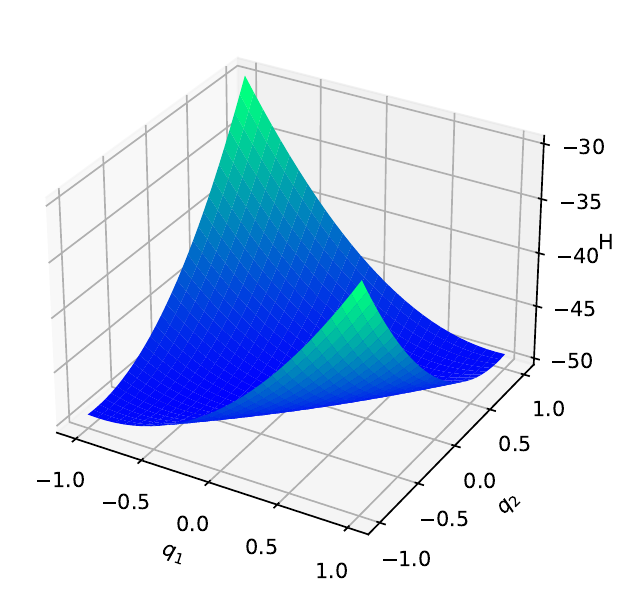}} 
    \subfigure[]{\includegraphics[width=0.32\textwidth]{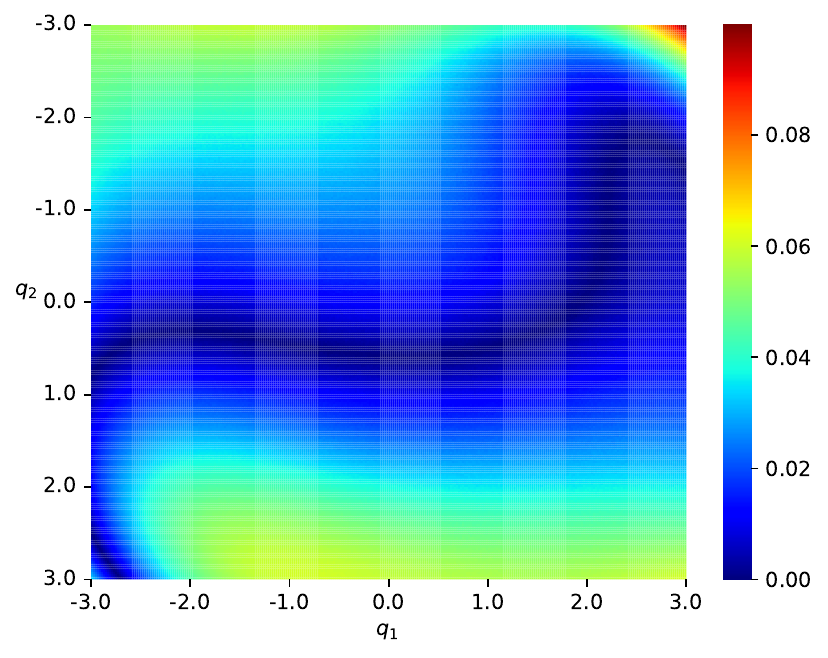}}
    \caption{Frenkel-Kontorova  model (a) Ground truth potential (b) Potential of the learned Hamiltonian (c) Mismatch error after vertical shift}
    \label{Frenkel-Kontorova}
\end{figure}

\subsection{Highly non-convex potential function}
\label{5.2}

It is generally a challenging task to learn a Hamiltonian function that has a highly non-convex potential function. In this subsection, we demonstrate the effectiveness of our approach in even such tasks. We showcase our algorithm by learning the following Hamiltonian function
\begin{align*}
    H(q_1,q_2,p_1,p_2) = \frac{1}{2}(p_1^2+p_2^2)+\sin\left(\frac{2\pi}{3}\cdot q_1\right)\cos\left(\frac{2\pi}{3}\cdot q_2\right)+\frac{\sin(\sqrt{q_1^2+q_2^2})}{\sqrt{q_1^2+q_2^2}},
\end{align*}
whose potential function is visualized below in Figure \ref{ground truth_nonconvex}. To illustrate how the algorithm's performance evolves with the sample size $N$ and the noise level determined by $\sigma$, we run our algorithm in different experimental settings. 

\begin{figure}[htp]
    \centering
{\includegraphics[width=0.45\textwidth]{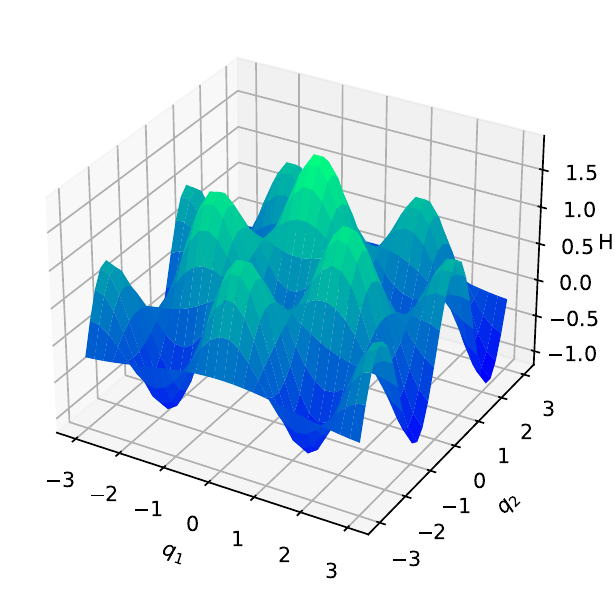}}
    \caption{Ground truth potential}
    \label{ground truth_nonconvex}
\end{figure}

First, we adopt $N=500$ and $\sigma=0$. We sample $N$ initial conditions $(q_1,q_2,p_1,p_2)$ over a uniform distribution on $[-3,3]^4\subset \mathbb{R}^4$ and obtain the corresponding Hamiltonian vector fields. We then perform a grid search of parameters over $\eta$ in np.arange(0.2, 2, 0.2) and $c$ as in \eqref{c grid}. The optimal parameters are $\eta=1.2$ and $c=5e^{-6}$. We plot the potential function of the reconstructed Hamiltonian (Figure \ref{nonconvex_1} (a)) on the $(q_1,q_2)$ plane restricted to $[-3,3]^2$, with the optimal parameters $\eta$ and $c$. We also visualize the error in a heatmap (Figure \ref{nonconvex_1} (b)) as elaborated in the introduction of Section \ref{Numerical experiments}.

\begin{figure}[htp]
    \centering
    \subfigure[]{\includegraphics[width=0.45\textwidth]{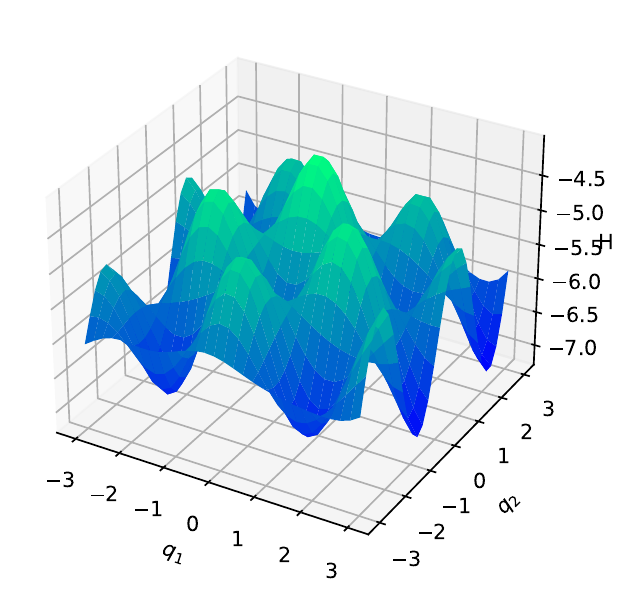}} 
    \subfigure[]{\includegraphics[width=0.45\textwidth]{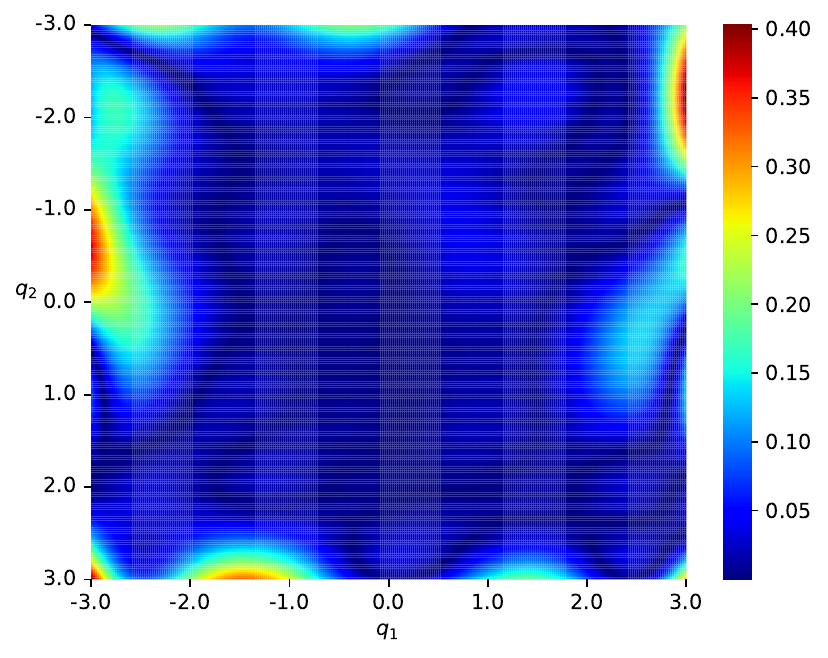}} 
    \caption{Learning with $N=500$ (a) Potential of the learned Hamiltonian (b) Mismatch error after vertical shift}
    \label{nonconvex_1}
\end{figure}

Second, we adopt $N=1500$ and $\sigma=0$. We sample $N$ initial conditions $(q_1,q_2,p_1,p_2)$ over a uniform distribution on $[-3,3]^4\subset \mathbb{R}^4$ and obtain the corresponding Hamiltonian vector fields. We then perform a grid search of parameters over $\eta$ in np.arange(0.2, 2, 0.2) and $c$ as in \eqref{c grid}. The optimal parameters are $\eta=1.2$ and $c=1e^{-5}$. We plot the potential function of the reconstructed Hamiltonian (Figure \ref{nonconvex3} (a)) on the $(q_1,q_2)$ plane restricted to $[-3,3]^2$, with the optimal parameters $\rho$ and $c$. We also visualize the error in a heatmap (Figure \ref{nonconvex3} (b)) as elaborated in the introduction of Section \ref{Numerical experiments}.

\begin{figure}[htp]
    \centering
    \subfigure[]{\includegraphics[width=0.45\textwidth]{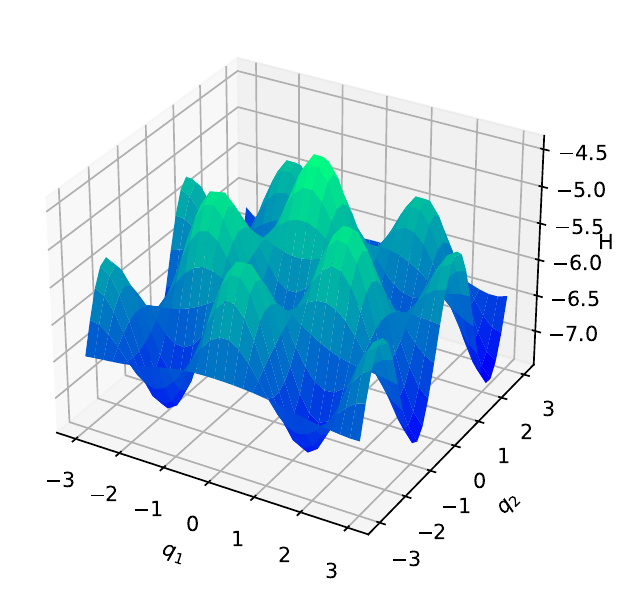}} 
    \subfigure[]{\includegraphics[width=0.45\textwidth]{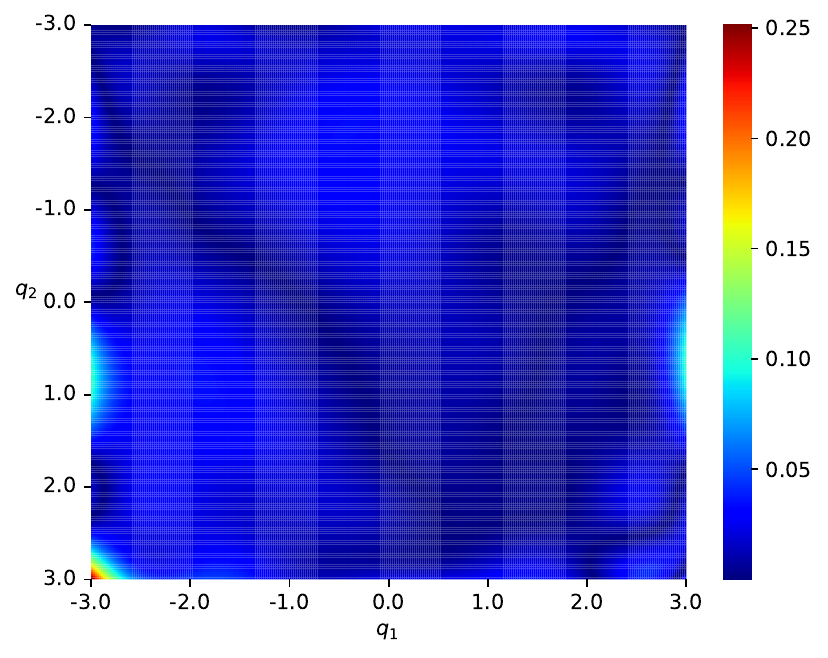}} 
    \caption{Learning with $N=1500$ (a) Potential of the learned Hamiltonian (b) Mismatch error after vertical shift}
    \label{nonconvex3}
\end{figure}

Third, we again adopt $N=500$, but this time, we repeat the experiment for various observation noise levels of the Hamiltonian vector fields, namely, for $\sigma=0.1,0.2,0.3,0.4,0.5$. We visualize the impact of noise on the quality of the learning in Figure \ref{noise_experiment}.

\begin{figure}[htp]
    \centering
    \subfigure[]{\includegraphics[width=0.32\textwidth]{Error_Nonconvex.pdf}} 
    \subfigure[]{\includegraphics[width=0.32\textwidth]{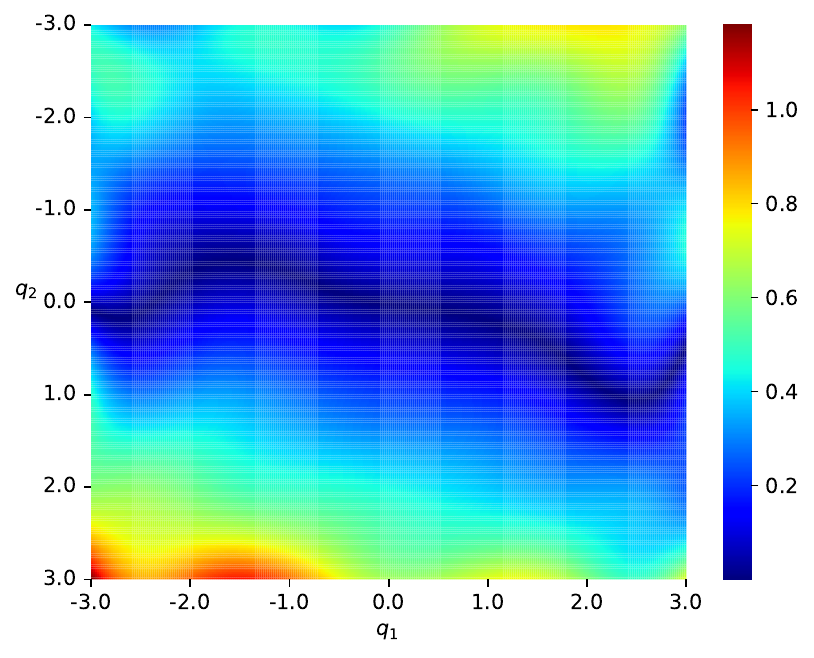}} 
    \subfigure[]{\includegraphics[width=0.32\textwidth]{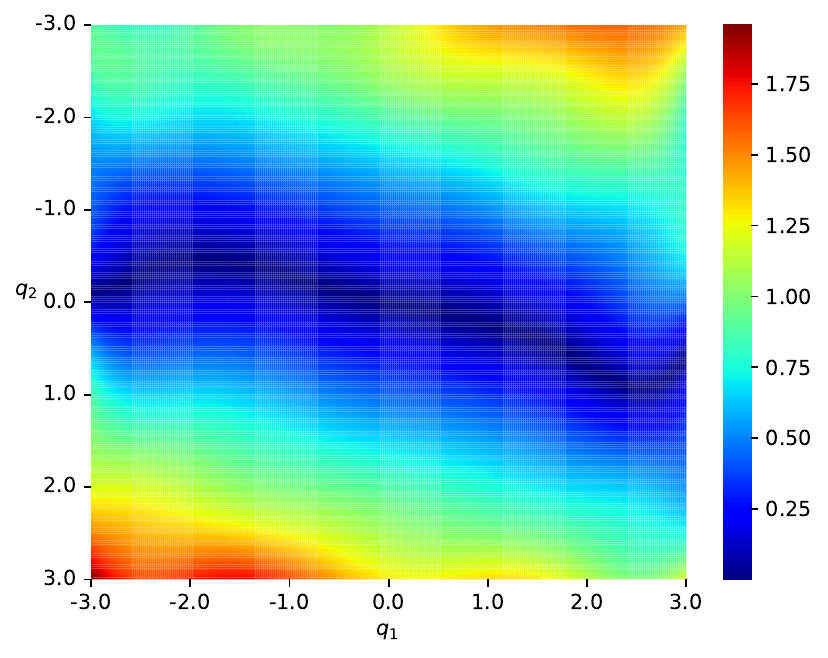}}
    \medskip
    \subfigure[]{\includegraphics[width=0.32\textwidth]{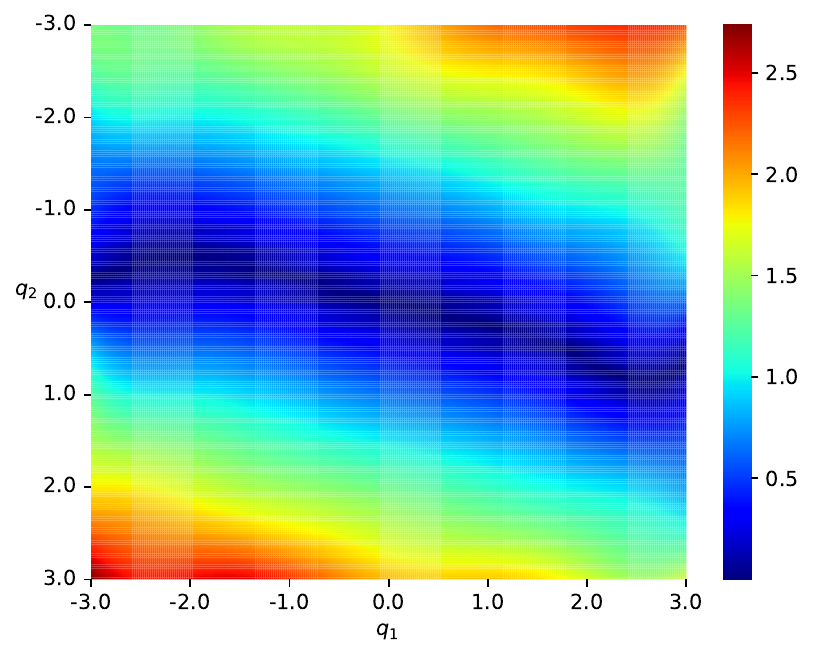}} 
    \subfigure[]{\includegraphics[width=0.32\textwidth]{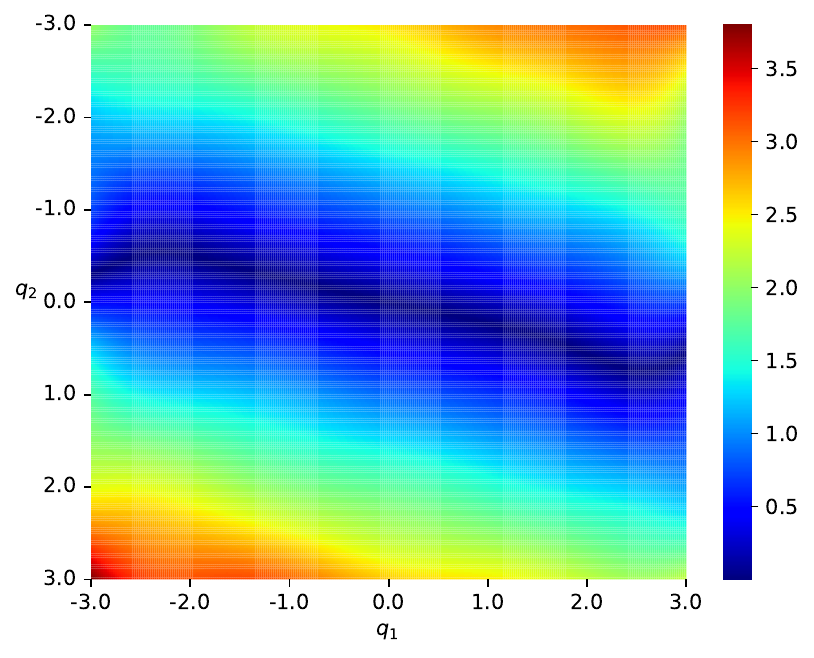}} 
    \subfigure[]{\includegraphics[width=0.32\textwidth]{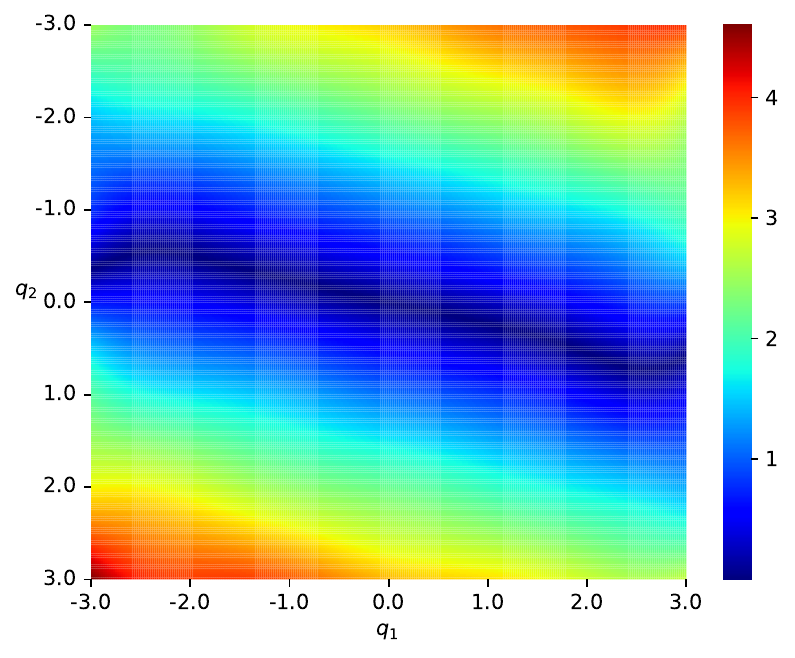}}
    \caption{Learning with $N=500$; Mismatch error after vertical shift with various noise levels corresponding to (a) $\sigma=0$ (b) $\sigma=0.1$ (c) $\sigma=0.2$ (d) $\sigma=0.3$ (e) $\sigma=0.4$ (f) $\sigma=0.5$.}
    \label{noise_experiment}
\end{figure}

\subsection{Convergence analysis}\label{numerical_convergence}

In this subsection, we perform a numerical examination of the convergence rates derived in Section \ref{Convergence analysis and error bounds}. We shall consider the Hamiltonian function
\begin{align*}
    H(q_1,q_2,p_1,p_2) = q^2_1q^3_2e^{-\frac{\|(q_1,q_2,p_1,p_2)\|^2_2}{\eta^2}},
\end{align*}
with $\eta = 2$. This Hamiltonian function, by Example \ref{Gau-Ker}, belongs to the RKHS corresponding to the Gaussian kernel $K_{\eta}$ with parameter $\eta$, with $d = 4$, $k=5$ and $\alpha = (2,3,0,0)$. 

We compute the reconstruction error term $\left\| \widehat{h}_{\lambda,N}-H\right\|_{\mathcal{H}_K}$ by splitting its square into three terms and by calculating them separately, namely,
\begin{align*}
    \left\| \widehat{h}_{\lambda,N}-H\right\|^2_{\mathcal{H}_K} = \left\| \widehat{h}_{\lambda,N}\right\|^2_{\mathcal{H}_K}+\left\|H\right\|^2_{\mathcal{H}_K}-2\langle\widehat{h}_{\lambda,N},H\rangle_{\mathcal{H}_K}.
\end{align*} Note that $\left\| \widehat{h}_{\lambda,N}\right\|^2_{\mathcal{H}_K} = \widehat {\bf c}^{\top}\nabla_{1,2}K({\bf Z}_N,{\bf Z}_N)\widehat {\bf c}$ by \eqref{rep-ker}, and that $\left\| H\right\|^2_{\mathcal{H}_K} = \frac{k!}{(2/\eta^2)^k C^{k}_{\alpha}}=384$. Additionally, by the differential reproducing property \eqref{dif-rep}, 
\begin{align*}
\langle\widehat{h}_{\lambda,N},H\rangle_{\mathcal{H}_K} = \langle \widehat {\bf c}^{\top}\nabla_1K({\bf Z}_N,\cdot), H\rangle_{\mathcal{H}_K} = \widehat {\bf c}^{\top} \nabla H({\bf Z}_N),
\end{align*}
and hence, can also be explicitly computed.

For the numerical experiment, we fix $\eta=2$ to ensure that the estimator lives in the same RKHS as the groundtruth Hamiltonian $H$. We also fix $c = 5e^{-6}$, $\alpha=0.4$, and sample $N$ states over a uniform distribution on $[-1,1]^4$. Since the training data is drawn randomly from the sampling measure, we decide to perform, for each sample size $N$, 50 independent experiments to compute the mean and standard deviation of the RKHS-norm error $\left\| \widehat{h}_{\lambda,N}-H\right\|_{\mathcal{H}_K}$, and visualize them against the sample size $N$ in the Figure \ref{rkhs-norm-convergence}.  The plot indicates that the RKHS-norm error exhibits linear convergence in the log-scale, consistent with the functional form of our theoretical upper bound. 

\begin{figure}[htp]
    \centering
    \subfigure[]{\includegraphics[width=0.48\textwidth]{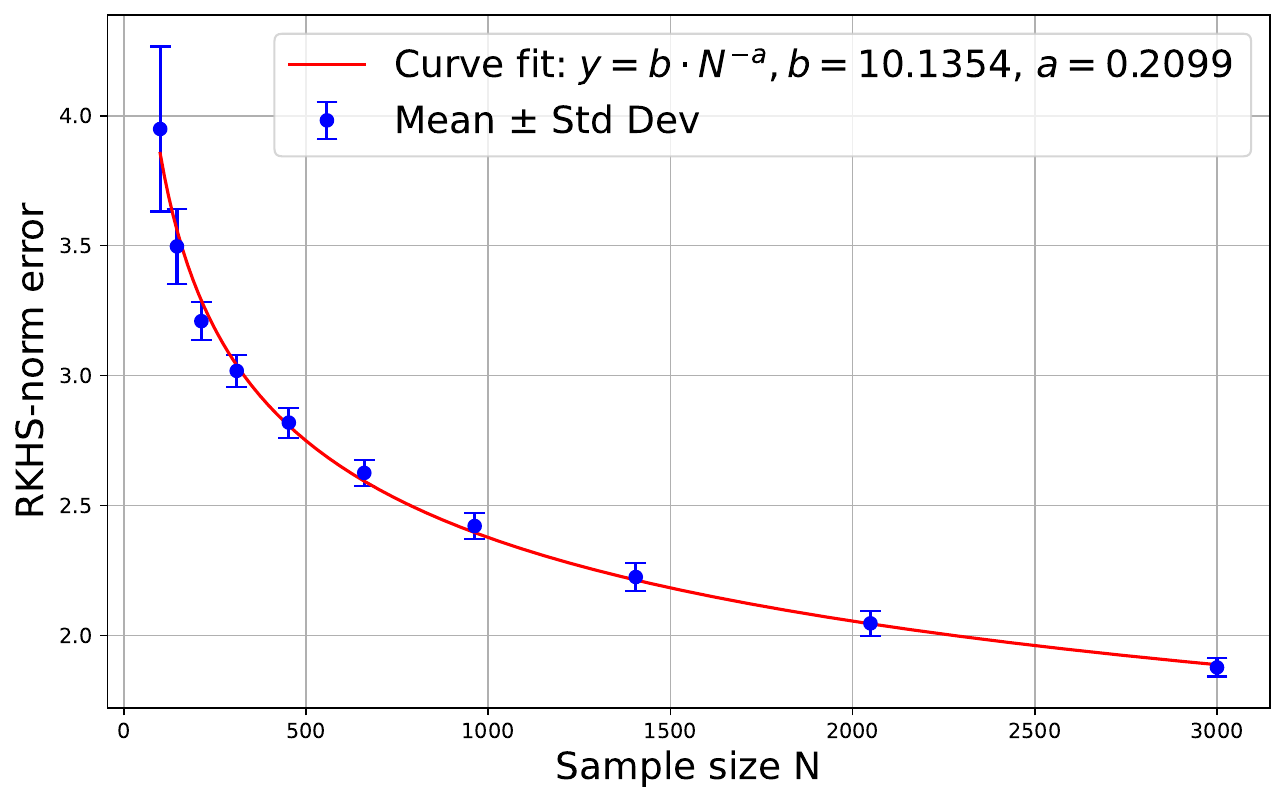}} 
    \subfigure[]{\includegraphics[width=0.5\textwidth]{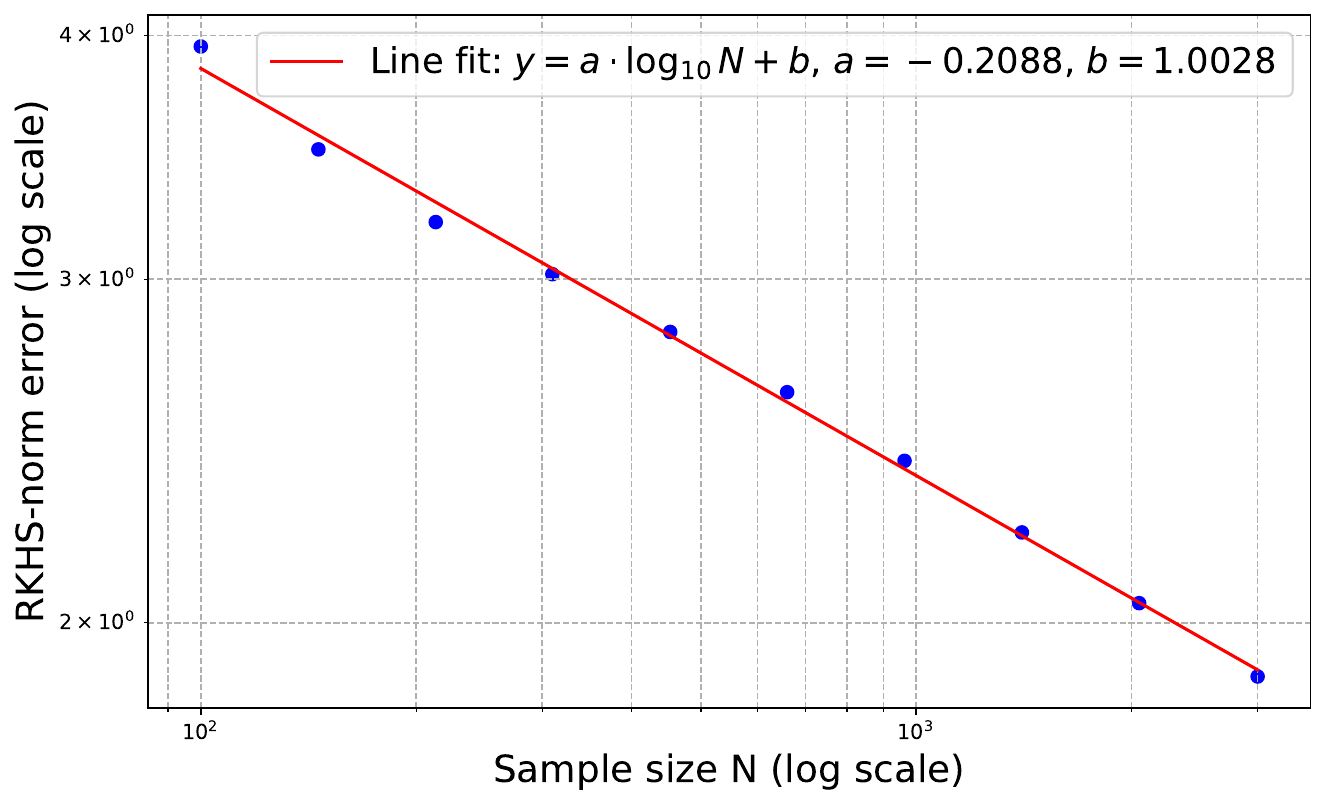}} 
    \caption{(a) Mean-std convergence of the RKHS-norm error $\left\| \widehat{h}_{\lambda,N}-H\right\|_{\mathcal{H}_K}$ versus sample size $N$ (b) Mean convergence of the RKHS-norm error in the log-scale }
    \label{rkhs-norm-convergence}
\end{figure}

We recall that by Theorem \ref{dif-rep}, the RKHS norm controls the $C_b^1$-norm up to a multiplicative constant, implying that, in the log-scale, the supremum norm upper-bound differs from the RKHS norm upper-bound by an additive constant. To investigate this, we numerically approximate the supremum norm on $[-1,1]^4$ by computing the maximum absolute value of the difference over a discretized grid of $10^4$ points, with each dimension partitioned into ten equally spaced intervals. Similar to the above, we perform fifty experiments and take the mean supremum error. We compare the RKHS-norm error and the supremum-norm error in Figure \ref{rkhs-sup-norm-convergence}. We observe that the supremum norm of the reconstruction error decreases very fast in the beginning, and then stays constant or slightly increases. 

\begin{figure}[htp]
    \centering
{\includegraphics[width=0.8\textwidth]{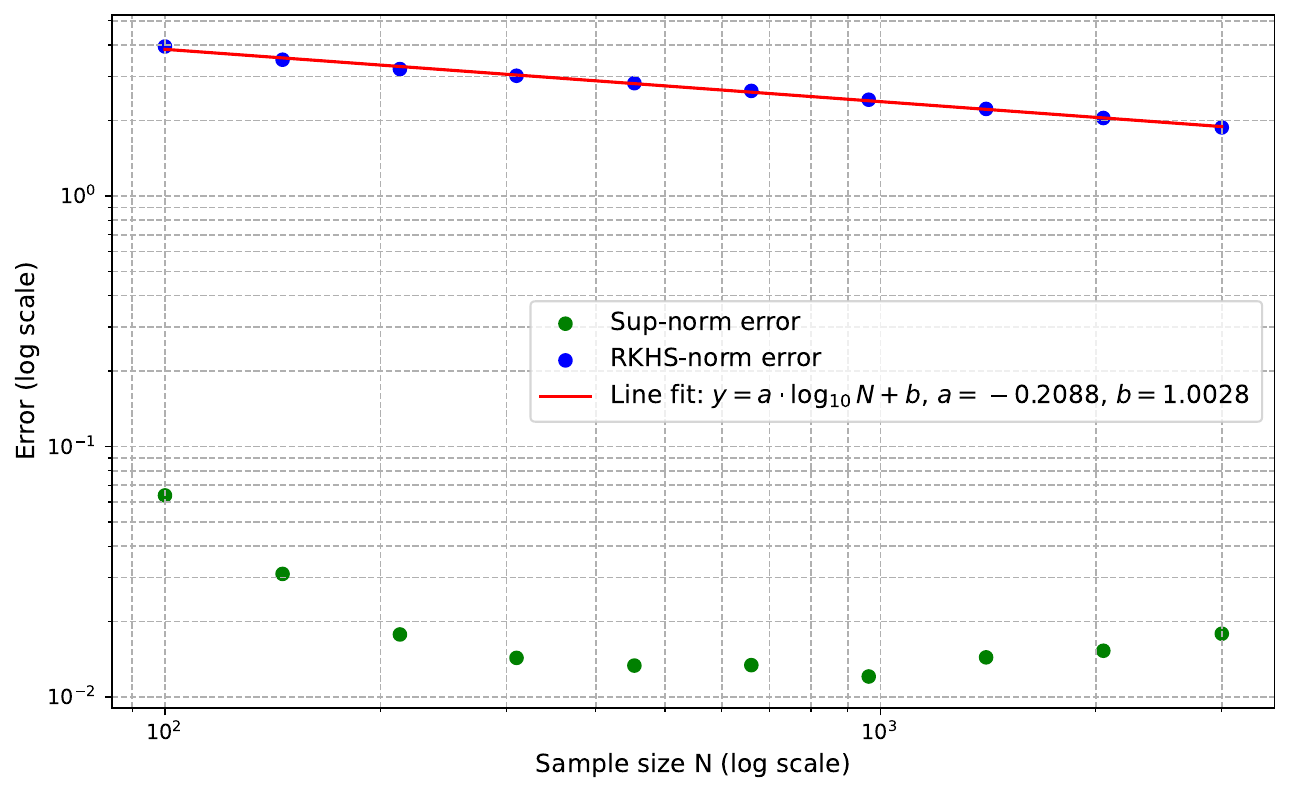}}
    \caption{Mean convergence of the RKHS-norm error and the supremum-norm error in the log-scale}
    \label{rkhs-sup-norm-convergence}
\end{figure}

\subsection{Potential function with singularities}\label{5.3}

In this subsection, we examine the necessity and the impact of the assumption that we made in Theorem \ref{Par-Rep} and, more generally, throughout most of the paper, namely that the kernel $K\in C_{b}^{2s+1}(\mathbb{R}^{2d}\times\mathbb{R}^{2d})$. As a consequence of Theorem \ref{Par-Rep}, under that hypothesis $\mathcal{H}_K$ can be embedded into $C^{s}_b(\mathbb{R}^{2d})$. In other words, our framework does not guarantee the learning of Hamiltonian functions that do not lie in $C_b^s(\mathbb{R}^{{2d}})$. Particular examples of this case are Hamiltonian functions that exhibit singularities, such as the two-body problem. As we see in the next paragraphs, our algorithm achieves, in those cases, qualitative but not quantitative learning performance, and hence proves the tightness of our assumption.

The classical system of two gravitationally interacting bodies has a Hamiltonian formulation with the Hamiltonian 
\begin{align*}
    H({\bf q},{\bf p})=\frac{\|{\bf p}\|^2_2}{2}-\frac{1}{\|{\bf q}\|_2},
\end{align*}
where $\|\cdot \|_2 $ denotes the Euclidean norm.
For the numerical experiment, we adopt $N=1000$. We sample $N$ initial conditions $(q_1,q_2,p_1,p_2)$ over a uniform distribution on $[-1,1]^4\subset \mathbb{R}^4$ and obtain the corresponding Hamiltonian vector fields. We then perform a grid search of parameters over $\eta$ in np.arange(0.2, 2, 0.2) and $c$ as in \eqref{c grid}. The optimal parameters are $\eta=0.2$ and $c=1$. We plot the potential function of the ground truth (Figure \ref{two-body} (a)) and the reconstructed Hamiltonian (Figure \ref{two-body} (b)) on the $(q_1,q_2)$ plane restricted to $[-1,1]^2$, with the optimal parameters $\eta$ and $c$. We also visualize the error in a heatmap (Figure \ref{two-body} (c)) as elaborated in the introduction of Section \ref{Numerical experiments}. Figure \ref{two-body} shows that the estimator captures the shape of the potential function, but numerically speaking, the approximation is not satisfactory.

\begin{figure}[htp]
    \centering
    \subfigure[]{\includegraphics[width=0.45\textwidth]{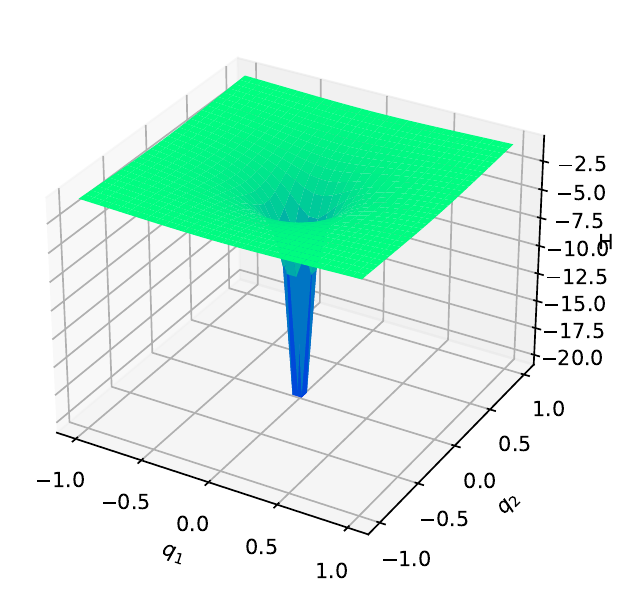}} 
    \subfigure[]{\includegraphics[width=0.45\textwidth]{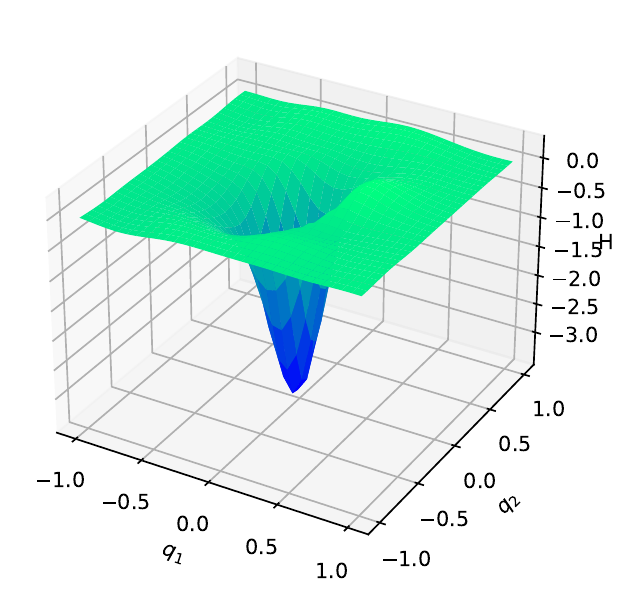}} 
    \caption{Two-body problem: (a) Ground truth potential (b) Potential of the learned Hamiltonian }
    \label{two-body}
\end{figure}

\subsection{Comparison with Hamiltonian neural networks}\label{5.4}

In this subsection, we compare the performance and training cost of our structure-preserving kernel regression with a modified version of the Hamiltonian neural network (HNN) approach \cite{greydanus2019hamiltonian}, which has been mainstream in the literature. The idea of the HNN approach is to model the Hamiltonian function as a neural network, integrate the Hamiltonian vector field to obtain trajectories, and then train the neural network by comparing the observed trajectories and the integrated trajectories. In our case, however, Hamiltonian vector fields are assumed to be available. Hence, to make the comparison fair, we train the neural network directly using vector fields. We apply the HNN approach to the example of the double pendulum in Section \ref{5.1.1} with $N=200$ and $N=500$, and the example of a non-convex potential function in Section \ref{5.2} with $N=500$ to compare with the kernel estimator. We use the same neural network architecture as in \cite{greydanus2019hamiltonian}, that is, a three-layer neural network with 200 neurons in the hidden layer and {\it tanh} as the activation function. 

To train both examples, we perform $4000$ iterations of gradient descent together with a multi-step learning rate scheduler in {\it PyTorch}, with a starting learning rate of $10^{-3}$, and parameters $\gamma=0.5$, milestones = $(200,400,800,1600,3200)$, so that the learning rate decays with a multiplicative factor of $0.5$ after each milestone stage for better convergence. We visualize the learned Hamiltonian functions and the errors of both examples in Figures \ref{double_pendulum_hnn} and \ref{nonconvex_hnn}, which should be compared with Figures \ref{double_pendulum} and \ref{nonconvex_1}. 

\begin{figure}[htp]
    \centering
    \subfigure[]{\includegraphics[width=0.45\textwidth]{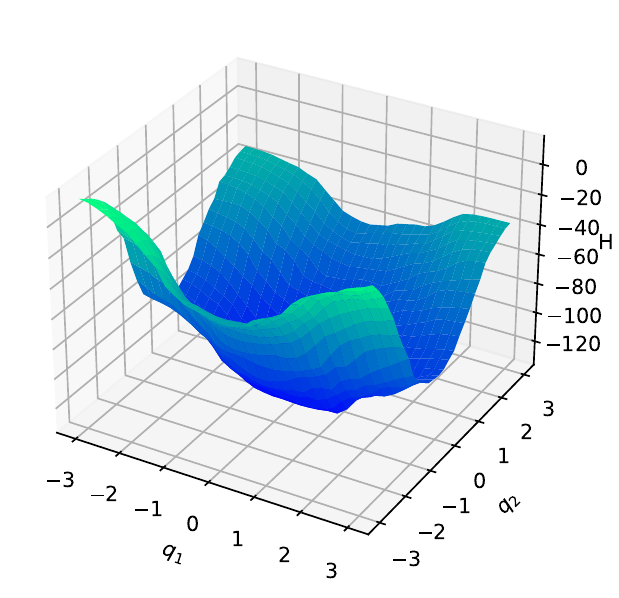}} 
    \subfigure[]{\includegraphics[width=0.45\textwidth]{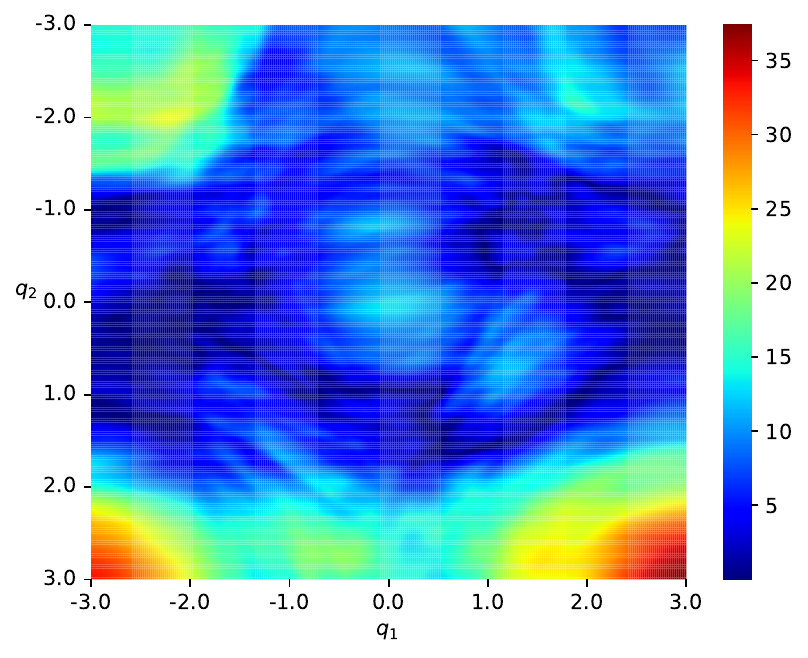}}
    \subfigure[]{\includegraphics[width=0.45\textwidth]{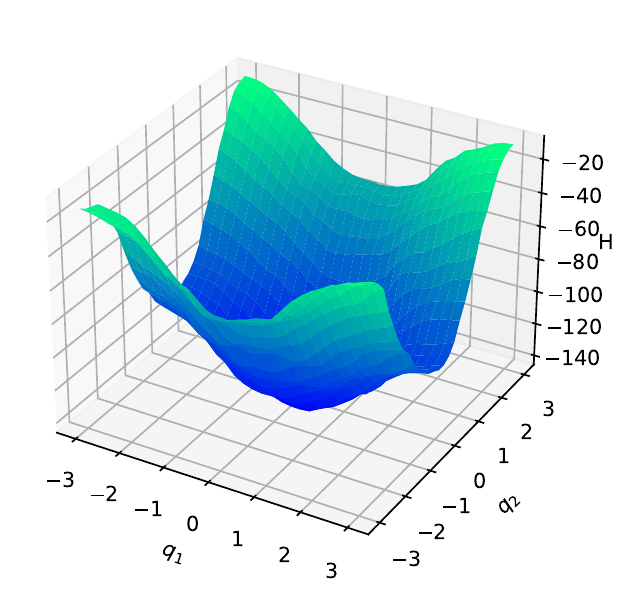}} 
    \subfigure[]{\includegraphics[width=0.45\textwidth]{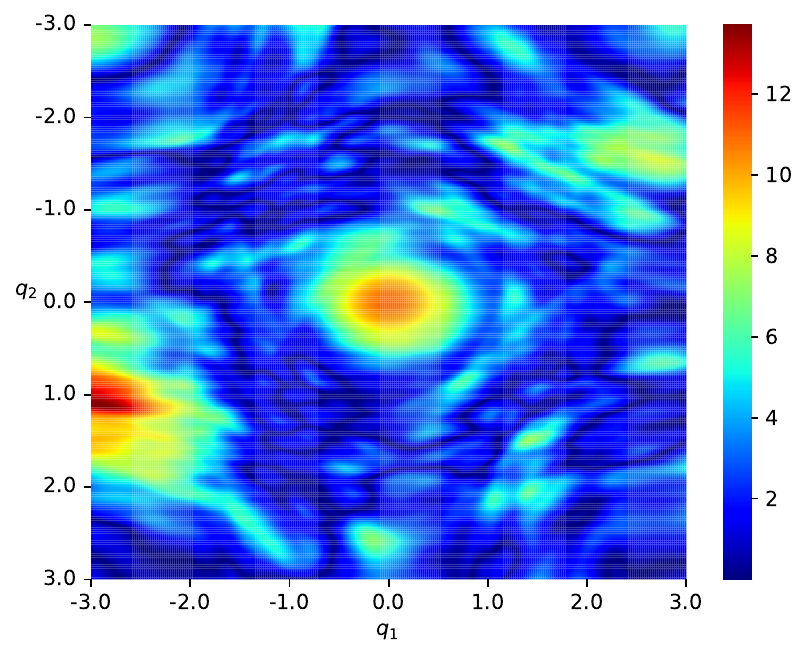}}
    \caption{Double pendulum learned with HNN: (a) and (b) with $N=200$, compared to (c) and (d) with $N=500$}
    \label{double_pendulum_hnn}
\end{figure}

\begin{figure}[htp]
    \centering
    \subfigure[]{\includegraphics[width=0.45\textwidth]{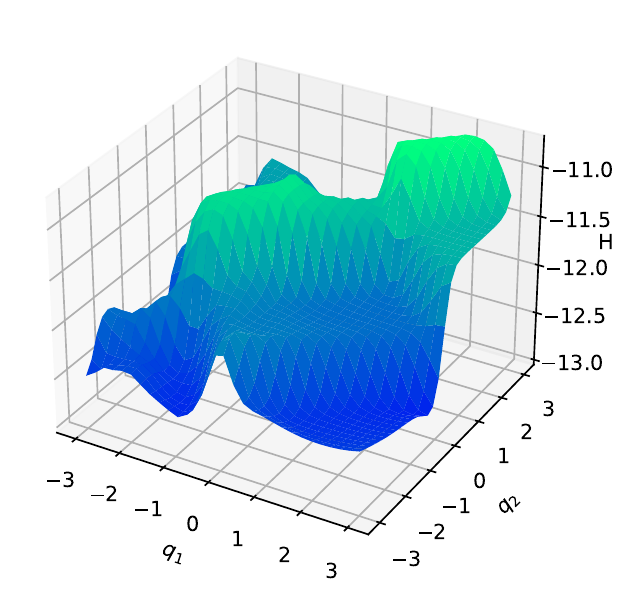}} 
    \subfigure[]{\includegraphics[width=0.45\textwidth]{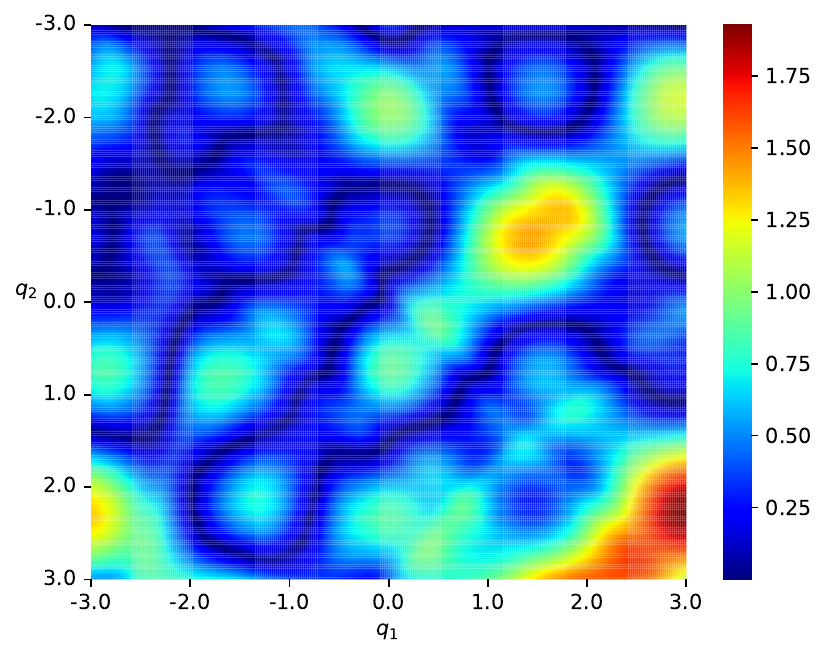}}
    \caption{Non-convex potential function learned with HNN and $N=500$ (a) Learned Hamiltonian (b) Mismatch error after vertical shift}
    \label{nonconvex_hnn}
\end{figure}

From the comparison, we can conclude that the structure-preserving kernel estimator outperforms the neural network approach both in terms of training cost and accuracy (Indeed, to fit the training data, the structure-preserving estimator requires less than a second to complete). Here is our explanation. Regarding training cost, the computation of the kernel regression method only involves matrix operations, whereas the neural network approach requires gradient descent iterations. As to the accuracy, especially in the example involving a non-convex potential function, the highly non-convex objective function imposes great difficulty for the gradient-based method to search for the global minimum of the loss function, and hence the training loss converges likely to a local minimum and ceases to decrease further after reaching a certain stage, whereas the kernel regression method circumvents such difficulty by making the learning problem convex and providing an explicit formula.

\section{Conclusion}
\label{Conclusion}

We have presented a structure-preserving kernel ridge regression method that allows the recovery of potentially high-dimensional and nonlinear Hamiltonian functions from data sets made of noisy observations of Hamiltonian vector fields. Our results generalize previous work in the literature on the learning of the Hamiltonian system describing systems of $n$ interacting particles. The methodology that we propose covers arbitrary Hamiltonian systems defined on Euclidean spaces endowed with the canonical symplectic form.

From a practical point of view, the method comes with a closed-form solution for the learning problem that yields excellent numerical performances that surpass other techniques proposed in the literature in this setup. We have illustrated this fact with several numerical experiments. Additionally, we have conducted a full error analysis that extends to our setup error bounds and convergence rates. Our contribution improves on some of those rates and can formulate them without some common hypotheses in the literature (e.g., the coercivity condition).

From the methodological point of view, our paper is the first one to extend kernel regression methods to general Hamiltonian systems in a structure-preserving fashion. Even more generally, the techniques in the paper can be adapted to handle general problems in which loss functions involving linear functions of gradients are required. In this context, we proved a differential reproducing property and an adapted version of the Representer Theorem. 

This paper is just a first step in the structure-preserving learning of autonomous Hamiltonian systems. In our forthcoming works, we are considering four main challenges. First, most Hamiltonian systems are defined in non-Euclidean spaces (e.g., pendula, rigid bodies) that could even be infinite-dimensional (e.g., ideal fluids, elasticity). The methods in this paper do not apply to these important applicative situations and, hence, need to be extended.

Second, many Hamiltonian systems that appear in important applications have external ports because they need to be controlled (e.g., electric circuits, robotics). The authors have already contributed to the learning of some simple linear port-Hamiltonian systems \cite{RCSP1} using traditional geometric mechanical treatments. However, it seems more appropriate to use kernel-based techniques when solving the non-linear case.

Third, a great wealth of knowledge has been accumulated on the qualitative behavior of Hamiltonian dynamical systems based on their geometry \cite{Abraham1978}, symmetries \cite{Ortega2004}, stability properties \cite{Ortega2005re}, or bifurcation phenomena \cite{pascal}. All these concepts surely have an interplay in relation to learnability that needs to be explored.

Finally, the regression setup in which we have placed ourselves involves using vector field data. It is nevertheless much more realistic to work with discrete-time trajectory data. This automatically puts structure-preserving integrators in the picture. Even though variational, symplectic, and other structure-preserving integrators have been thoroughly studied in the Hamiltonian context, their interplay with learning schemes has not been fully exploited.

\appendix

\section{Auxiliary results and proofs}

\begin{lemma}[Appendix A in \cite{williams2006gaussian}]\label{Con-Gau}
Let $\mathbf{x}$ and $\mathbf{y}$ be jointly Gaussian random vectors
\begin{equation*}
\begin{bmatrix}
\mathbf{x}\\ \mathbf{y}
\end{bmatrix}
\sim \mathcal{N} (
\begin{bmatrix}
\mu_{\mathbf{x}}\\ \mu_{\mathbf{y}}
\end{bmatrix}
, 
\begin{bmatrix} 
A & C\\
C^{\top} & B
\end{bmatrix}
),
\end{equation*}
then the marginal distribution of $\mathbf{x}$ and the conditional distribution of $\mathbf{x}$ given $\mathbf{y}$ are
\begin{equation*}
 \mathbf{x} \sim \mathcal{N}(\mu_{\mathbf{x}},A), \quad \textrm{and } \mathbf{x}\mid\mathbf{y} \sim \mathcal{N}(\mu_{\mathbf{x}} + CB^{-1}(\mathbf{y} - \mu_{\mathbf{y}}), A - CB^{-1}C^{\top}).
\end{equation*}
\end{lemma}

\begin{lemma}[Lemma 8 in \cite{de2005learning}]\label{Hil-Bou}
Let $\mathcal{H}$ be a Hilbert space and $\xi$ be a random variable on $(Z,\rho)$ with values in $\mathcal{H}$. Suppose that, $\|\xi\|_{\mathcal{H}}\leq S < \infty$ almost surely. Let $z_n$ be IID drawn from $\rho$. For any $0<\delta<1$, with confidence $1-\delta$,
$$\bigg\| \frac{1}{N}\sum_{n=1}^{N}(\xi(z_n)-\mathbb{E}(\xi))\bigg\| \leq \frac{4S\log(2/\delta)}{N}+\sqrt{\frac{2\mathbb{E}(\|\xi\|_{H}^2)\log(2/\delta)}{N}}.$$
\end{lemma}
The original version of Lemma \ref{Hil-Bou} is presented in \cite{yurinsky1995sums}.

\begin{lemma}
\label{feature map is continuous}
Let $K: {\cal X} \times {\cal X} \longrightarrow \mathbb{R}$ be a continuous kernel map defined on the topological space ${\cal X} $ and let ${\mathcal H} _K  $ be the associated RKHS. Then, the canonical feature map $\Phi: {\cal X} \longrightarrow {\mathcal H} _K  $ is continuous.
\end{lemma}

\begin{proof}
The continuity of $K$ implies that the kernel sections $K _x \in C({\cal X})$ for all $x \in {\cal X} $. Now let $x _0 \in {\cal X} $ be arbitrary and let $\left\{x _n\right\}_{n=0} ^{\infty} $ be a sequence in ${\cal X}  $ such that $\lim\limits_{n \rightarrow \infty} x _n = x _0 $. Then, 
\begin{equation*}
\left\|\Phi(x _n)- \Phi(x _0)\right\|_{{\mathcal H}_K} ^2= \langle K_{x _n}- K_{x _0}, K_{x _n}- K_{x _0}\rangle_{{\mathcal H}_K}=K(x _n, x _n)+K(x _0, x _0)-2K(x _n, x _0)\stackrel{n \rightarrow \infty}{\longrightarrow}0,
\end{equation*}
by the continuity of $K$, which proves the continuity of $\Phi $.
\end{proof}

\begin{theorem}[Hanson-Wright inequality \cite{rudelson2013hanson}] \label{Han-Wri}
Let $X=(X_1,\cdots,X_n) \in \mathbb{R}^n$ be a random vector with independent components $X_i$ which satisfy 
$\mathbb{E} [X_i]=0$ and $\|X_i\|_{\varphi_2} \leq S_0$, where $\|\cdot\|_{\varphi_2}$ is the subGaussian norm. Let $A$ be an $n \times n$ matrix and $\|A\|=\max_{\|x\|_2\leq 1}\|Ax\|_2$, and $\|A\|_{HS}$ denote the Hilbert-Schmidt norm. Then, for every $\varepsilon \geq 0$
$$\mathbb{P}\bigg\{\bigg\| X^{\top}AX-\mathbb{E} [X^{\top}AX] \bigg\| \geq \varepsilon \bigg\} \leq 2\exp \bigg\{ -c \min \bigg\{ \frac{\varepsilon^2}{S_0^4\|A\|_{HS}^2}, \frac{\varepsilon}{S_0^2\|A\|}\bigg\} \bigg\},$$ where $c$ is an absolute positive constant. 
\end{theorem}

\begin{lemma}[Gr\"{o}nwall's inequality, Lemma 1.1 in \cite{barbu2016differential}]\label{Gro-Ine}
Let $x, \alpha$ and $\beta$ are continuous functions on $[a, b]$ and $\beta(t) \geq 0, \forall t \in[a, b]$.
If
$$
x(t) \leq \alpha(t)+\int_a^t \beta(s) x(s) \mathrm{d} s, \quad t \in[a, b],
$$
then $x(t)$ satisfies the inequality
$$
x(t) \leq \alpha(t)+\int_a^t \alpha(s) \beta(s) \exp \left(\int_s^t \beta(\tau) \mathrm{d}\tau\right) \mathrm{d} s.
$$
If, in addition, $\alpha$ is non-decreasing, then
$$
x(t) \leq \alpha(t) \exp \left(\int_a^t \beta(s) \mathrm{d}s\right). 
$$
\end{lemma}

\begin{lemma} \label{Dec-Omp}
Let  $K\in C_b^3(\mathbb{R}^{2d}\times\mathbb{R}^{2d})$ be a Mercer kernel. For any function $h \in \mathcal{H}_K$ and $0< \delta <1$, with probability at least $1-\delta$, there holds 
\begin{align*}
\|B_N h-B h\|_{\mathcal{H}_K} \leq \left(\sqrt{ \frac{8\log(2/\delta)}{N}}+ 
1\right)\sqrt{ \frac{2\log(2/\delta)}{N}}2d\kappa^2\|h\|_{\mathcal{H}_K}.
\end{align*}
\end{lemma}
\begin{proof}
Since $K\in C_b^3(\mathbb{R}^{2d}\times\mathbb{R}^{2d})$, then for any  $h\in\mathcal{H}_K$, we have 
\begin{align*}
\|B_Nh\|_{\mathcal{H}_K}&=\frac{1}{N}\|\nabla^{\top}h(\mathbf{Z}_N)\nabla_1 K(\mathbf{Z}_N,\cdot)\|_{\mathcal{H}_K}\leq\frac{1}{N}\sum_{n=1}^{N}\sum_{i=1}^{2d}\Big|\partial_ih(\mathbf{Z}^{(n)})\Big|\Big\|(\nabla_1K)_i(\mathbf{Z}^{(n)},\cdot)\Big\|_{\mathcal{H}_K}\\
&\leq 2d\kappa\|h\|_{C_b^1}\leq  2d\kappa^2\|h\|_{\mathcal{H}_K},
\end{align*}
which shows that $B_Nh$ are bounded random variables in $\mathcal{H}_K$. Moreover, we have $\mathbb{E}\left[\|B_Nh\|^2_{\mathcal{H}_K}\right]\leq 4d^2\kappa^4\|h\|^2_{\mathcal{H}_K}$.  Define now the $\mathcal{H}_K$-valued random variables
$$
\xi^{(n)}=\sum_{i=1}^{2d}\partial_i h(\mathbf{Z}^{(n)})(\nabla_1K)_i(\mathbf{Z}^{(n)},\cdot) \quad \mbox{$n=1, \ldots, N$.} \quad
$$ 
Note that the random variables $\{\xi^{(n)}\}_{n=1}^{N}$ are IID and that 
$B_Nh-Bh=\frac{1}{N}\sum_{n=1}^N (\xi^{(n)}-\mathbb{E}(\xi^{(n)})).$ The result follows by applying Lemma \ref{Hil-Bou} to $\{\xi^{(n)}\}_{n=1}^{N}$. 
\end{proof}

\section{Detailed proofs of theorems, propositions, and corollaries}
\subsection{Proof of Theorem \ref{Par-Rep}}\label{The proof of some theorems}

The proof follows a strategy very similar to the one in \cite{zhou2008derivative} in which we have circumvented all the arguments that needed the compactness in the space where the kernel is defined. 
\begin{proof}
We prove {\bf (i)} and {\bf (ii)} together by induction on $|\alpha|$.
The case $|\alpha|=0$ is trivial since that means $\alpha=0$ and for any $\mathbf{x}\in\mathbb{R}^d$, $(D^{0}K)_{\mathbf{x}}=K_{\mathbf{x}}$ satisfies the standard reproducing property in $\mathcal{H}_K$.
    
Let $0\leq l\leq s-1$. Suppose that $(D^{\alpha}K)_{\mathbf{x}}\in \mathcal{H}_K$ and \eqref{dif-rep} holds for any $\mathbf{x}\in \mathbb{R}^{d}$ and $\alpha \in I_l$. Then \eqref{dif-rep} implies that for any $\mathbf{y}\in \mathbb{R}^{d}$,
\begin{equation}\label{def-rep}
\langle (D^{\alpha}K)_{\mathbf{y}}, (D^{\alpha}K)_{\mathbf{x}}\rangle_{\mathcal{H}_K} = D^{\alpha}((D^{\alpha}K)_{\mathbf{x}})(\mathbf{y})=D^{\alpha}(D^{\alpha}K(\mathbf{x},\cdot))(\mathbf{y})=D^{(\alpha,\alpha)}K(\mathbf{x},\mathbf{y}).
\end{equation}
Now, we turn to the case $l+1$. Consider the index $\alpha+e_j$, where $e _j $ is the $j$th-canonical vector with a $1$ in the entry $j$ and $0$ elsewhere. Note that $|\alpha+e_j|=l+1$. We prove that {\bf (i)} and {\bf (ii)} hold for this index in three steps.

\noindent {\bf  Step 1}: Proving $(D^{\alpha+e_j}K)_{\mathbf{x}}\in \mathcal{H}_K$ for $\mathbf{x}\in \mathbb{R}^{d}$. For some $r>0$, the set $\{\frac{1}{t}((D^{\alpha}K)_{\mathbf{x}+te_j}-(D^{\alpha}K)_{\mathbf{x}}):|t|\leq r\}$ of functions in $\mathcal{H}_K$ satisfies that
\begin{align}
&\Big\|\frac{1}{t}((D^{\alpha}K)_{\mathbf{x}+te_j}-(D^{\alpha}K)_{\mathbf{x}})\Big\|^2_{\mathcal{H}_K}\notag\\
&=\frac{1}{t^2}\Big(D^{(\alpha,\alpha)}K(\mathbf{x}+te_j,\mathbf{x}+te_j)-D^{(\alpha,\alpha)}K(\mathbf{x}+te_j,\mathbf{x})
-D^{(\alpha,\alpha)}K(\mathbf{x},\mathbf{x}+te_j)+D^{(\alpha,\alpha)}K(\mathbf{x},\mathbf{x})\Big)\notag\\
&\leq \|D^{(\alpha+e_j,\alpha+e_j)}K\|_{\infty}<\infty, \quad \forall ~|t|\leq r, \label{bound for derivative}
\end{align}
where we have used the assumption that $K\in C_b^{2s+1}(\mathbb{R}^{d}\times \mathbb{R}^{d})$, and that $|(\alpha+e_j,\alpha+e_j)|=2|\alpha|+2=2l+2\leq 2s$.

Since the bound \eqref{bound for derivative} does not depend on $t$, this means that $\{\frac{1}{t}((D^{\alpha}K)_{\mathbf{x}+te_j}-(D^{\alpha}K)_{\mathbf{x}}):|t|\leq r\}$ lies in a closed ball of the Hilbert space $\mathcal{H}_K$ with a finite radius. Since this ball is sequentially weakly compact (see \cite[Theorem 4.2]{conway:book}), there is a sequence $\{t_i\}_{i=1}^{\infty}$ with $|t_i|\leq r$ and $\lim_{i\rightarrow \infty}t_i=0$ such that $\{\frac{1}{t_i}((D^{\alpha}K)_{\mathbf{x}+t_ie_j}-(D^{\alpha}K)_{\mathbf{x}}):|t|\leq r\}$ converges weakly to an element $g_{\mathbf{x}}$ of $\mathcal{H}_K$ as $i\rightarrow\infty$. The weak convergence tells us that

\begin{equation}\label{wea-con}
\lim\limits_{i\rightarrow \infty}\Big \langle \frac{1}{t_i}((D^{\alpha}K)_{\mathbf{x}+t_ie_j}-(D^{\alpha}K)_{\mathbf{x}}), f \Big \rangle_{\mathcal{H}_K} = \langle g_{\mathbf{x}},f \rangle_{\mathcal{H}_K}, \quad \forall f\in \mathcal{H}_K.
\end{equation}
In particular, by taking $f=K_{\mathbf{y}}$ with $\mathbf{y}\in \mathbb{R}^{d}$, it holds that
\begin{equation*}
g_{\mathbf{x}}(\mathbf{y})=\lim\limits_{i\rightarrow\infty}\Big \langle \frac{1}{t_i}((D^{\alpha}K)_{\mathbf{x}+t_ie_j}-(D^{\alpha}K)_{\mathbf{x}}), K_{\mathbf{y}} \Big \rangle_{\mathcal{H}_K}.
\end{equation*}
By the induction hypothesis, we then have that
\begin{align*}
g_{\mathbf{x}}(\mathbf{y}) &= \lim\limits_{i\rightarrow\infty} \frac{1}{t_i}(D^{\alpha}(K_{\mathbf{y}})(\mathbf{x}+t_ie_j)-D^{\alpha}(K_{\mathbf{y}})(\mathbf{x}))\\
&=\lim\limits_{i\rightarrow\infty} \frac{1}{t_i}(D^{\alpha}K(\mathbf{x}+t_ie_j,\mathbf{y})-D^{\alpha}K(\mathbf{x},\mathbf{y}))=D^{\alpha+e_j}K(\mathbf{x},\mathbf{y})=(D^{\alpha+e_j}K)_{\mathbf{x}}(\mathbf{y}).
\end{align*}
This is true for an arbitrary point $\mathbf{y}\in \mathbb{R}^{d}$. Hence $(D^{\alpha+e_j}K)_{\mathbf{x}}=g_{\mathbf{x}}$ as functions on $\mathbb{R}^{d}$. Since $g_{\mathbf{x}}\in \mathcal{H}_K$, we have hence shown that $(D^{\alpha+e_j}K)_{\mathbf{x}}\in \mathcal{H}_K$.

\noindent {\bf  Step 2:} Proving the convergence 
\begin{equation}\label{con-in-hk}
\frac{1}{t}((D^{\alpha}K)_{\mathbf{x}+te_j}-(D^{\alpha}K)_{\mathbf{x}})\underset{t\rightarrow 0}{\longrightarrow}(D^{\alpha+e_j}K)_{\mathbf{x}} \quad \mbox{in ${\mathcal H} _K $, for all $\mathbf{x} \in \mathbb{R}^d $.} 
\end{equation}

Applying the induction hypothesis together with \eqref{wea-con} for $\alpha$ to the function $(D^{\alpha+e_j}K)_{\mathbf{x}} \in \mathcal{H}_K$ yields 
\begin{align*}
\langle (D^{\alpha+e_j}K)_{\mathbf{x}} , &D^{\alpha+e_j}K)_{\mathbf{x}}\rangle_{\mathcal{H}_K}\\
&=\lim\limits_{i\rightarrow\infty}\frac{1}{t_i} \Big( D^{\alpha}((D^{\alpha+e_j}K)_{\mathbf{x}})(\mathbf{x}+t_ie_j)-D^{\alpha}((D^{\alpha+e_j}K)_{\mathbf{x}})(\mathbf{x})\Big )\\
&=\lim\limits_{i\rightarrow\infty}\frac{1}{t_i} \Big( D^{\alpha}(D^{\alpha+e_j}K(\mathbf{x},\cdot))(\mathbf{x}+t_ie_j)-D^{\alpha}(D^{\alpha+e_j}K(\mathbf{x},\cdot))(\mathbf{x})\Big )\\
&=D^{(\alpha+e_j,\alpha+e_j)}K(\mathbf{x},\mathbf{x}).
\end{align*}
Together with the induction hypothesis, this implies that
\begin{align*}
\Big\|\frac{1}{t}&\Big((D^{\alpha}K)_{\mathbf{x}+te_j}-(D^{\alpha}K)_{\mathbf{x}}\Big)-(D^{\alpha+e_j}K)_{\mathbf{x}}\Big\|^2_{\mathcal{H}_K}\\
=~&\frac{1}{t^2}\Big(D^{(\alpha,\alpha)}K(\mathbf{x}+te_j,\mathbf{x}+te_j)-2D^{(\alpha,\alpha)}K(\mathbf{x}+te_j,\mathbf{x})+D^{(\alpha,\alpha)}K(\mathbf{x},\mathbf{x})\Big)\\
&-\frac{2}{t}\Big(D^{(\alpha,\alpha+e_j)}K(\mathbf{x}+te_j,\mathbf{x})-D^{(\alpha,\alpha+e_j)}K(\mathbf{x},\mathbf{x})\Big)+D^{(\alpha+e_j,\alpha+e_j)}K(\mathbf{x},\mathbf{x})\\
=~&\frac{1}{t^2}\int_{0}^t\int_{0}^tD^{(\alpha+e_j,\alpha+e_j)}K(\mathbf{x}+ue_j,\mathbf{x}+ve_j)~\mathrm{d}u\mathrm{d}v\\
	&-\frac{2}{t}\int_0^{t}D^{(\alpha+e_j,\alpha+e_j)}K(\mathbf{x},\mathbf{x}+ve_j)~\mathrm{d}v+D^{(\alpha+e_j,\alpha+e_j)}K(\mathbf{x},\mathbf{x})\\
=~&\frac{1}{t^2}\int_{0}^t\int_{0}^t\Big(D^{(\alpha+e_j,\alpha+e_j)}K(\mathbf{x}+ue_j,\mathbf{x}+ve_j)-2D^{(\alpha+e_j,\alpha+e_j)}K(\mathbf{x},\mathbf{x}+ve_j)+D^{(\alpha+e_j,\alpha+e_j)}K(\mathbf{x},\mathbf{x})\Big)~\mathrm{d}u\mathrm{d}v\\
\leq~&\frac{1}{t^2}\int_{0}^t\int_{0}^t\|D^{(\alpha+2e_j,\alpha+e_j)}K\|_{\infty} u~\mathrm{d}u\mathrm{d}v+\frac{1}{t^2}\int_{0}^t\int_{0}^t \|D^{(\alpha+e_j,\alpha+2e_j)}K\|_{\infty}v~\mathrm{d}u\mathrm{d}v\\
\leq~&\frac{t}{2}\sup_{|\alpha+\beta|=2s+1}\|D^{(\alpha,\beta)}K\|_{\infty},
\end{align*}
where the fourth inequality is due to the mean value theorem and in the last two we use the fact that $K\in C_b^{2s+1}(\mathbb{R}^{d}\times \mathbb{R}^{d})$.
\eqref{con-in-hk} follows by taking the limit $t\rightarrow 0$.

\noindent {\bf  Step 3:} Proving \eqref{dif-rep} for $\mathbf{x}\in \mathbb{R}^{d}$ and $\alpha+e_j$. Let $f\in \mathcal{H}_K$. By (\ref{con-in-hk}) we have
\begin{equation*}
\langle (D^{\alpha+e_j}K)_{\mathbf{x}},f \rangle_{\mathcal{H}_K}=\lim\limits_{t\rightarrow0}\Big\langle \frac{1}{t}((D^{\alpha}K)_{\mathbf{x}+te_j}-(D^{\alpha}K)_{\mathbf{x}}),f \Big \rangle_{\mathcal{H}_K}.
\end{equation*}
Since by the induction hypothesis \eqref{dif-rep} holds for $\alpha$,  this implies
\begin{equation*}
\langle (D^{\alpha+e_j}K)_{\mathbf{x}},f \rangle_{\mathcal{H}_K}=\lim\limits_{t\rightarrow0}\frac{1}{t}\left ( D^{\alpha}f({\mathbf{x}+te_j})-D^{\alpha}f(\mathbf{x})\right ).
\end{equation*}
That is, $D^{\alpha+e_j}f(\mathbf{x})$ exists and equals $\langle (D^{\alpha+e_j}K)_{\mathbf{x}},f\rangle_{\mathcal{H}_K}$. This verifies (\ref{dif-rep}) for $\alpha+e_j$.\\
    
We conclude by proving part {\bf (iii)} using \eqref{dif-rep} and \eqref{def-rep}. For $f\in\mathcal{H}_K$, $\mathbf{x},\mathbf{x}^{\prime}\in \mathbb{R}^d$, and $\alpha\in I_s$, the Cauchy-Schwarz inequality implies that
\begin{align*}
|D^{\alpha}f(\mathbf{x})-D^{\alpha}f(\mathbf{x}^{\prime})|^2&=|\langle (D^{\alpha}K)_{\mathbf{x}}-(D^{\alpha}K)_{\mathbf{x}^{\prime}},f\rangle_{\mathcal{H}_K}|^2\leq \|(D^{\alpha}K)_{\mathbf{x}}-(D^{\alpha}K)_{\mathbf{x}^{\prime}}\|^2_{\mathcal{H}_K}\|f\|^2_{\mathcal{H}_K}\\
&\leq \Big(D^{(\alpha,\alpha)}K(\mathbf{x},\mathbf{x})-2D^{(\alpha,\alpha)}K(\mathbf{x},\mathbf{x}^{\prime})+D^{(\alpha,\alpha)}K(\mathbf{x}^{\prime},\mathbf{x}^{\prime})\Big) ~\|f\|^2_{\mathcal{H}_K}\\
&=2\sup_{|\alpha+\beta|=2s+1}\|D^{(\alpha,\beta)}K\|_{\infty}~\|f\|^2_{\mathcal{H}_K}\|\mathbf{x}-\mathbf{x}^{\prime}\|,
\end{align*}
which shows that $D^{\alpha}f$ is H\"older continuous with exponent $1/2 $, that is,  $D^{\alpha}f\in C^{0, \frac{1}{2}}(\mathbb{R}^{d})$. Moreover, again by the Cauchy-Schwarz inequality and the relation \eqref{dif-rep}, we have that
\begin{equation*}
|D^{\alpha}f(\mathbf{x})|=|\langle (D^{\alpha}K)_{\mathbf{x}},f \rangle_{\mathcal{H}_K}|\leq \sqrt{D^{(\alpha,\alpha)}K(\mathbf{x},\mathbf{x})}~\|f\|_{\mathcal{H}_K}\leq \sqrt{\|D^{(\alpha,\alpha)}K\|_{\infty}}~\|f\|_{\mathcal{H}_K}.
\end{equation*}
Therefore, $f\in C^{s}_b(\mathbb{R}^{d})$.
It follows that
\begin{align*}
\|f\|_{C^{s}_b}&=\sup_{\alpha\in I_s}\|D^{\alpha}f\|_{\infty}\leq \sup_{\alpha\in I_s}\sqrt{\|D^{(\alpha,\alpha)}K\|_{\infty}}\|f\|_{\mathcal{H}_K}\\
&= \sqrt{\sup_{\alpha\in I_s}\|D^{(\alpha,\alpha)}K\|_{\infty}}\|f\|_{\mathcal{H}_K}\leq \sqrt{\|K\|_{C^{2s}_b(\mathbb{R}^{d} \times \mathbb{R}^d)}}\|f\|_{\mathcal{H}_K}.\qquad  \square
\end{align*}
\end{proof}

\subsection{Proof of Proposition \ref{Wel-Ope}}
\begin{proof}
By Theorem \ref{Par-Rep}, $K\in C_b^3(\mathbb{R}^{2d}\times\mathbb{R}^{2d})$ implies that $\mathcal{H}_K\subseteq C_b^1(\mathbb{R}^{2d})$ and that the operator $A$ is well-defined as a map $A: {\mathcal H}_K \longrightarrow L^2(\mathbb{R}^{2d};\mu_{\mathbf{Z}};\mathbb{R}^{2d})$. 
Indeed, for any  $h\in \mathcal{H}_K$, this fact together with part {\bf (iii)} in Theorem \ref{Par-Rep} imply that
\begin{equation*}
\begin{aligned}
\|A h\|^2_{L^2(\mu_{{\bf Z}})} =\int_{\mathbb{R}^{2d}}\|J\nabla h(\mathbf{x})\|^2\mathrm{d}\mu_{\mathbf{Z}}(\mathbf{x}) =\int_{\mathbb{R}^{2d}}\|\nabla h(\mathbf{x})\|^2\mathrm{d}\mu_{\mathbf{Z}}(\mathbf{x})\leq 2d\|h\|_{C_b^1}^2\leq 2d\kappa^2\|h\|_{\mathcal{H}_K}^2,
\end{aligned}    
\end{equation*}
which shows that $A: {\mathcal H}_K \longrightarrow L^2(\mathbb{R}^{2d};\mu_{\mathbf{Z}};\mathbb{R}^{2d})$ is a bounded linear operator and that $\|A\|\leq \sqrt{2d}\kappa$.

Next, we prove \eqref{adjoint}. For any
$h \in \mathcal{H}_K$ and any $g\in L^2(\mathbb{R}^{2d};\mu_{\mathbf{Z}};\mathbb{R}^{2d})$,
\begin{align*}
\langle A h,g\rangle_{L^2(\mu_{{\bf Z}})} &=    \langle J\nabla h,g\rangle_{L^2(\mu_{{\bf Z}})} = \int_{\mathbb{R}^{2d}} \langle \nabla h(\mathbf{x}), J^{T}g(\mathbf{x})\rangle_{\mathbb{R}^{2d}}  \mathrm{d} \mu_{\mathbf{Z}}(\mathbf{x})\\
&=\int_{\mathbb{R}^{2d}} \left\langle \langle h, \nabla_1 K(\mathbf{x},\cdot) \rangle_{\mathcal{H}_K}, J^{\top}g(\mathbf{x})\right\rangle_{\mathbb{R}^{2d}} \, \mathrm{d} \mu_{\mathbf{Z}}(\mathbf{x})=\int_{\mathbb{R}^{2d}} \langle h, g^{\top}(\mathbf{x})J\nabla_1 K(\mathbf{x},\cdot) \rangle_{\mathcal{H}_K} \, \mathrm{d} \mu_{\mathbf{Z}}(\mathbf{x})\\
&=\left\langle h, \int_{\mathbb{R}^{2d}} g^{\top}(\mathbf{x})J\nabla_1 K(\mathbf{x},\cdot) \, \mathrm{d} \mu_{\mathbf{Z}}(\mathbf{x}) \right\rangle_{\mathcal{H}_K},
\end{align*} 
where third equality is due to the partial derivative reproducing property \eqref{dif-rep} in Theorem \ref{Par-Rep}. Since $g\in L^2(\mathbb{R}^{2d};\mu_{\mathbf{Z}};\mathbb{R}^{2d})$ in the previous equality is arbitrary, we have hence shown that \eqref{adjoint} holds.

Since $B=A^{\ast}A$, $B$ is clearly a bounded linear operator. Equation \eqref{positive} follows from \eqref{adjoint} by direct calculation and the fact that the integral commutes with the scalar product. Indeed, for any $h\in\mathcal{H}_K$,  
\begin{align*}
Bh &= A^{\ast}A h= A^{\ast}(J\nabla h)=\int_{\mathbb{R}^{2d}} \nabla^{\top}h(\mathbf{x}) J^{\top}J\nabla_1 K(\mathbf{x},\cdot) \, \mathrm{d} \mu_{\mathbf{Z}}(\mathbf{x})\\
&=\int_{\mathbb{R}^{2d}} \nabla^{\top}h(\mathbf{x}) \nabla_1 K(\mathbf{x},\cdot) \, \mathrm{d} \mu_{\mathbf{Z}}(\mathbf{x}).
\end{align*}

We now prove that $B$ is a trace class operator, that is, we show that $\operatorname{Tr}(|B|)<\infty$, where $|B|=\sqrt{B^* B}$. Since $B$ is positive semidefinite, we have that $|B|=B$. Therefore, it is equivalent to show that $\operatorname{Tr}(B)<\infty$. In order to do that, we choose a spanning orthonormal set $\left\{e _n\right\} _{n \in \mathbb{N}}$ for ${\mathcal H} _K$ whose existence is guaranteed by the continuity of the canonical feature map associated to $K$ that we established in Lemma \ref{feature map is continuous} and \cite[Theorem 2.4]{owhadi2017separability}. Then,
$$
\begin{aligned}
\operatorname{Tr}(B)&=\operatorname{Tr}\left(A^* A\right) =\sum_n\left\langle A^* A e_n, e_n\right\rangle_{\mathcal{H}_K}=\sum_n\left\langle A e_n, A e_n\right\rangle_{L^2\left(\mu_{\mathbf{Z}}\right)} \\
& =\sum_n\int_{\mathbb{R}^{2d}}\sum_{i=1}^{2d} \left|\partial_ie_n(\mathbf{x})\right|^2\mathrm{d}\mu_{\mathbf{Z}}(\mathbf{x})= \int_{\mathbb{R}^{2d}} \sum_{i=1}^{2d}(\nabla_{1,2}K)_{i,i}(\mathbf{x},\mathbf{x})\mathrm{d}\mu_{\mathbf{Z}}(\mathbf{x}) \\
& \leq 2d\kappa^2,
\end{aligned}
$$
where the fifth equality is due to 
\begin{align*}
\sum_n |\partial_ie_n(\mathbf{x})|^2&=
\sum_n \partial_ie_n(\mathbf{x})\langle(\nabla_1K)_i(\mathbf{x},\cdot),e_n\rangle_{\mathcal{H}_K}=\sum_n \langle(\nabla_1K)_i(\mathbf{x},\cdot),\partial_ie_n(\mathbf{x})e_n\rangle_{\mathcal{H}_K}\\
&= \langle(\nabla_1K)_i(\mathbf{x},\cdot),\sum_n\partial_ie_n(\mathbf{x})e_n\rangle_{\mathcal{H}_K} = \langle(\nabla_1K)_i(\mathbf{x},\cdot),\sum_n\langle(\nabla_1K)_i(\mathbf{x},\cdot),e_n\rangle_{\mathcal{H}_K}e_n\rangle_{\mathcal{H}_K}\\
&=\langle(\nabla_1K)_i(\mathbf{x},\cdot),(\nabla_1K)_i(\mathbf{x},\cdot)\rangle_{\mathcal{H}_K}=(\nabla_{1,2}K)_{i,i}(\mathbf{x},\mathbf{x}).    
\end{align*}
Finally, the form of the operator $B=A^{*}A$ automatically guarantees that it is positive semidefinite. A non-trivial kernel occurs when constant functions in $\mathbb{R}^{2d}$ belong to ${\mathcal H}_K $.
\end{proof}

\subsection{Proof of Proposition \ref{Wel-Emp}}

\begin{proof}
The formal explicit forms of $A_N^{\ast}$ and $B_N$ follow from a direct computation. We now show that $A_N$ and $B_N$ are bounded linear operators and that $B_N$ is a compact operator. We have
\begin{equation*}
\|A_Nh\|^2=\frac{1}{N}\sum_{n=1}^{N}\|J\nabla h(\mathbf{Z}^{(n)})\|^2=\frac{1}{N}\sum_{n=1}^{N}\|\nabla h(\mathbf{Z}^{(n)})\|^2
\leq 2d \|h\|^2_{C_b^1}\leq 2d\kappa^2\|h\|^2_{\mathcal{H}_K},
\end{equation*}
which implies that $A_N$ is bounded and that $\|A_N\|\leq \sqrt{2d}\kappa$. obviously, $B_N$ is bounded, since $B_N=A^{\ast}_NA_N$.

We now prove $B_N$ is compact. Let $\{h_i\}_{i=1}^{\infty}$ be an infinite sequence in the closed unit ball $B_{\mathcal{H}_K}(0,1)$ of $\mathcal{H}_K$. Then 
\begin{equation*}
B_Nh_i=\frac{1}{N}\nabla^{\top}h_i(\mathbf{Z}_N)\nabla_1 K(\mathbf{Z}_N,\cdot)
=\frac{1}{N}\sum_{n=1}^{N}\sum_{j=1}^{2d}\partial_jh_i(\mathbf{Z}^{(n)})(\nabla_1 K)_j(\mathbf{Z}^{(n)},\cdot).
\end{equation*}
Note that for each fixed $j \in \left\{1,2,\dots,2d \right\}$ and $n \in \left\{1,2,\dots,N \right\}$, the sequence of numbers $\{|\partial_jh_i(\mathbf{Z}^{(n)})|\}_{i \in \mathbb{N}}$ is such that  $|\partial_jh_i(\mathbf{Z}^{(n)})|\leq \|h_i\|_{C^1_b}\leq \kappa\|h_i\|_{\mathcal{H}_K}\leq \kappa$ and hence it is bounded. Therefore, the Bolzano-Weierstrass theorem guarantees that it has a convergent subsequence. Since $j \in \left\{1,2,\dots,2d \right\}$ and $n \in \left\{1,2,\dots,N \right\}$ are a finite collection, it follows that we can choose a subsequence $\{h_{i_q}\}_{q=1}^{\infty}$ of $\{h_i\}_{i=1}^{\infty}$ such that for all $j \in \left\{1,2,\dots,2d \right\}$ and $n \in \left\{1,2,\dots,N \right\}$ the sequence of numbers $\{\partial_jh_{i_q}(\mathbf{Z}^{(n)})\}_{q=1}^{\infty}$ converges. Let us write that $\{\partial_jh_{i_q}(\mathbf{Z}^{(n)})\}_{q=1}^{\infty}$ converges to some $a^{(n)}_j\in \mathbb{R}$. Then
\begin{align*}
\Bigg\|B_N(h_{i_q})-\frac{1}{N}\sum_{n=1}^{N}\sum_{j=1}^{2d}a_j^{(n)}(\nabla_1 K)_j(\mathbf{Z}^{(n)},\cdot)\Bigg\|_{\mathcal{H}_K}&
\leq\ \frac{1}{N}\sum_{n=1}^{N}\sum_{j=1}^{2d}\Big|a_j^{(n)}-(\nabla^{\top}h_{i_q}(\mathbf{Z}^{(n)}))_j\Big|\cdot\|(\nabla_1 K)_j(\mathbf{Z}^{(n)},\cdot)\|_{\mathcal{H}_K}\\
\leq&\ \frac{1}{N}\sum_{n=1}^{N}\sum_{j=1}^{2d}\Big|a_j^{(n)}-(\nabla^{\top}h_{i_q}(\mathbf{Z}^{(n)}))_j\Big|\cdot \kappa< \varepsilon.
\end{align*}
for any $\varepsilon>0 $ and for all the terms $i _q$ above a sufficiently high $q \in \mathbb{N}$.
Hence, $B_N(B_{\mathcal{H}_K}(0,1))$ is relatively compact in $\mathcal{H}_K$. We conclude that $B_N$ is a compact operator.    

Finally, as in the previous proposition, the form of the operator $B_N=A_N^{*}A_N$ automatically guarantees that it is positive semidefinite.
\end{proof}
\subsection{Proof of Theorem \ref{Sam-Con}}
\begin{proof}
Recall that $B_N=A_N^* A_N$ is a positive compact operator. Let $B_N=\sum_{n=1}^N \lambda_n\left\langle\cdot, e_n\right\rangle_{\mathcal{H}_K} e_n$ be the spectral decomposition of $B_N$ with $0<\lambda_{n+1}<\lambda_n$ and $\left\{e_n\right\}_{n=1}^N$ be an orthonormal basis of $\mathcal{H}^N_K$. Then we can represent $\widetilde{h}_{\lambda, N}$ as
\begin{align*}
\widetilde{h}_{\lambda, N}=(B_N+\lambda)^{-1} B_N H=\sum_{n=1}^N\frac{\lambda_n}{\lambda_n+\lambda}\left\langle H, e_n\right\rangle_{\mathcal{H}_K} e_n.
\end{align*}
Denote $h^*:=\widetilde{h}_{\lambda, \infty}$. Note that 
\begin{align*}
\begin{aligned}
\left\|h^*\right\|_{\mathcal{H}_K}^2  =\sum_{n=1}^\infty\left(\frac{\lambda_n}{\lambda_n+\lambda}\right)^2\left|\left\langle H, e_n\right\rangle_{\mathcal{H}_K}\right|^2 \leq\sum_{n=1}^\infty\left|\left\langle H, e_n\right\rangle_{\mathcal{H}_K}\right|^2\leq \|H\|^2_{\mathcal{H}_K}<\infty .
\end{aligned}
\end{align*}
Therefore, $h^{*}\in \mathcal{H}_K$, and 
\begin{align*}
\left\|\widetilde{h}_{\lambda,N}-h^*\right\|_{\mathcal{H}_K}^2= \sum_{n=N+1}^\infty\left(\frac{\lambda_n}{\lambda_n+\lambda}\right)^2\left|\left\langle H, e_n\right\rangle_{\mathcal{H}_K}\right|^2 \to 0,\quad \text{as}\quad  N\to\infty,
\end{align*}
which shows that $\widetilde{h}_{\lambda,N}$ converges to $h^*$ with respect to the RKHS norm. Next, we will show that $h^*$ coincides with $h^*_{\lambda}$.
 
We now recall that the random samples $\{\mathbf{Z}^{(n)}\}_{n=1}^N$ are made out of independent random variables in $\mathbb{R}^{2d}$ with the same probability distribution $\mu_{\mathbf{Z}}$. 
Denote the empirical measure as $\mu^N_{\mathbf{Z}}:=\frac{1}{N}\sum_{n=1}^N\delta_{\mathbf{Z}^{(n)}}$.
The strong law of large numbers shows that for each $h\in\mathcal{H}_K$, we have
\begin{align*}
\int_{\mathbb{R}^{2d}}\|X_h(\mathbf{y})- X_H(\mathbf{y})\|^2\mathrm{d}\mu^N_{\mathbf{Z}}(\mathbf{y})\xrightarrow{N \rightarrow \infty}\int_{\mathbb{R}^{2d}}\|X_h(\mathbf{y})- X_H(\mathbf{y})\|^2\mathrm{d}\mu_{\mathbf{Z}}(\mathbf{y}),
\end{align*}
almost surely, as $N\to\infty$. Thus it follows that, 
\begin{align*}
\widetilde{R}_{\lambda,N}(h)&~=\frac{1}{N}\sum_{n=1}^{N}\left\|X_h(\mathbf{Z}^{(n)})- X_H(\mathbf{Z}^{(n)})\right\|^2+\lambda \|h\|_{\mathcal{H}_K}^2\\
&~=\int_{\mathbb{R}^{2d}}\|X_h(\mathbf{y})- X_H(\mathbf{y})\|^2\mathrm{d}\mu^N_{\mathbf{Z}}(\mathbf{y})+\lambda \|h\|_{\mathcal{H}_K}^2
\xrightarrow{N \rightarrow \infty}~\int_{\mathbb{R}^{2d}}\|X_h(\mathbf{y})- X_H(\mathbf{y})\|^2\mathrm{d}\mu_{\mathbf{Z}}(\mathbf{y})+\lambda \|h\|_{\mathcal{H}_K}^2\\
&=R_{\lambda}(h)-\sigma^2,
\end{align*}
almost surely, as $N\to\infty$. Therefore, for every $\lambda>0$ and $0<\delta<1$, with probability at least $1-\delta$, we have 
\begin{align}\label{poi-con}
\widetilde{R}_{\lambda,N}(h)~\xrightarrow{N \rightarrow \infty}~R_{\lambda}(h)-\sigma^2,\quad \forall ~h\in\mathcal{H}_K.   
\end{align}

For each $h_0\in\mathcal{H}_K$ and any $\varepsilon>0$, there exits $\delta=\min\left\{\varepsilon\left((2\kappa^2+2\lambda)\|h_0\|_{\mathcal{H}_K}+2\kappa^2\|H\|_{\mathcal{H}_K}+\lambda+\kappa^2\right)^{-1},1\right\}$, such that for any $h\in\mathcal{H}_K$ with $\|h-h_0\|_{\mathcal{H}_K}<\delta$, we have in particular
$\|h\|_{\mathcal{H}_K}<\|h_0\|_{\mathcal{H}_K}+1$, so that 
\begin{align*}
&\left|\widetilde{R}_{\lambda,N}(h)-\widetilde{R}_{\lambda,N}(h_0)\right|\\
\leq ~&\int_{\mathbb{R}^{2d}}\left|\|X_h(\mathbf{y})- X_H(\mathbf{y})\|^2-\|X_{h_0}(\mathbf{y})- X_H(\mathbf{y})\|^2\right| \mathrm{d}\mu^N_{\mathbf{Z}}(\mathbf{y})+\lambda \left|\|h\|_{\mathcal{H}_K}^2- \|h_0\|_{\mathcal{H}_K}^2\right| \\
=~&\int_{\mathbb{R}^{2d}}\left|\langle X_h(\mathbf{y})- X_{h_0}(\mathbf{y}),X_h(\mathbf{y})+X_{h_0}(\mathbf{y})- 2X_H(\mathbf{y})\rangle\right| \mathrm{d}\mu^N_{\mathbf{Z}}(\mathbf{y})+\lambda \left|\|h\|_{\mathcal{H}_K}^2- \|h_0\|_{\mathcal{H}_K}^2\right|\\
\leq ~&\|X_h- X_{h_0}\|_{L^2(\mu^N_{\mathbf{Z}})}\|X_h+ X_{h_0}-2X_H\|_{L^2(\mu^N_{\mathbf{Z}})}+\lambda \|h+ h_0\|_{\mathcal{H}_K}\|h- h_0\|_{\mathcal{H}_K}\\
\leq ~& \kappa^2\|h- h_0\|_{\mathcal{H}_K}\|h+ h_0-2H\|_{\mathcal{H}_K}+\lambda \|h+ h_0\|_{\mathcal{H}_K}\|h- h_0\|_{\mathcal{H}_K}\\
\leq ~&\left((2\kappa^2+2\lambda)\|h_0\|_{\mathcal{H}_K}+2\kappa^2\|H\|_{\mathcal{H}_K}+\lambda+\kappa^2\right)\|h- h_0\|_{\mathcal{H}_K}<\varepsilon,
\end{align*}
which shows the pointwise equi-continuity of $\{\widetilde{R}_{\lambda,N}\}_{N\geq1}$. Then combining the pointwise convergence \eqref{poi-con} and \cite[Proposition 5.9]{dal2012introduction}, we have that for every $\lambda>0$ and $0<\delta<1$, with probability at least $1-\delta$, the functional $\widetilde{R}_{\lambda,N}$ $\Gamma$-converges to $R_{\lambda}-\sigma^2$
as $N\to\infty$. Therefore by the fundamental theorem of $\Gamma$-convergence \cite[Corollary 7.20]{dal2012introduction}, the limit $h^*$ of the minimizers $\widetilde{h}_{\lambda, N}$ of $\widetilde{R}_{\lambda,N}$ is indeed the minimizer $h_{\lambda}^*$ of $R^{\lambda}-\sigma^2$ and hence the minimizer of $R^{\lambda}$. Thus, the result follows.
\end{proof}

\subsection{Proof of Theorem \ref{Sam-Err}}
\begin{proof} We introduce the intermediate quantity $(B_N+\lambda)^{-1}BH$ and decompose 
\begin{multline*}
(B_N+\lambda)^{-1}B_NH-(B+\lambda)^{-1}BH \\
=(B_N+\lambda)^{-1}B_NH-(B_N+\lambda)^{-1}BH+(B_N+\lambda)^{-1}BH- (B+\lambda)^{-1}BH.
\end{multline*}
Since the operator norm satisfies $\|(B_N+\lambda)^{-1}\|\leq \frac{1}{\lambda}$, we have that 
\begin{equation*}
\|(B_N+\lambda)^{-1}B_NH-(B_N+\lambda)^{-1}BH\|_{\mathcal{H}_K}\leq \frac{1}{\lambda} \|B_NH-BH\|_{\mathcal{H}_K}.
\end{equation*}
Applying Lemma \ref{Dec-Omp} to $B_NH-BH$, we obtain that, with probability at least $1-\delta/2$,
\begin{align}\label{bounds1}
\|(B_N+\lambda)^{-1}B_NH-(B_N+\lambda)^{-1}B H\|_{\mathcal{H}_K} \leq \left(\sqrt{ \frac{8\log(4/\delta)}{N}}+ 
1\right)\sqrt{ \frac{2\log(4/\delta)}{N\lambda^2}}2d\kappa^2\|H\|_{\mathcal{H}_K}.
\end{align}
On the other hand, we have that
\begin{align*}
\|(B_N+\lambda)^{-1}BH- (B+\lambda)^{-1}BH\|_{\mathcal{H}_K} &= \|(B_N+\lambda)^{-1}(B-B_N)(B+\lambda)^{-1}BH\|_{\mathcal{H}_K}\\ 
&\leq \frac{1}{\lambda}\|(B-B_N)(B+\lambda)^{-1}BH\|_{\mathcal{H}_K}.
\end{align*}

Since $h_{\lambda}^*=(B+\lambda)^{-1}BH$ is the unique minimizer of the regularized statistical risk $R _\lambda(h) =\|A{h}-A{H}\|^2_{L^2(\mu_{\mathbf{Z}})}+\lambda\|h\|_{\mathcal{H}_K}^2,$ plugging $h=0$, we obtain that 
\begin{align}\label{eq1}
\|A{h_{\lambda}^*}-A{H}\|^2_{L^2(\mu_{\mathbf{Z}})}+\lambda\|h_{\lambda}^*\|_{\mathcal{H}_K}^2 <\| A{H}\|^2_{L^2(\mu_{\mathbf{Z}})}.
\end{align}
Then by Proposition \ref{Wel-Ope}, we have
\begin{equation}\label{eq2}
\begin{aligned}
\|h_{\lambda}^*\|_{\mathcal{H}_K} <\frac{1}{\sqrt{\lambda}}\| A{H}\|_{L^2(\mu_{\mathbf{Z}})}\leq \frac{\sqrt{2d}\kappa}{\sqrt{\lambda}}\| H\|_{\mathcal{H}_K}.
\end{aligned}
\end{equation}
Applying Lemma \ref{Dec-Omp} to $h_{\lambda}^*=(B+\lambda)^{-1}BH$ and combining it with equation \eqref{eq2}, we obtain that with probability at least $1-\delta/2$,
\begin{equation}\label{bounds2}
\begin{aligned}
\frac{1}{\lambda}\|(B-B_N)(B+\lambda)^{-1}BH\|_{\mathcal{H}_K} & \leq \left(\sqrt{ \frac{8\log(4/\delta)}{N}}+ 
1\right)\sqrt{ \frac{2\log(4/\delta)}{N\lambda^2}}2d\kappa^2\|h_{\lambda}^*\|_{\mathcal{H}_K}\\
&\leq \left(\sqrt{ \frac{8\log(4/\delta)}{N}}+ 
1\right)\sqrt{ \frac{2\log(4/\delta)}{N\lambda^3}}2d\sqrt{2d}\kappa^3\|H\|_{\mathcal{H}_K}.
\end{aligned}
\end{equation}
Finally, by combining the bounds \eqref{bounds1} and \eqref{bounds2}, we obtain that with a probability at least $1-\delta$,
\begin{multline*}
\|(B_N+\lambda)^{-1}B_NH-(B+\lambda)^{-1}BH\|_{\mathcal{H}_K}  
\leq\left(\sqrt{ \frac{8\log(4/\delta)}{N}}+ 
1\right)\sqrt{ \frac{2\log(4/\delta)}{N\lambda^2}}\|H\|_{\mathcal{H}_K}2d\kappa^2\left(1+\kappa\sqrt{\frac{2d}{\lambda}}\right).
\end{multline*}
Finally, if we assume that $\lambda$ satisfies \eqref{dyn-sca}, the convergence upper rate $N^{-\frac{1}{2}(1-3\alpha)}$ follows. 
\end{proof}

\subsection{Proof of Corollary \ref{Sam-Err2}}
\begin{proof}
Using the coercivity condition \eqref{coercivity}, we could get a bound better than \eqref{eq2}. More precisely,  from \eqref{eq1} we have that
\begin{equation*}
2d\kappa^2\|H\|^2_{\mathcal{H}_K}\geq \| A{H}\|^2_{L^2(\mu_{\mathbf{Z}})}>c_{\mathcal{H}_K}\|{h_{\lambda}^*}-{H}\|^2_{\mathcal{H}_K}+\lambda\|h_{\lambda}^*\|_{\mathcal{H}_K}^2\geq c_{\mathcal{H}_K}\left(\|{h_{\lambda}^*}\|_{\mathcal{H}_K}-\|{H}\|_{\mathcal{H}_K}\right)^2+\lambda\|h_{\lambda}^*\|_{\mathcal{H}_K}^2.
\end{equation*}
Thus, it follows that 
\begin{align}\label{Opt-con}
(2d\kappa^2-c_{\mathcal{H}_K})\|{H}\|_{\mathcal{H}_K}^2+2c_{\mathcal{H}_K}\|{h_{\lambda}^*}\|_{\mathcal{H}_K}\|{H}\|_{\mathcal{H}_K}-(\lambda+c_{\mathcal{H}_K})\|{h_{\lambda}^*}\|_{\mathcal{H}_K}^2>0. 
\end{align}
Since $2d\kappa^2-c_{\mathcal{H}_K}>0$ and $\|H\|_{\mathcal{H}_K}\geq 0$, the above equation \eqref{Opt-con} always holds if
\begin{align*}
\|H\|_{\mathcal{H}_K} &>\frac{-2c_{\mathcal{H}_K}\|{h_{\lambda}^*}\|_{\mathcal{H}_K}+\sqrt{4c_{\mathcal{H}_K}^2\|{h_{\lambda}^*}\|_{\mathcal{H}_K}^2+4(2d\kappa^2-c_{\mathcal{H}_K})(\lambda+c_{\mathcal{H}_K})\|{h_{\lambda}^*}\|_{\mathcal{H}_K}^2}}{2(2d\kappa^2-c_{\mathcal{H}_K})}   \\
&=\frac{\sqrt{c_{\mathcal{H}_K}^2+(2d\kappa^2-c_{\mathcal{H}_K})(\lambda+c_{\mathcal{H}_K})}-c_{\mathcal{H}_K}}{2d\kappa^2-c_{\mathcal{H}_K}}\|{h_{\lambda}^*}\|_{\mathcal{H}_K}\geq \frac{\kappa\sqrt{2dc_{\mathcal{H}_K}}-c_{\mathcal{H}_K}}{2d\kappa^2-c_{\mathcal{H}_K}}\|{h_{\lambda}^*}\|_{\mathcal{H}_K}=\frac{\sqrt{c_{\mathcal{H}_K}}}{\sqrt{2d}\kappa+\sqrt{c_{\mathcal{H}_K}}}\|{h_{\lambda}^*}\|_{\mathcal{H}_K}.
\end{align*}
Applying Lemma \ref{Dec-Omp} to $h_{\lambda}^*=(B+\lambda)^{-1}BH$, and using \eqref{Opt-con}, we obtain that with probability at least $1-\delta/2$,
\begin{equation}\label{bounds3}
\begin{aligned}
\frac{1}{\lambda}\|(B-B_N)(B+\lambda)^{-1}BH\|_{\mathcal{H}_K} & \leq \left(\sqrt{ \frac{8\log(4/\delta)}{N}}+ 
1\right)\sqrt{ \frac{2\log(4/\delta)}{N\lambda^2}}2d\kappa^2\|h_{\lambda}^*\|_{\mathcal{H}_K}\\
\leq& \left(\sqrt{ \frac{8\log(4/\delta)}{N}}+ 
1\right)\sqrt{ \frac{2\log(4/\delta)}{N\lambda^2}}2d\kappa^2\frac{\sqrt{2d}\kappa+\sqrt{c_{\mathcal{H}_K}}}{\sqrt{c_{\mathcal{H}_K}}}\|H\|_{\mathcal{H}_K}.
\end{aligned}    
\end{equation}
Finally, by combining the bounds \eqref{bounds1} and \eqref{bounds3}, we obtain that with a probability at least $1-\delta$,
\begin{align*}
\|(B_N+\lambda)^{-1}B_NH-(B+\lambda)^{-1}BH\|_{\mathcal{H}_K}  
\leq\left(\sqrt{ \frac{8\log(4/\delta)}{N}}+ 
1\right)\sqrt{ \frac{2\log(4/\delta)}{N\lambda^2}}\|H\|_{\mathcal{H}_K}2d\kappa^2\left(2+\kappa\sqrt{\frac{2d}{c_{\mathcal{H}_K}}}\right),
\end{align*} 
which shows, in particular, that the convergence upper rate is $N^{-\frac{1}{2}(1-2\alpha)}$. 
\end{proof}

\subsection{Proof of Proposition \ref{Dis-Con}}
\begin{proof}
For each $t\in[0,T]$ and ${\bf z} \in \mathbb{R}^{2d} $, we have 
\begin{align*}
\left\|F_t({\bf z})-\widehat{F}_t({\bf z})\right\|&=\left\|\int_0^tJ\nabla H(F_s({\bf z}))-J\nabla\widehat{H}(\widehat{F}_s({\bf z}))\mathrm{d}s\right\|\leq\int_0^t\big\|\nabla H(F_s({\bf z}))-\nabla\widehat{H}(\widehat{F}_s({\bf z}))\big\|\mathrm{d}s\\
&\leq \int_0^t\big\|\nabla H(F_s({\bf z}))-\nabla \widehat{H}(F_s({\bf z}))\big\|+\big\|\nabla \widehat{H}(F_s({\bf z}))-\nabla\widehat{H}(\widehat{F}_s({\bf z}))\big\|\mathrm{d}s\\
&\leq \|H-\widehat{H}\|_{\mathcal{H}_K}\bigintsss_0^t\sqrt{\sum_{i=1}^{2d}(\nabla_{1,2}K)_{i,i}(\widehat{F}_s({\bf z}),\widehat{F}_s({\bf z}))}\mathrm{d}s+\int_0^t2d\|\widehat{H}\|_{C^2_b}\|F_s({\bf z})-\widehat{F}_s({\bf z})\|\mathrm{d}s\\
&\leq \|H-\widehat{H}\|_{\mathcal{H}_K}\alpha_t+\int_0^t2d\|\widehat{H}\|_{C^2_b}\|F_s({\bf z})-\widehat{F}_s({\bf z})\|\mathrm{d}s,
\end{align*}
where $\alpha_t:=\sqrt{2d}\kappa t$ and we used the inequality \eqref{embedding ineq} with $s=2$. Since $\alpha_t$ is non-decreasing as a function of $t$, then by the integral form of Gr\"{o}nwall's inequality that we recalled in Lemma \ref{Gro-Ine}, for $\|F_t({\bf z})-\widehat{F}_t({\bf z})\|$ we obtain
\begin{align*}
\|F_t({\bf z})-\widehat{F}_t({\bf z})\|&\leq \|H-\widehat{H}\|_{\mathcal{H}_K}\alpha_t\exp\left\{\int_0^t2d\|H\|_{C^2_b}\mathrm{d}s\right\} \\
&\leq \|H-\widehat{H}\|_{\mathcal{H}_K}\alpha_t\exp\{2d\kappa\|H\|_{\mathcal{H}_K}
t\}.
\end{align*}
Hence, we have that
\begin{align*}
\|F({\bf z})-\widehat{F}({\bf z})\|_{\infty}:=\max_{t\in[0,T]}\|F_t({\bf z})-\widehat{F}_t({\bf z})\|\leq C \|H-\widehat{H}\|_{\mathcal{H}_K},   
\end{align*}
where $C=\max_{t\in[0,T]}\{\alpha_t\exp\{2d\kappa\|H\|_{\mathcal{H}_K}
t\}\}=\sqrt{2d}\kappa T\exp\{2d\kappa\|H\|_{\mathcal{H}_K}
T\}$. The result follows.
\end{proof}

\section{Analysis of the approximation and the noisy sampling errors}
\label{Analysis of the approximation error and noise part of the sampling error}
\paragraph{Analysis of the approximation error ${h_{\lambda}^*-H}$.}
Under the source condition that we defined in \eqref{sou-con}, the analysis of $\| h_{\lambda}^*-H\|_{\mathcal{H}_K}$ can be carried out using standard results on Tikhonov regularization (see, for instance, Section 5 in \cite{caponnetto2005fast}). We nevertheless present such analysis here for the sake of completeness. 

Recall first that by Proposition \ref{Wel-Ope}, the operator $B=A^*A$ is a positive compact operator. Let $B=\sum_{n=1}^{L}\lambda_n\langle \cdot, e_n\rangle e_n$ (possibly $L=\infty$) be the spectral decomposition of $B$ with $0<\lambda_{n+1}<\lambda_{n}$ and $\{e_n\}_{n=1}^{L}$ be an orthonormal basis of $\mathcal{H}_K$. 
Then,
\begin{equation*}
\| h_{\lambda}^*-H\|_{\mathcal{H}_K}^2=\|(B+\lambda)^{-1}BH-H\|_{\mathcal{H}_K}^2 =\|\lambda (B+\lambda)^{-1}H\|_{\mathcal{H}_K}^2 =\sum_{n=1}^{L}\left(\frac{\lambda}{\lambda_n+\lambda}\right)^2|\langle H, e_n\rangle_{\mathcal{H}_K}|^2. 
\end{equation*}
Since the function $x^\gamma$ is concave on $[0,\infty]
$, $\frac{\lambda}{\lambda_n+\lambda}\leq \frac{\lambda^\gamma}{\lambda_n^\gamma}$. 
Then by the source condition \eqref{sou-con}, we have 
\begin{align*}
\|h_{\lambda}^*-H\|_{\mathcal{H}_K} \leq \lambda^{\gamma}\|B^{-\gamma}H\|_{\mathcal{H}_K},
\end{align*}
where $B^{-\gamma}H$ represents the pre-image of $H$.

\paragraph{Analysis of the noisy sampling error $\widehat{h}_{\lambda,N}-\widetilde h_{\lambda,N}$.}
By the decomposition of the estimation error in \eqref{min-dec} and \eqref{noisy part of sampling}, the noisy part is 
\begin{align*}
\widehat{h}_{\lambda,N}-\widetilde h_{\lambda,N}=\frac{1}{\sqrt{N}}(B_N+\lambda)^{-1}A_N^{*} \mathbf{E}_N,
\end{align*} 
where the noise vector $\mathbf{E}_N$ follows a multivariate distribution with zero mean and variance $\sigma^2I_{2dN}$. Using an approach similar to the one in the Differential Representer Theorem \ref{Rep-Ker}, we obtain that
\begin{align*}
\big\|\widehat{h}_{\lambda,N}-\widetilde h_{\lambda,N}\big\|_{\mathcal{H}_K}^2 &= \frac{1}{N}\langle \mathbf{E}_N, A_N(B_N+\lambda)^{-2}A_N^*\mathbf{E}_N\rangle\\
&= \mathbf{E}_N^{\top}\Sigma_N \mathbf{E}_N,
\end{align*} 
with the matrix $$\Sigma_N:= (K_{X_H}(\mathbf{Z}_N,\mathbf{Z}_N) +\lambda N I)^{-1}K_{X_H}(\mathbf{Z}_N,\mathbf{Z}_N) (K_{X_H}(\mathbf{Z}_N,\mathbf{Z}_N) +\lambda N I)^{-1}.$$  
Notice that
\begin{align*}
\mathrm{Tr}(\Sigma_N)\leq \frac{1}{\lambda^2N^2}\mathrm{Tr}(K_{X_H}(\mathbf{Z}_N,\mathbf{Z}_N) )= \frac{1}{\lambda^2N^2} \sum_{n=1}^{N} \sum_{i=1}^{2d}(\nabla_{1,2}K)_{i,i}(\mathbf{Z}^{(n)},\mathbf{Z}^{(n)})\leq \frac{2d}{\lambda^2N}\kappa^2,
\end{align*}
with $\kappa^2=\|K\|_{C_b^{2}(\mathbb{R}^{2d} \times \mathbb{R}^{2d})}$ as introduced in the statement of Proposition \ref{Wel-Ope} and 
\begin{align*}
\mathrm{Tr}(\Sigma_N^2)\leq 
\frac{4d^2\kappa^4}{\lambda^4N^2},
\end{align*}
We now apply the Hanson-Wright inequality (Theorem \ref{Han-Wri}) for the random vector $\mathbf{E}_N$ with $S_0=\sigma^2$. Then we obtain that, for any $\varepsilon>0$,
\begin{align*}
\mathbb{P}\left(\|\mathbf{E}_N^{\top}\Sigma_N \mathbf{E}_N -\mathbb{E}[\mathbf{E}_N^{\top}\Sigma_N \mathbf{E}_N] \|\geq\varepsilon\right)&\leq 2\exp\bigg\{-c\min \bigg\{ \frac{\varepsilon^2}{\sigma^4\|\Sigma_N\|_{\mathrm{HS}}^2}, \frac{\varepsilon}{\sigma^2\|\Sigma_N\|}\bigg\}\bigg\} \\
&\leq 2\exp\bigg\{-c\min \bigg\{ \frac{\varepsilon^2}{\sigma^4\mathrm{Tr}(\Sigma_N^2)}, \frac{\varepsilon}{\sigma^2 \mathrm{Tr}(\Sigma_N)}\bigg\}\bigg\},
 \end{align*}
where $c$ is a positive constant appearing in the Hanson--Wright inequality. Let $t>0$ and denote $t^2=\min \Big\{ \frac{\varepsilon^2}{\sigma^4\mathrm{Tr}(\Sigma_N^2)}, \frac{\varepsilon}{\sigma^2 \mathrm{Tr}(\Sigma_N)}\Big\}$, that is, $\varepsilon=\sigma^2\max \Big\{ t^2\mathrm{Tr}(\Sigma_N),t\sqrt{\mathrm{Tr}(\Sigma_N^2)}\Big\}$. Then with probability at least $1-2e^{-ct^2}$, we have
\begin{align*}
\mathbf{E}_N^{\top}\Sigma_N \mathbf{E}_N &\leq \mathbb{E}[\mathbf{E}_N^{\top}\Sigma_N \mathbf{E}_N]+\varepsilon =\mathrm{Tr}(\Sigma_N)\sigma^2+\sigma^2\max \bigg\{ t^2\mathrm{Tr}(\Sigma_N),t\sqrt{\mathrm{Tr}(\Sigma_N^2)}\bigg\}\\
&\leq \sigma^2 \max \bigg\{ \mathrm{Tr}(\Sigma_N),\sqrt{\mathrm{Tr}(\Sigma_N^2)}\bigg\}(1+t+t^2)= \frac{2d\sigma^2\kappa^2 }{\lambda^2N}(1+t+t^2)
 \end{align*}
Therefore, with a probability of at least $1-\delta/2$, it holds that
\begin{equation*}
\begin{aligned}
\|\widetilde h_{\lambda,N}- \widehat{h}_{\lambda,N}\|_{\mathcal{H}_K} \leq \sqrt{ \frac{2d\sigma^2\kappa^2}{\lambda^2N}(1+t+t^2)}\leq \frac{\sigma \kappa}{\lambda }\sqrt{\frac{2d}{N}}\left(1+\sqrt{\frac{1}{c}\log(4/\delta)}\right).
\end{aligned} 
\end{equation*}

\section{Online regression with kernels}\label{online learning}
Suppose that we have observed $N$ noisy data of the Hamiltonian vector field. Then, according to equation \eqref{rep-ker}, we can compute the structure-preserving kernel estimator as $ \widehat{h}_{\lambda,N}=\widehat{\mathbf{c}}_N\cdot \nabla_1K(\mathbf{Z}_N,\cdot)$, where
\begin{align}
\label{recursion_def}
\widehat{\mathbf{c}}_N= (\nabla_{1,2}K(\mathbf{Z}_N,\mathbf{Z}_N)+\lambda NI)^{-1}\mathbb{J}^{\top}\mathbf{X}_{\sigma^2,N}=:\mathbf{K}_N^{-1}\mathbb{J}^{\top}\mathbf{X}_{\sigma^2,N}.  
\end{align}
Suppose we now observe one more data point $(\mathbf{Z},\mathbf{X})$. the objective is then to derive a recursive expression for $\mathbf{K}_{N+1}^{-1}$. It is clear that the computation of the inverse of ${\bf K}_{N+1}$ for each new data point is expensive. 

Alternatively, we note that 
\begin{align*}
\mathbf{K}_{N+1}=\left[\begin{array}{cc}
\mathbf{K}_{N}+\lambda  I & \mathbf{b}_N\\
\mathbf{b}_N^{\top} & \widetilde{\mathbf{A}}
\end{array}\right],   
\end{align*}
where $\mathbf{b}^{\top}_N=\left[\nabla_{1,2}K(\mathbf{Z}^{(1)},\mathbf{Z})| \cdots | \nabla_{1,2}K(\mathbf{Z}^{(n)},\mathbf{Z})\right]$ and the matrix $\widetilde{\mathbf{A}}=\nabla_{1,2}K(\mathbf{Z},\mathbf{Z})+\lambda(N+1)I$.
Then by \cite[Theorem 2.1]{lu2002inverses}, we obtain
\begin{align}
\label{online update bad}
\mathbf{K}_{N+1}^{-1}&=\left[\begin{array}{cc}
\widetilde{\mathbf{K}}_N^{-1}+\widetilde{\mathbf{K}}_N^{-1}\mathbf{b}_N\widetilde{\mathbf{D}}_N^{-1}\mathbf{b}_N^{\top}\widetilde{\mathbf{K}}_N^{-1} & -\widetilde{\mathbf{K}}_N^{-1}\mathbf{b}_N\widetilde{\mathbf{D}}_N^{-1} \\
-\widetilde{\mathbf{D}}_N^{-1}\mathbf{b}_N^{\top}\widetilde{\mathbf{K}}_N^{-1} & \widetilde{\mathbf{D}}_N^{-1}
\end{array}\right], 
\end{align}
where $\widetilde{\mathbf{D}}_N=\widetilde{\mathbf{A}}-\mathbf{b}_N^{\top}\widetilde{\mathbf{K}}_{N}^{-1}\mathbf{b}_N$, and $\widetilde{\mathbf{K}}_N^{-1}=(\mathbf{K}_N+\lambda I)^{-1}.$ In general, it will be expensive to compute $\widetilde{\mathbf{K}}_N^{-1}$ for each data update. One way to solve this problem is dynamically updating the ridge regression constant $\lambda $ as the sample size $N$ grows; this implicitly means that for each sample size, we are solving a different kernel ridge regression problem, but in exchange, this allows the formulation of an online updating rule that is much more convenient than \eqref{online update bad}.

Indeed, let $C>0$ be a constant and let $\lambda(N)>0  $ be given by the relation $\lambda(N) N=C $, for any $N \in \mathbb{N}  $. With this prescription, the solutions given by \eqref{recursion_def} and where $\lambda $ is replaced by $\lambda(N)$ can be recursively obtained by using the update rule 
\begin{align}
\label{online update good}
\mathbf{K}_{N+1}^{-1}&=\left[\begin{array}{cc}
\mathbf{K}_{N}^{-1}+\mathbf{K}_{N}^{-1}\mathbf{b}_N\mathbf{D}_N^{-1}\mathbf{b}_N^{\top}\mathbf{K}_{N}^{-1} & -\mathbf{K}_{N}^{-1}\mathbf{b}_N\mathbf{D}_N^{-1} \\
-\mathbf{D}_N^{-1}\mathbf{b}_N^{\top}\mathbf{K}_{N}^{-1} & \mathbf{D}_N^{-1}
\end{array}\right], 
\end{align}
where $\mathbf{D}_N=\mathbf{A}-\mathbf{b}_N^{\top}\mathbf{K}_{N}^{-1}\mathbf{b}_N$ and the matrix $\mathbf{A}=\nabla_{1,2}K(\mathbf{Z},\mathbf{Z})+CI$. In this way, by updating $\mathbf{K}_{N}^{-1}$ at each iteration, we avoid recomputing the inverse of the possibly very large matrix ${\bf K}_N$, and we hence achieve a computationally cheap online update. A possible choice of constant $C$ is given by Theorem \ref{Brd-Gau} that suggests that if we take $\lambda(N)  $ such that $\lambda(N) N=\sigma^2$ then the online updates \eqref{online update good} of the kernel ridge regression solution \eqref{recursion_def} will also provide an expression for the mean of the Gaussian posterior.

\nomenclature{$d$}{Dimension of the configuration space}
\nomenclature{$N$}{Number of random samples}
\nomenclature{$M$}{$M=2dN$}
\nomenclature{$H:\mathbb{R}^{2d}\longrightarrow\mathbb{R}$}{Hamiltonian function}
\nomenclature{$\mathbf{X}_{\sigma^2,N}$}{Random sample of (noisy) Hamiltonian vector field values}
\nomenclature{$\mathbf{Z}_N $}{Random sample of phase space values}
\nomenclature{$X _h $}{Hamiltonian vector field associated to $h:\mathbb{R}^{2d}\longrightarrow\mathbb{R}$}
\nomenclature{$K:{\cal X} \times {\cal X}\longrightarrow \mathbb{R}$}{Mercer kernel}
\nomenclature{$K(x,\cdot)=K _x$}{Kernel section associated to $x \in {\cal X} $}
\nomenclature{$C_b^s(\mathbb{R}^d)$}{Bounded $s$-continuously differentiable functions with bounded derivatives}
\nomenclature{${\mathcal H}_K $}{Reproducing kernel Hilbert space (RKHS) associated to the Mercer kernel $K$}
\nomenclature{$\widehat{R}_{\lambda,N}$}{Regularized empirical risk}
\nomenclature{$\widetilde{R}_{\lambda,N}$}{Noiseless regularized empirical risk}
\nomenclature{$ \widehat{h}_{\lambda,N}$}{Structure-preserving kernel estimator}
\nomenclature{$R_{\lambda}$}{Regularized statistical risk}
\nomenclature{$h^{*}_{\lambda}\in\mathcal{H}_K$}{Best-in-class Hamiltonian estimator} 
\nomenclature{$\widetilde{h}_{\lambda,N}$}{Noise-free structure-preserving kernel estimator}

\footnotesize
\addcontentsline{toc}{section}{Glossary of symbols}
\printnomenclature[15em]

\normalfont
\addcontentsline{toc}{section}{Acknowledgments}
\section*{Acknowledgments}
The authors thank Lyudmila Grigoryeva for helpful discussions and remarks and two referees whose suggestions have significantly improved the paper. We acknowledge partial financial support from the School of Physical and Mathematical Sciences of the Nanyang Technological University. DY is funded by the Nanyang President's Graduate Scholarship of Nanyang Technological University.

\footnotesize
\addcontentsline{toc}{section}{References}
\bibliographystyle{wmaainf}
\bibliography{Refs}
\end{document}